\documentclass[12pt]{article}

\usepackage[T1]{fontenc}
\usepackage[letterpaper,margin=1in]{geometry}
\usepackage{microtype}
\usepackage{graphicx}
\usepackage{amsmath,amssymb,amsfonts}
\usepackage{amsthm}
\usepackage{newtxtext,newtxmath}
\usepackage[table]{xcolor}
\usepackage{booktabs}
\usepackage{tabularx}
\usepackage{threeparttable}
\usepackage{makecell}
\usepackage{array}
\usepackage{caption}
\usepackage{float}

\usepackage{rotating}

\newcommand{\suppfigure}[1]{\makebox[\linewidth][c]{\includegraphics[width=0.98\textwidth]{#1}}}
\newcommand{\suppfigurefit}[1]{\makebox[\linewidth][c]{\includegraphics[width=0.98\textwidth]{#1}}}
\newcommand{\suppfiguretall}[1]{\makebox[\linewidth][c]{\includegraphics[width=0.98\textwidth]{#1}}}
\usepackage[section]{placeins}
\usepackage[most]{tcolorbox}
\usepackage{url}
\usepackage[numbers,sort&compress]{natbib}
\usepackage[hidelinks]{hyperref}
\graphicspath{{figures/}}
\date{}

\renewenvironment{abstract}{\quotation\bfseries\boldmath}{\endquotation}
\setcitestyle{numbers,round,sort&compress}

\definecolor{tablehead}{HTML}{DDE5E8}
\definecolor{tablerow}{HTML}{F2F5F4}
\definecolor{copal}{HTML}{2F7C76}
\definecolor{csafe}{HTML}{315D82}
\definecolor{cwarn}{HTML}{B48743}
\definecolor{cneutral}{HTML}{6E767B}
\newcommand{\OPAL}{OPAL}

\newcommand{\mainfigure}[2][0.92\textwidth]{\makebox[\linewidth][c]{\includegraphics[width=#1]{#2}}}
\newcolumntype{L}[1]{>{\raggedright\arraybackslash}p{#1}}
\newcolumntype{C}[1]{>{\centering\arraybackslash}p{#1}}
\newcolumntype{Y}{>{\raggedright\arraybackslash}X}
\theoremstyle{plain}
\newtheorem{proposition}{Proposition}
\newtheorem{lemma}{Lemma}
\newtheorem{corollary}{Corollary}
\theoremstyle{definition}

\theoremstyle{remark}
\newtheorem*{remark}{Remark}
\makeatletter
\renewcommand{\fnum@figure}{\textbf{Fig. \thefigure}}
\renewcommand{\fnum@table}{\textbf{Table \thetable}}
\makeatother
\DeclareCaptionFont{compactcaption}{\small\linespread{1.05}\selectfont}
\hypersetup{
  pdftitle={Held-out evidence resolves follow-up measurement decisions in biological screens},
  pdfauthor={Jia Bi, Samuel Pinilla and Chenyang Zhu}
}

\begin{document}
\nocite{macarron2011hts,
hase2019,
burger2020,
kusne2020cameo,
macleod2022selfdriving,
szymanski2023,
dai2024robots,
mahecic2022event,
meirovitch2026smartem,
wognum2024assessment,
bernett2024leakage,
graber2025bias,
asiaee2026cancerleakage,
ahlmanneltze2025linear,
kapoor2024reforms,
szymanski2026correction,
lindley1956,
chaloner1995,
shahriari2016,
rainforth2024,
ma2019eddi,
vonkleist2025afa,
sui2015safeopt,
thomas2019seldonian,
angelopoulos2024crc,
fenwick2020voi,
tosh2025batchie,
bendavid2010domains,
bendavid2010impossibility,
pearl2014transportability,
wawer2014profiling,
bray2017cellpaintingdata,
caicedo2017profiling,
weisbart2024gallery,
nist_acceptance_sampling,
serfling1974,
subramanian2017l1000,
niepel2017lincs,
seashoreludlow2015,
chandrasekaran2024jump,
arevalo2024batch,
gamella2025,
tikhomirov2026silent,
blackwell1953,
angelopoulos2025,
clopper1934,
efron1987bca,
chernozhukov2018,
chernozhukov2013,
howard2021,
maurer2009empirical,
hoeffding1963,
thomas2015,
laroche2019spibb,
cho2025cspimt,
chow1970,
elyaniv2010selective,
sawarni2026abstention,
wu2016conservative,
swaminathan2015,
jiang2016,
jacobs1991,
jordan1994,
diamond2024,
esrf2024}

\title{\bfseries\boldmath Held-out evidence resolves follow-up measurement decisions in biological screens}
\author{Jia Bi$^{1,*}$, Samuel Pinilla$^{2}$, Chenyang Zhu$^{3,*}$\\[0.35em]
\small $^{1}$Scientific Computing Department, Science and Technology Facilities Council,\\
\small Rutherford Appleton Laboratory, Didcot, UK\\
\small $^{2}$Diamond Light Source, Harwell Science and Innovation Campus, Didcot, UK\\
\small $^{3}$School of Electronics and Computer Science, University of Southampton,\\
\small Southampton, UK\\[0.35em]
\small $^{*}$Corresponding authors. Email: jia.bi@stfc.ac.uk; c.zhu@soton.ac.uk}

\maketitle
\noindent\textbf{Short title:} Follow-up decisions in biological screens

\begin{abstract}
Machine learning determines which follow-up measurements biological screens collect. In a six-rule Cell Painting battery, the highest-value rule would re-image 96.01\% of the library and had a 97.14\% false-activation upper bound, showing why predicted value alone cannot justify replacing a fixed plan. We developed \OPAL{}, a held-out decision test that freezes a rule and judges unnecessary measurement, coverage and value after cost against archive-specific criteria fixed before final evaluation. A development-selected sparse Cell Painting rule had 18.2-fold lower added-well burden, but its false-discovery bound exceeded 35\%, so the fixed plan remained. LINCS--LJP favored broad acquisition under point-estimate criteria set during development, not selective saving. CTRP required fallback because its frozen score missed measured opportunity. \OPAL{} separates optimization from evidence sufficient to change an experiment.
\end{abstract}

\section*{INTRODUCTION}
\phantomsection\label{sec:introduction}

High-throughput biological experiments must decide where to spend follow-up
wells, imaging time, samples and staff effort. Machine learning offers a direct
acquisition strategy: learn from initial measurements, then prioritize compounds,
conditions or samples for deeper phenotyping. Such rules increasingly shape
drug discovery, high-throughput screening and shared scientific facilities
~\cite{macarron2011hts}, while closed-loop and self-driving laboratories extend
them to sequential campaigns
~\cite{hase2019,burger2020,kusne2020cameo,macleod2022selfdriving,szymanski2023,dai2024robots}.
Learned decisions already control biological instruments: event-driven
fluorescence microscopy uses neural-network detection to switch acquisition
rates in real time, whereas SmartEM allocates slower rescans only to regions
that require higher-quality electron microscopy
~\cite{mahecic2022event,meirovitch2026smartem}.
Each choice consumes experimental resources and determines which evidence will
exist later. Predicted value alone, however, does not tell an experimentalist
whether to measure a sparse subset, acquire broadly or keep the fixed plan.
That choice must be made before the follow-up result is known.

Recent work in biological machine learning has sharpened this distinction.
Calls for blinded assessment in real-world drug discovery and analyses of
split construction, dependence, and data leakage show that benchmark
performance can change materially when the intended unit of generalization is
enforced~\cite{wognum2024assessment,bernett2024leakage,graber2025bias,asiaee2026cancerleakage}.
Critical benchmarking can also show when model complexity does not translate
into experimental value. Five foundation models and two other deep-learning
models for gene-perturbation prediction did not outperform deliberately simple
linear baselines~\cite{ahlmanneltze2025linear}. REFORMS likewise codifies
consensus recommendations for validity, reproducibility and generalizability
in machine-learning-based science~\cite{kapoor2024reforms}. An author
correction to the A-Lab report makes the consequence concrete: one target
mistakenly present in the training data was removed, and manual reanalysis
upheld 36 of 40 reported successes while classifying four as inconclusive from
XRD alone; the article title was revised from ``novel'' to ``inorganic''
materials~\cite{szymanski2023,szymanski2026correction}. Together, these studies
establish the need for leakage-resistant, held-out assessment of predictive
claims. They stop before the operational decision considered here: after a
rule has been trained, what evidence is sufficient for it to replace the
laboratory's fixed measurement plan?

Active learning, Bayesian optimization and experimental design choose the next
experiment after an objective has been specified
~\cite{lindley1956,chaloner1995,shahriari2016,rainforth2024}, and active feature
acquisition chooses which optional observation to collect
~\cite{ma2019eddi,vonkleist2025afa}. Safe learning, selective prediction and
calibrated-risk methods constrain candidate decisions, whereas
value-of-information analysis prices additional evidence
~\cite{sui2015safeopt,thomas2019seldonian,angelopoulos2024crc,fenwick2020voi}.
Together, these approaches optimize, constrain or value acquisition once a
candidate objective and action space have been accepted. They do not answer the
prior authorization question studied here: whether evidence from the intended
setting is sufficient to replace a fixed experimental plan before laboratory
capacity is committed. This distinction matters because a broad rule can
accumulate expected value by measuring almost every unit, while a sparse rule
can lower burden without being reliable enough to deploy. BATCHIE
prospectively explored approximately 4\% of 1.4 million possible drug-combination
experiments and experimentally validated prioritized combinations
~\cite{tosh2025batchie}. It demonstrates that adaptive acquisition can succeed
when the assay, objective and validation path are aligned, but it does not
provide a held-out test of whether a learned rule should replace an existing
measurement plan.

To close this gap, we developed \OPAL{} (the opportunity-aware protocol for
authorizing learned measurement rules), a held-out decision framework. \OPAL{}
tests whether a learned rule has enough evidence to change a fixed experiment;
unlike methods that optimize the next measurement, it treats adoption itself
as the scientific decision. It
freezes one learned rule before evaluation outcomes are opened and judges its
assignments against three coupled requirements: unnecessary acquisition,
nontrivial coverage and value after cost. Within a specified evaluation plan,
failure of any fixed requirement retains the fixed action. If a criterion is
revised during development, the revised rule must be fixed before evaluation
on a new held-out partition.
The numerical criteria are archive-specific
and fixed before the corresponding final outcome stage; they are not a common
biological effect scale. A separate methodological layer defines
when these judgments are identifiable and certifiable. Target outcomes or
justified transport assumptions are needed when outcomes may shift across
settings
~\cite{bendavid2010domains,bendavid2010impossibility,pearl2014transportability},
and an identified finite campaign can still be too small to support a claim
about its unobserved remainder. Exact finite-campaign results define this
boundary, while held-out simulations show that risk control can preserve a
nontrivial active branch under constructed conditions.

Three measured biological archives then resolved distinct experimental
decisions. Cell Painting produced a lower-burden sparse candidate whose mean and
prespecified nominal compound-bootstrap endpoint were positive, but
false-discovery uncertainty kept the fixed plan in place; post-hoc block
sensitivity made the value conclusion inconclusive.
LINCS--LJP paired a positive mean point estimate with acquisition of every
episode in the untouched final partition, indicating broad acquisition in the
observed setting rather than selective savings. CTRP fell back because its
frozen decision-time score did not rank measured opportunity. These states
direct further rule development, broad measurement and score revision,
respectively; they are not levels of algorithmic success or results on one
evidentiary scale. None yielded a measurement-saving selective policy.

\section*{RESULTS}\label{sec:results}

Three measured archives first resolve the experimental decisions; exact
analysis, reconstructed execution and simulation then separate limitations of a
learned rule from limits imposed by the evidence design. Supplementary
Table~\ref{tab:study_design} distinguishes these empirical and
methodological roles, and Supplementary Table~\ref{tab:claims} maps each
reported claim to its evidence and boundary.

\subsection*{Three biological archives resolve distinct evidence states}\label{sec:r0_framework}
Across the three measured archives, held-out evaluation resolved a sparse Cell
Painting candidate, broad LINCS--LJP acquisition and CTRP fallback. \OPAL{}
separates these states by evaluating a learned selection rule before it replaces
the fixed action \(K_0\). In Cell Painting, all 11,002 compounds with measured
outcomes reveal broad overlap between selected POSITIVE and selected measured-NULL
units on the decision-time predicted-gain axis (Fig.~\ref{fig:archive_decisions}, A).
The practical burden difference is visible in the comparator battery: among
rules that assigned follow-up wells, only the sparse candidate remained below
the prespecified false-activation limit, and it assigned 1,190 wells rather than
21,630 for the highest-value rule (Fig.~\ref{fig:archive_decisions}, B). The three archive evaluations used
different frozen rules and archive-specific criteria. Their operational
decisions can be read on a common selection axis, but not on a common
evidentiary scale: Cell Painting and CTRP used confidence-bound criteria,
whereas LINCS--LJP used point-estimate criteria revised before FINAL, with bounds reported
separately. They resolved 5.28\%
selection with authorization withheld, 100\% selection for broad acquisition
and 0\% selection with fallback (Fig.~\ref{fig:archive_decisions}, C). The
underlying access sequence and partition roles are reported in Supplementary
Table~\ref{tab:study_design} and Methods. None of the three supported a
measurement-saving selective policy; the LINCS--LJP action was a measured-data
``yes'' to broad acquisition rather than to selection.

\begin{figure}[!htbp]
\centering
\mainfigure{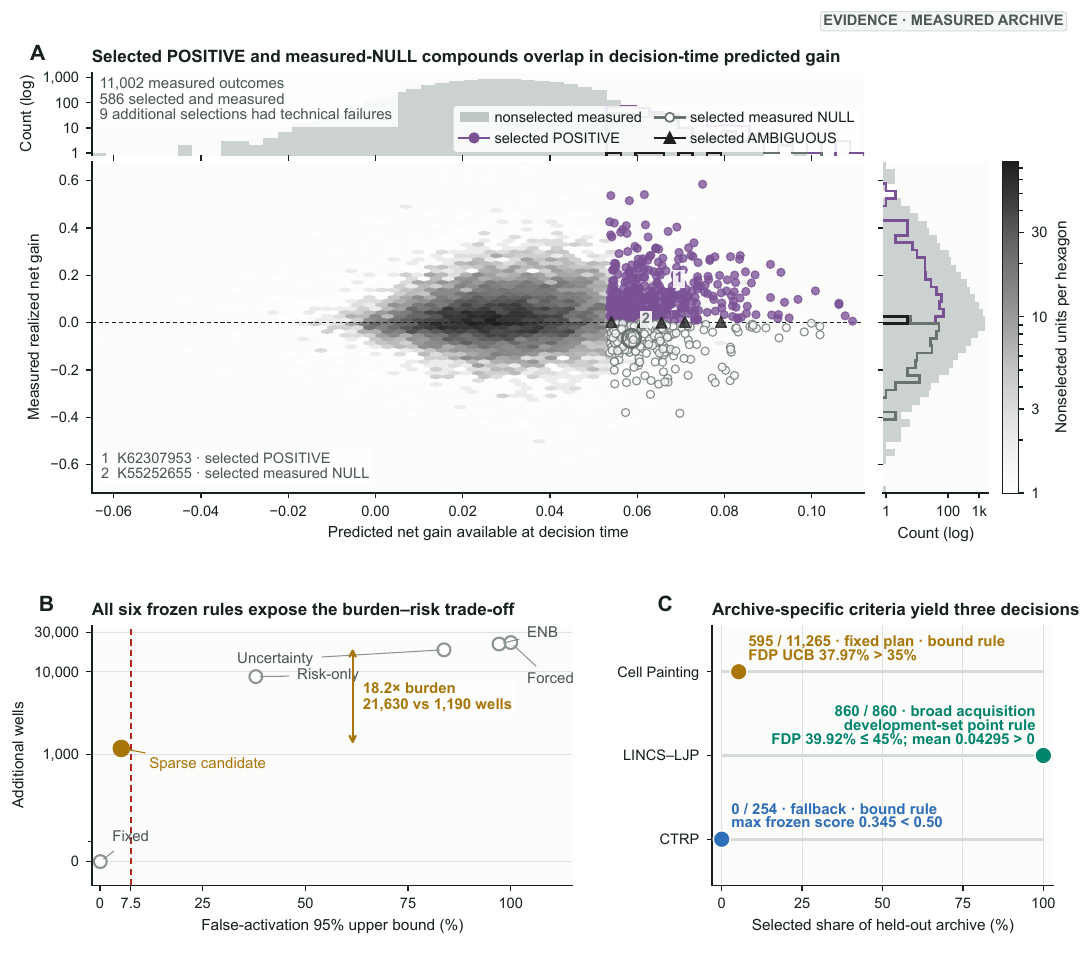}
\caption{\textbf{Measured outcomes separate prediction, burden and archive-level decision.} (\textbf{A}) All 11,002 measured Cell Painting outcomes. Hexagons show 10,416 nonselected compounds; points show all 586 selected measured compounds as POSITIVE, AMBIGUOUS or measured NULL. Marginals summarize the complete selected and nonselected distributions; the upper count axis is logarithmic. Exemplars 1 and 2 are expanded in Fig.~\ref{fig:cpg_target_authorization}, and nine selected technical failures are counted separately. (\textbf{B}) Six frozen rules by false-activation upper bound and added wells. The sparse candidate assigned 1,190 wells below the 7.5\% limit, versus 21,630 for the highest-value rule, an 18.2-fold burden difference. (\textbf{C}) Archive-specific selected fractions and decisive evidence: Cell Painting, 595/11,265 with FDP UCB 37.97\% above 35\%, so the fixed plan remained under a bound-based rule; LINCS--LJP, 860/860 with FDP 39.92\% at or below 45\% and mean value 0.04295 above zero under point-estimate criteria set during development, favoring broad acquisition; CTRP, 0/254 with 0/7 POSITIVE families activated and fallback under a bound-based rule. The horizontal axis is common, but the archive-specific criteria are not exchangeable or on one evidentiary scale.}
\label{fig:archive_decisions}
\end{figure}

\FloatBarrier
\subsection*{Cell Painting identifies a lower-burden sparse candidate without replacing the fixed plan}
\label{sec:r7_cpg0012}

Cell Painting provided an archive-based, outcome-unopened test of the
target-labeled route. One initial cpg0012 compound profile informed whether
two additional profiles were warranted, while a fourth replicate supplied
evaluation-only verification
~\cite{wawer2014profiling,bray2017cellpaintingdata,caicedo2017profiling,weisbart2024gallery}. The
gate used only hashed structure, initial morphology and pre-imaging metadata.
Here scientific return meant more reproducible estimation of a compound's
morphological profile: agreement with an independent verifier after adding two
replicate wells. It was not hit confirmation, mechanism-of-action discovery or
biological efficacy. Replicate-stable profiles are needed before perturbations
can be compared or grouped by phenotype, so the endpoint measures a decision-relevant
quality-control gain rather than a downstream biological effect
~\cite{caicedo2017profiling}. Development on 13,748 labeled target compounds produced no
rule that met all seven screens. The frozen candidate met the other six and was
the unique near-boundary rule with the smallest false-discovery-bound excess
(35.11\% against 35\%). That rule
made every assignment on the held-out 11,265-compound partition before any
additional or verifier profile was accessed (Supplementary
Section~\ref{sec:sr_cpg0012}).

A prespecified expected-net-benefit gate exposed the difference between high
predicted value and an evidence-backed measurement decision. It would have
selected additional imaging for 10,815 of 11,265 compounds (96.01\%). This gate
assigned 21,630 additional wells and also had the largest held-out executed
value in the six-rule battery (mean,
\(1.727\times10^{-2}\), with a nominal one-sided 95\% BCa lower endpoint,
\(1.483\times10^{-2}\)). The same gate activated 4,419 of 4,569
compounds in the prespecified least-favorable NULL denominator, giving a
one-sided 95\% false-activation upper bound of 97.14\%. In this battery,
maximizing value therefore favored additional imaging for almost the entire
library (Fig.~\ref{fig:archive_decisions}, B, and Supplementary
Table~\ref{tab:cpg0012_comparators}).

\begin{figure}[!htbp]
\centering
\mainfigure{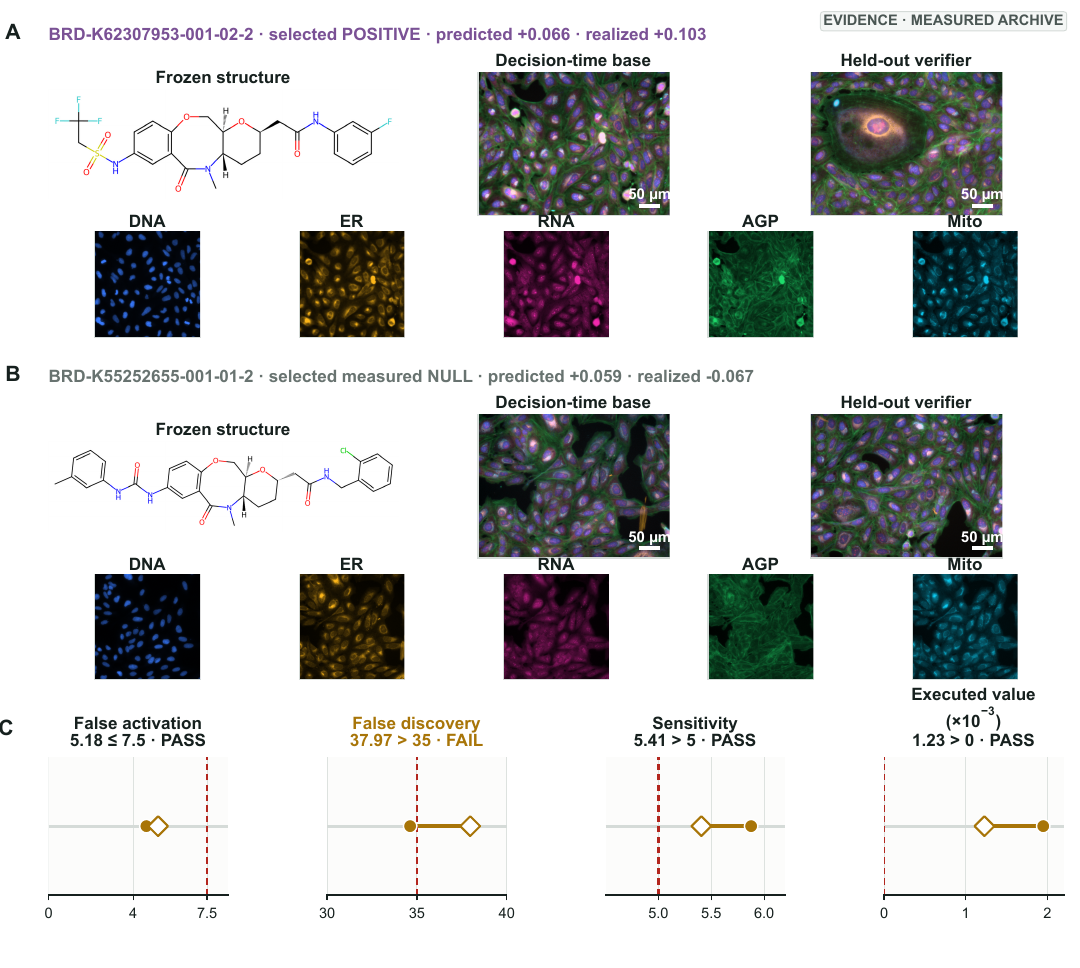}
\caption{\textbf{Archived Cell Painting measurements connect individual assignments to the binding evidence criterion.} (\textbf{A} and \textbf{B}) Frozen structures and site-5 fields for the selected POSITIVE and selected measured-NULL exemplars identified in Fig.~\ref{fig:archive_decisions}A. Each row shows full-field decision-time and held-out-verifier composites, followed by a centered 680-by-680-pixel detail crop from each decision-time channel (DNA, blue; endoplasmic reticulum, yellow; RNA, magenta; actin/Golgi/plasma membrane, green; and mitochondria, cyan). For display only, intensities were clipped to pooled channel-specific 0.20th and 99.80th percentiles calculated across the 40 bound site-5 files, scaled to [0,1], and transformed by the same fixed gamma exponent, 0.82, within every channel; scale bars, 50~\textmu m. The conspicuous ring-like object in the POSITIVE verifier is retained from the archived raw field without digital removal; it is displayed as observed morphology and is not assigned a mechanistic interpretation. The examples are descriptive and have no inferential role. The displayed raw fields were not direct model inputs; the rule used archive-supplied per-well morphology profiles. (\textbf{C}) Four prespecified criteria on their native scales. Filled points are estimates, open diamonds one-sided bounds and red lines criteria; every panel states PASS or FAIL explicitly, and false-discovery uncertainty was the sole binding failure.}
\label{fig:cpg_target_authorization}
\end{figure}

The NULL denominator was evaluated separately for each frozen rule. Because
the assignments and least-favorable treatment of missing outcomes were
rule-specific, it contained 4,569 compounds for the expected-net-benefit rule
and 4,453 for \OPAL{}; these are distinct estimands rather than inconsistent
counts.

The held-out outcome table included all 263 technical outcome failures. Under the
prespecified worst-case completion, each missing realized gain was set to
\(-1.02\), the theoretical lower bound of the utility scale. Only nine of
these compounds were activated and therefore entered executed value at that
bound; the other 254 followed the fallback and contributed zero under the
frozen assignment. The reported executed value already includes this
conservative completion.

By contrast, \OPAL{} selected a different operating point
(Fig.~\ref{fig:archive_decisions}, B). It activated 595
compounds (5.28\%), corresponding to 1,190 additional wells; the
highest-value rule required 18.2 times as many. \OPAL{} had
positive all-library executed value (mean, \(1.948\times10^{-3}\), with a nominal
one-sided 95\% BCa lower endpoint, \(1.229\times10^{-3}\)). Its
false-activation upper bound was 5.18\%, below the prespecified 7.5\% limit.
The 206 least-favorable false activations accounted for 412 wells, compared
with 8,838 under the expected-net-benefit gate. In a descriptive sweep of the
fixed battery, no other active rule became feasible until the false-activation
upper bound reached 37.94\%, more than five times the prespecified limit. Thus,
\OPAL{} alone occupied the active risk-feasible interval from 5.18\% up to,
but not including, 37.94\%.

The resulting active branch was sparse and low recall. Under the prespecified
all-unit least-favorable accounting, it contained 384 measured POSITIVE
compounds among 595 activations (64.54\%), compared with 6,280 measured
POSITIVE compounds among all 11,265 final units (55.75\%). The latter
denominator includes 254 inactive technical failures under the least-favorable
completion; the measured-only baseline used in the score-sweep diagnostic is
6,280/11,002 (57.08\%). The all-unit comparison corresponded to a descriptive
1.16-fold enrichment, whereas the measured-only comparison was 1.148-fold.
Worst-case POSITIVE sensitivity was \(384/6{,}534=5.88\%\), with a one-sided 95\% lower bound of
5.41\% above the prespecified 5\% non-degeneracy floor. The gate recovered a
non-trivial subset rather than most compounds with positive measured gain.

To connect the aggregate decision to the underlying assay,
Fig.~\ref{fig:cpg_target_authorization}, A and B, show archived Cell Painting fields
for the active POSITIVE and active measured-NULL compounds nearest their
respective group-median gains under one fixed selection rule. These examples
were selected after the locked outcome analysis and are descriptive rather
than an additional inferential sample.

A complete population view showed why the aggregate decision could not be
reduced to rank accuracy. Among 11,002 final compounds with a measured outcome,
frozen predicted gain and realized gain had Pearson \(r=0.090\). The 586
measured compounds selected by \OPAL{} had a higher mean and median gain than
the 10,416 nonselected compounds, but selected POSITIVE and selected
measured-NULL units occupied overlapping predicted-gain ranges
(Fig.~\ref{fig:archive_decisions}, A; Supplementary
Fig.~\ref{fig:cpg_real_data_validation}, A and B). In a
separate post-hoc rank-cutoff sweep, enrichment did not exceed 1.238-fold when
units were ordered by frozen predicted gain or 1.181-fold when ordered by
frozen positive probability. The measured subset selected by the prespecified
multicriterion rule had 1.148-fold enrichment; because that rule was not a
top-\(n\) cutoff, each sweep additionally assumes an externally chosen
acquisition budget and does not satisfy the seven-part contract. The larger
top-\(n\) enrichments therefore motivate future budget-constrained rule
development but do not provide held-out evidence for a new rule
(Supplementary
Section~\ref{sec:cpg_score_sweep}).

However, the complete evidence test kept the fixed plan because active-set false
discovery remained too uncertain. Its point estimate was
\(206/595=34.62\%\), below the internal 35\% target, but the one-sided 95\%
upper bound was 37.97\%. Development placed the selected near-boundary rule just outside the
FDP screen (35.11\% against 35\%); the held-out final partition yielded 37.97\%
while the other six contract components were met. The held-out data therefore identified a non-trivial,
false-activation-controlled branch that met the prespecified nominal
compound-bootstrap value criterion, but not the
precision required to replace the fixed plan
(Fig.~\ref{fig:cpg_target_authorization}, C, and
Table~\ref{tab:cpg_contract_blocks}, A). Full stage results are in Supplementary
Table~\ref{tab:cpg0012_stage_results}.

A post-hoc plate-map block analysis tested sensitivity to dependence in the assay
layout. A bootstrap of 80 non-overlapping four-plate blocks gave pointwise
one-sided 95\% percentile endpoints of 5.65\% for false activation and 37.82\%
for false discovery. The corresponding lower endpoints were 4.86\% for
sensitivity and \(5.319\times10^{-4}\) for executed value.
False activation and value remained favorable, false discovery remained
above its target, and sensitivity fell just below its 5\% floor. A secondary
cluster-BCa calculation also made the value boundary inconclusive. These
descriptive results identify block-disjoint replication as the next evidence
requirement (Table~\ref{tab:cpg_contract_blocks} and Supplementary
Table~\ref{tab:cpg0012_plate_map_sensitivity}; Supplementary
Fig.~\ref{fig:cpg_real_data_validation}, C).

\begin{table}[!t]
\caption{Cell Painting final-contract decision and post-hoc dependence sensitivity.}
\label{tab:cpg_contract_blocks}
\centering
\begin{threeparttable}
\footnotesize
\setlength{\tabcolsep}{3pt}
\textbf{A\quad Prespecified seven-part final contract}\par\smallskip
\begin{tabular*}{\textwidth}{@{\extracolsep{\fill}}L{0.20\textwidth} C{0.18\textwidth} C{0.20\textwidth} C{0.15\textwidth} C{0.10\textwidth}@{}}
\toprule
Component & Observed result & One-sided 95\% endpoint & Criterion & Status \\
\midrule
Activations & 595 & --- & $\geq100$ & \textcolor{copal}{met} \\
Coverage & 5.2818\% & --- & $\geq5\%$ & \textcolor{copal}{met} \\
Outcome missingness & 2.3347\% & --- & $\leq5\%$ & \textcolor{copal}{met} \\
False activation & 4.6261\% & UCB 5.1776\% & $\leq7.5\%$ & \textcolor{copal}{met} \\
False discovery & 34.6218\% & UCB 37.9660\% & $\leq35\%$ & \textcolor{cwarn}{not met} \\
POSITIVE sensitivity & 5.8770\% & LCB 5.4054\% & $\geq5\%$ & \textcolor{copal}{met} \\
Executed value & $1.948\times10^{-3}$ & BCa LCB $1.229\times10^{-3}$ & $>0$ & \textcolor{copal}{met} \\
\bottomrule
\end{tabular*}

\vspace{0.7em}
\textbf{B\quad Fixed-assignment plate-map block sensitivity}\par\smallskip
\begin{tabular*}{\textwidth}{@{\extracolsep{\fill}}L{0.19\textwidth} C{0.15\textwidth} C{0.15\textwidth} C{0.14\textwidth} C{0.11\textwidth} C{0.10\textwidth}@{}}
\toprule
Endpoint & Primary compound-level analysis & Block percentile & Block BCa & Criterion & Status \\
\midrule
False-activation UCB & 5.1776\% & 5.6525\% & 5.6923\% & $\leq7.5\%$ & \textcolor{copal}{met} \\
False-discovery UCB & 37.9660\% & 37.8165\% & 37.8739\% & $\leq35\%$ & \textcolor{cwarn}{not met} \\
Sensitivity LCB & 5.4054\% & 4.8592\% & 4.9246\% & $\geq5\%$ & \textcolor{cwarn}{not met} \\
Executed-value LCB & $1.229\times10^{-3}$ & $5.319\times10^{-4}$ & $-9.083\times10^{-5}$ & $>0$ & mixed \\
\bottomrule
\end{tabular*}
\begin{tablenotes}[flushleft]
\item Values in \textbf{A} are the binding final-contract results. The executed-value row gives the fixed-archive point estimate after the prespecified \(-1.02\) worst-case completion and the nominal BCa endpoint used by the prespecified criterion. A more conservative empirical-Bernstein sensitivity analysis gave an LCB of \(-1.099\times10^{-3}\). Values in \textbf{B} are post hoc, pointwise sensitivity analyses of the same frozen assignments over 80 non-overlapping four-plate blocks. ``Mixed'' means that the percentile endpoint remained positive while the block BCa endpoint crossed zero. UCB denotes upper confidence bound, LCB lower confidence bound and BCa bias-corrected and accelerated bootstrap.
\end{tablenotes}
\end{threeparttable}
\end{table}

The sparse branch was also less exposed to measurement cost. In a post hoc
sensitivity analysis that held all six pre-outcome assignments and the
prespecified \(c=0.02\) outcome strata fixed, added wells fell from 22,530 for
forced activation to 1,190 for \OPAL{}. Across the same sequence, the cost at
which the nominal value lower endpoint reached zero rose monotonically from
0.02206 to 0.04320. \OPAL{} was the only active rule with a positive lower
endpoint between approximately 0.03997 and 0.04320 (Supplementary
Fig.~\ref{fig:cpg_cost_sensitivity} and Supplementary
Fig.~\ref{fig:cpg_real_data_validation}, D).

Cell Painting therefore resolved a lower-burden selective candidate rather
than a rule ready to replace the fixed plan. The candidate combined controlled false activation and
positive nominal compound-bootstrap executed value, but false-discovery uncertainty and
the plate-map sensitivity results kept the fixed plan in place. The next step is to
improve active-set enrichment and evaluate the revised rule on new held-out,
preferably block-disjoint, data.

\FloatBarrier
\subsection*{Formal boundaries and executed comparisons define when adaptive measurement is justified}
\label{sec:r2_boundary}

The Cell Painting decision raises a broader question: can a revised model
change the evidence state, or does the design itself limit what can be
established? Two formal boundaries and an executed-value reconstruction
separate model-dependent limitations from constraints imposed by the evidence
design. The first boundary concerns transfer into a target setting whose
outcomes have not been observed.

Specializing established domain-adaptation and transportability impossibility
results to this authorization setting, we show a limit of source-only
evidence. Under unrestricted conditional
outcome shift, the same source outcomes, unlabeled target covariates and frozen
gate can be compatible with favorable and adverse target laws. A source-only
certificate with uniform level-\(\alpha\) false-authorization control can
therefore authorize with probability at most \(\alpha\) under an
observationally indistinguishable favorable law. The full statement and proof
are given in Supplementary Proposition~\ref{prop:target_shift}.
Target-labeled development followed by one untouched target evaluation makes
the selected rule testable, although it does not guarantee that the contract
will pass (Supplementary Section~\ref{sec:target_calibrated_recovery}).

Target labels resolve identification, but they cannot add information about a
finite campaign's unobserved units. For a binary best-case boundary, let \(X\)
be the actionable count in a simple-random pilot of size \(q\), let \(d=N-q\),
and require the actionable fraction in the specific unobserved complement to
exceed \(\rho\in[0,1)\). Under the boundary-null composition, the probability
of the most favorable observation \(X=q\) is

\[
h(N,q,\rho)=
\frac{\binom{q+\lfloor\rho(N-q)\rfloor}{q}}{\binom{N}{q}}.
\]

This probability gives an exact decision boundary. A non-randomized,
monotone level-\(\alpha\) certificate is nontrivial at look \(q\) if and only if
\(h(N,q,\rho)\leq\alpha\). If \(h(N,q,\rho)>\alpha\), even an all-actionable
pilot cannot be certified. For a prespecified look set \(\mathcal Q_N\), the
minimum qualifying pilot size is
\[
q^\dagger(N,\rho,\alpha)=
\min\{q\in\mathcal Q_N:h(N,q,\rho)\leq\alpha\}
\],
which is undefined when \(\mathcal Q_N\) contains no qualifying \(q\).

The complete statement, randomized extension, proof and large-\(N\) limit are
given in Supplementary
Proposition~\ref{prop:impossibility}. This result differs from classical
acceptance sampling, which concerns an entire finite lot
~\cite{nist_acceptance_sampling}. Successes already observed in the pilot
cannot also be counted in the complement that would receive the additional
measurements. Standard concentration bounds for sampling without replacement
address deviations of finite-population averages rather than this
complement-specific certification question~\cite{serfling1974}.

At the prespecified \(\alpha=0.02/7\), the landscape of qualifying designs is
shown in Fig.~\ref{fig:n_alpha_boundary}, A. No pilot size qualified for a
strict-majority complement claim when \(N=16\)
(Fig.~\ref{fig:n_alpha_boundary}, B). The first campaign size not excluded by
the noiseless best-case boundary was \(N=25\), with \(q=14\)
(Fig.~\ref{fig:n_alpha_boundary}, C). These are abstract campaign-unit
calculations for the stated binary complement claim, not sample-size estimates
for Cell Painting, LINCS--LJP or CTRP. Applying the boundary to an assay would
require defining its unit, binary target, pilot design and error allocation
before outcomes. Thus, unidentified target outcomes
require labels or transport assumptions, whereas an identified but unsupported
complement requires a larger or more informative campaign design.

\begin{figure}[!htbp]
\centering
\mainfigure{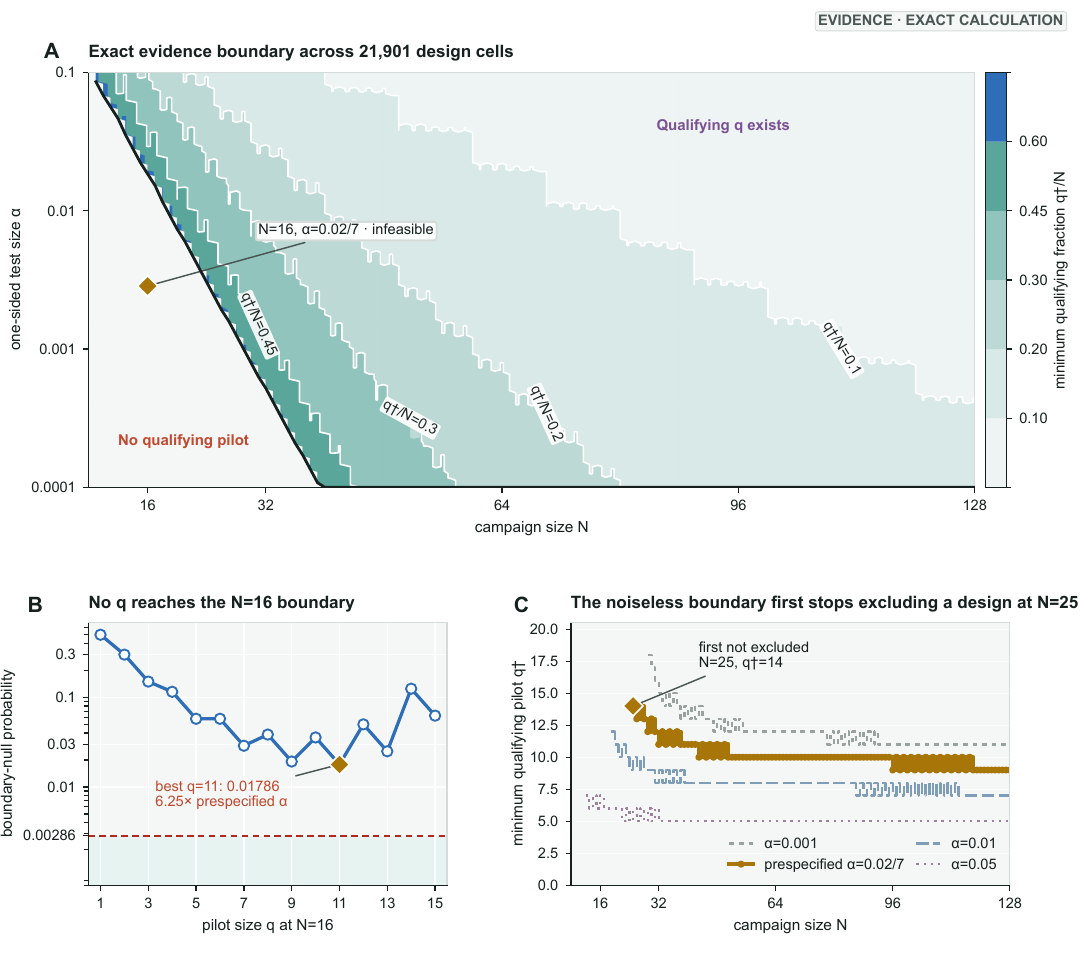}
\caption{\textbf{Finite experiment size imposes an exact methodological evidence boundary.} (\textbf{A}) Exact certifiability landscape for the minimum qualifying pilot fraction \(q^\dagger/N\) over campaign size \(N\) and one-sided test size \(\alpha\) for a complement target \(\rho=1/2\); the prespecified design is marked directly. (\textbf{B}) Exact boundary-null probability for every pilot size at \(N=16\), none of which meets the prespecified \(\alpha=0.02/7\). (\textbf{C}) Integer-valued step functions for the minimum qualifying pilot size across campaign sizes. The emphasized curve is the prespecified \(\alpha=0.02/7\) boundary; \(\alpha=0.001\), 0.01 and 0.05 are shown for context. The stair steps reflect the integer-valued \(q^\dagger\), not simulation noise. The first design not excluded by the prespecified noiseless best-case boundary is \(N=25,q=14\); this necessary boundary uses abstract campaign units and is not an archive-specific sample-size prescription.}
\label{fig:n_alpha_boundary}
\end{figure}

\FloatBarrier
\paragraph{Executed baselines and information cost determine realized value.}
\label{sec:r1_frontier}

Evidence size was not the only determinant. An apparent gain can reverse when
the candidate selection rule is compared with the
fallback that will actually be executed. A label-based surrogate suggested a
\(5.549\times10^{-3}\) improvement against its score-construction reference,
\(K_{\mathrm{score},-f}=1\). Executing the same rule against the best fixed
action, \(K_0=2\), changed the contrast to \(-8.14\times10^{-4}\)
(Fig.~\ref{fig:selector_frontier}, C, and Supplementary
Fig.~\ref{fig:ed_estimand}; Supplementary
Table~\ref{tab:sr_gap_decomposition}). Baseline identity supplied the dominant correction
(\(-12.862\times10^{-3}\)). Value-source and execution terms completed the
accounting identity. Across all 720 frozen banks, 487 contrasts were exactly
unchanged, 159 favored the fixed fallback and 74 favored the selector; their
mechanism-cell structure and complete ranked distribution are shown in
Fig.~\ref{fig:selector_frontier}, A and B. The surrogate and executed contrasts
therefore answered different decisions. The complete factorial ablation and
estimand crosswalk are in Supplementary Tables~\ref{tab:sr_factorial}
and~\ref{tab:ed_estimands}.

\begin{figure}[!htbp]
\centering
\mainfigure{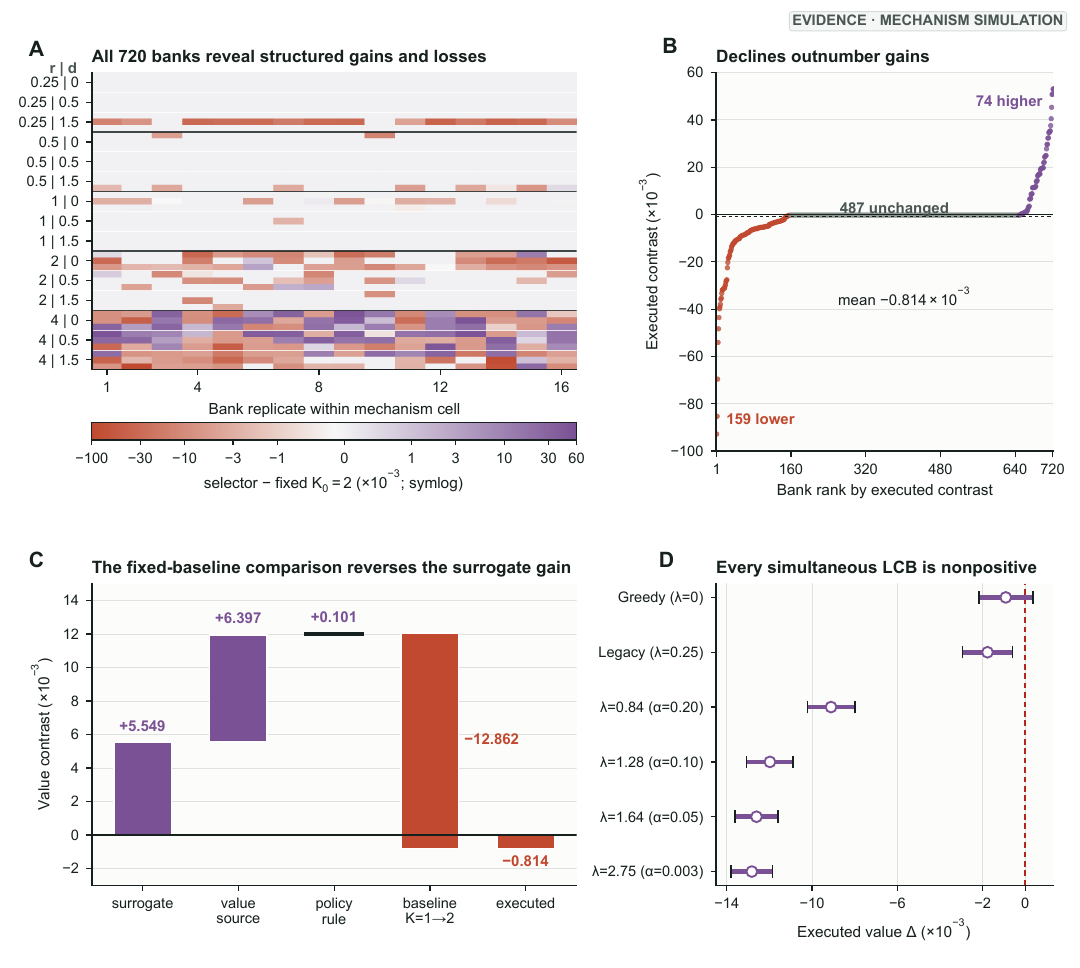}
\caption{\textbf{Execution against the actual fallback reverses the surrogate gain.} (\textbf{A}) Complete 45-mechanism-cell by 16-bank matrix of selector-minus-\(K_0\) executed contrasts. Cells are ordered first by execution-to-process-noise ratio \(r\in\{0.25,0.5,1,2,4\}\), then by normalized drift \(d\in\{0,0.5,1.5\}\), and then by substitutability \(s\in\{0,0.5,0.9\}\); within each \(r\mid d\) block, the three rows correspond to \(s=0,0.5,0.9\). Separators mark the \(d\) and \(r\) blocks. A symmetric-log color scale preserves exact zeros and both tails. (\textbf{B}) All 720 contrasts ranked without jitter: 487 were unchanged, 159 favored \(K_0\) and 74 favored the selector; the overall mean was \(-0.814\times10^{-3}\). (\textbf{C}) Accounting waterfall for the exact underlying reconciliation, with displayed labels rounded: the label-based surrogate contrast, \(5.549\times10^{-3}\), is updated by value-source, policy-rule and comparator corrections to the executed contrast against \(K_0=2\), \(-0.814\times10^{-3}\). The small \(+0.101\times10^{-3}\) policy-rule correction is labeled directly without changing the plotting scale. (\textbf{D}) Simultaneous max-\(t\) intervals for six frozen thresholds. Open circles show point contrasts, horizontal bars show simultaneous intervals and the red dashed line marks zero. The four \(\alpha\)-labeled rows use their corresponding one-sided critical values as \(\lambda\). Only \(\lambda=0\) overlapped zero; every lower bound was nonpositive.}
\label{fig:selector_frontier}
\end{figure}

None of the six thresholds applied to the common cross-fitted score had a
simultaneous lower bound above zero. Every rule therefore selected \(K_0\)
(Fig.~\ref{fig:selector_frontier}, D, and Supplementary
Table~\ref{tab:sr_frontier_full}). A separate
full-state analysis nevertheless found a positive opportunity envelope of
\(2.690\times10^{-3}\). Heterogeneous action value was therefore present even
though the learned selector did not convert it into executed value over the
fallback (Supplementary Section~\ref{sec:sr_frontier}). The available analysis could not distinguish whether the remaining
gap arose from decision-time information, the representation or model error.

\paragraph{Probe and expert diagnostics.}
\label{sec:r2_option}

The architecture analysis localized the same value loss. Cross-validation
selected no residual expert in any fold, and only 10 of 810 banks purchased a
probe. Only two probes changed capacity. After information cost, the realized probe
contribution was \(-5.86\times10^{-4}\), with a simultaneous interval ending
at zero (Supplementary Fig.~\ref{fig:ed_architecture}).
The probe ledger added 90 OOD banks to the 720 in-support banks used for the
executed-value reconciliation, yielding 810 diagnostic banks; the two totals
therefore describe different ledgers rather than missing observations.

\FloatBarrier
\paragraph{Finite evidence separates opportunity from certifiability.}\label{sec:r3_finite}
Positive expected opportunity did not ensure that a finite pilot could justify
the claim about the units that would receive the additional measurements. We inverted the
complete three-class likelihood for every prespecified pilot size, noise level
and signal-count vector in the 16-bank campaign. Across all 1,113
configurations, a least-favorable campaign composition yielded non-positive
action gain. No observation was actionable, the uncosted envelope reached only
zero and the costed envelope remained negative
(Supplementary Fig.~\ref{fig:finite_campaign_diagnostics}, B, and Supplementary
Table~\ref{tab:full_likelihood_certificate}).

The exact boundary in Fig.~\ref{fig:n_alpha_boundary}, B, explains this result.
At \(N=16\), even an all-actionable pilot can leave every counterexample in the
unobserved complement. Adding noise and cost through the full three-class
likelihood cannot remove that best-case obstruction. A more elaborate score on
the same evidence design therefore cannot cross the boundary without changing
the target claim, the certificate class or its assumptions. Exact all-success
enumeration and operating-characteristic denominators are reported in
Supplementary Tables~\ref{tab:sr_n16_enumeration}
and~\ref{tab:sr_finite_oc}.

An independent enumeration verified exact coverage over all finite-support
campaign--pilot pairs. The rule never activated in 7,442 sequential campaigns,
and every prespecified look remained below zero (Supplementary
Figs.~\ref{fig:finite_campaign_diagnostics}, A and B, and
\ref{fig:finite_campaign_values}, A).
Across 104,188 selected and oracle campaign--look values, oracle gross gain
never exceeded 0.05. In contrast, campaign-normalized information cost rose
from 0.06875 at \(q=2\) to 0.44375 at \(q=14\)
(Supplementary Figs.~\ref{fig:finite_campaign_values}, B;
\ref{fig:ed_coverage}; and~\ref{fig:ed_finite_economics}). Full-state and
deployable opportunity were positive in 5,310 and 3,664 final-look campaigns,
respectively, leaving 1,646 campaigns with positive opportunity that the
prespecified evidence could not deploy (Supplementary
Fig.~\ref{fig:finite_campaign_values}, A, inset). Adaptive opportunity can therefore exist
even when a small campaign cannot justify the required claim about its
unobserved complement. Crossing this boundary requires more informative
evidence, a different target claim or additional justified structure.

\FloatBarrier
\subsection*{Simulation shows that risk control can preserve a selective branch}\label{sec:r4_risk}
Risk control is useful only if it can preserve selective action when evidence
warrants it. A held-out mechanism simulation tested this property independently
of the biological archives. A NULL-only tuning set selected
\(\tau=0.0180806\), which was fixed and evaluated once on an independent risk
lock (Supplementary
Fig.~\ref{fig:ed_d6_calibration_heterogeneity}, A and B).

On 900 previously unseen mechanism fixtures, the gate activated none of 400
NULL fixtures. Its one-sided 99.5\% false-activation upper bound was 1.316\%,
below the prespecified 5\% limit. It also activated 342 of 500 positive fixtures,
giving 68.4\% sensitivity (95\% confidence interval, 64.1--72.5\%,
Supplementary Fig.~\ref{fig:risk_validation}, A to C). In this population, false-activation
control preserved a substantial active branch rather than defaulting to the
fallback everywhere.

The gate changed the failure profile without a detectable increase in mean
value over forced activation. Forced activation was harmful in 13.6\% of
positive fixtures. Gating removed that negative tail while declining some
small positive effects. The paired mean difference was
\(8.60\times10^{-5}\) (retrospective 95\% interval,
\(-6.78\times10^{-4}\) to \(8.75\times10^{-4}\),
Supplementary Fig.~\ref{fig:risk_validation}, D). Post-result threshold ablations increased
sensitivity, but the most permissive rule exceeded the 5\% risk limit
(Supplementary Fig.~\ref{fig:ed_d6_comparators} and Supplementary
Table~\ref{tab:sr_d6_comparator_battery}; primary results and value
decomposition are in Supplementary Tables~\ref{tab:sr_study3_results}
and~\ref{tab:sr_d6_value}; the complete prespecified validation summary is
in Supplementary Table~\ref{tab:ed_d6}). These descriptive comparisons
trace the trade-off between sensitivity and NULL activation. The held-out
result belongs to the frozen threshold.

The campaign score \(S^{\mathrm{eval}}\) used common-random-number potential
outcomes and served as an evaluation statistic rather than a decision-time
score. These results establish nontrivial risk-controlled behavior under
constructed conditions. The separate Causal Chambers analysis tested whether
an archive-backed proxy could be computed from measured trajectories
(Supplementary Section~\ref{sec:sr_study4p}; archive and loss summaries in
Supplementary Tables~\ref{tab:sr_study4p_datasets}
and~\ref{tab:sr_study4p_loss}).

\FloatBarrier
\subsection*{LINCS--LJP and CTRP resolve broad acquisition and fallback}
Cell Painting resolved a sparse candidate without changing the fixed plan. The
two external archives asked whether the same held-out decision layer would
instead support broad measurement or fallback. They resolved those states from
measured biological data.

\subsubsection*{LINCS--LJP indicates broad acquisition in the observed setting}
\label{sec:r7_lincs}

LINCS--LJP supplied a temporally ordered transcriptomic measurement decision
using public LINCS L1000 profiles and matched long-term perturbation outcomes
~\cite{subramanian2017l1000,niepel2017lincs}. A 3-h profile supplied the
decision-time features, a 24-h profile was the optional acquisition and a 72-h
growth-rate response supplied the payoff-relevant endpoint. Complete matching
yielded 4,049 complete-case episodes from 105 perturbation-identity components. Components, rather
than individual episodes, were assigned to disjoint LEARN, CALIBRATION,
CERTIFICATION and FINAL partitions of 42, 21, 21 and 21 components,
respectively. Thus, no perturbation component crossed an outcome-evaluation
boundary. All component-equal summaries gave every perturbation component the
same weight regardless of its number of episodes (Supplementary
Section~\ref{sec:sr_lincs}).

The initially issued confidence-bound workflow produced no eligible
calibration threshold. After CALIBRATION, FDP acceptance was relaxed from a
95\% upper bound no greater than 35\% to a component-equal point estimate no
greater than 45\% before CERTIFICATION. The original 35\% limits had markedly
different finite-sample resolution across
the two archives, whose inferential units and bounds also differed: the Cell
Painting Clopper--Pearson rule admitted at most 188 false activations among 595
selections (31.60\%), whereas the LINCS KL--Hoeffding rule over 21 equal-weight components
required a component-equal point FDP no greater than 11.675\%. Under the
amended criterion, calibration established a fixed threshold of \(-0.06\) from
29 candidates. The selected rule acquired the 24-h profile for 706
calibration episodes across all 21 components, corresponding to a
component-equal activation fraction of 94.07\%. Its component-equal
false-discovery proportion (FDP) was 42.67\%, and its normalized mean net value per episode after the
prespecified acquisition cost was 0.03704. CERTIFICATION met the point-FDP,
support and technical screens but not the original positive value-LCB
criterion. With 21 components and the prespecified four-unit range,
the one-sided Hoeffding penalty was 1.068, so a positive lower bound required a
component-equal mean above 1.068; the observed certification mean was 0.04358.
Value acceptance was then relaxed from a positive lower bound to a positive
component-equal mean before FINAL, while all support and influence screens were
retained. Accordingly, CALIBRATION and CERTIFICATION contributed to rule
development, and only FINAL evaluated the complete amended rule. The 95\% FDP
upper bounds and value lower bounds remained uncertainty summaries rather than
acceptance criteria.

The rule remained broadly active in the certification aggregate and the
untouched final partition (Fig.~\ref{fig:lincs_ljp}, A and B). It selected all 790 certification episodes
and all 860 FINAL episodes, with no unevaluable selected episode and all 21
components represented in each partition. The 94.07\% calibration activation
and 100\% FINAL activation were measured on different component-disjoint
partitions under the same frozen threshold; the rule itself did not change.
The component-equal FDP point
estimate was 35.55\% in certification and 39.92\% in FINAL, below the 45\%
criterion. Mean net value point estimates were positive in both partitions (0.04358 and 0.04295,
respectively). In FINAL, leave-one-technical-group-out values ranged from
0.03968 to 0.04615 across 12 cell-line--batch groups, and the largest absolute
group contribution share was 0.1647. The positive mean was therefore
distributed across the observed groups rather than concentrated in one of
them. Transport beyond the six cell lines and two batches remains untested
(Supplementary Table~\ref{tab:lincs_ljp_stage_results}).

\begin{figure}[!htbp]
\centering
\mainfigure{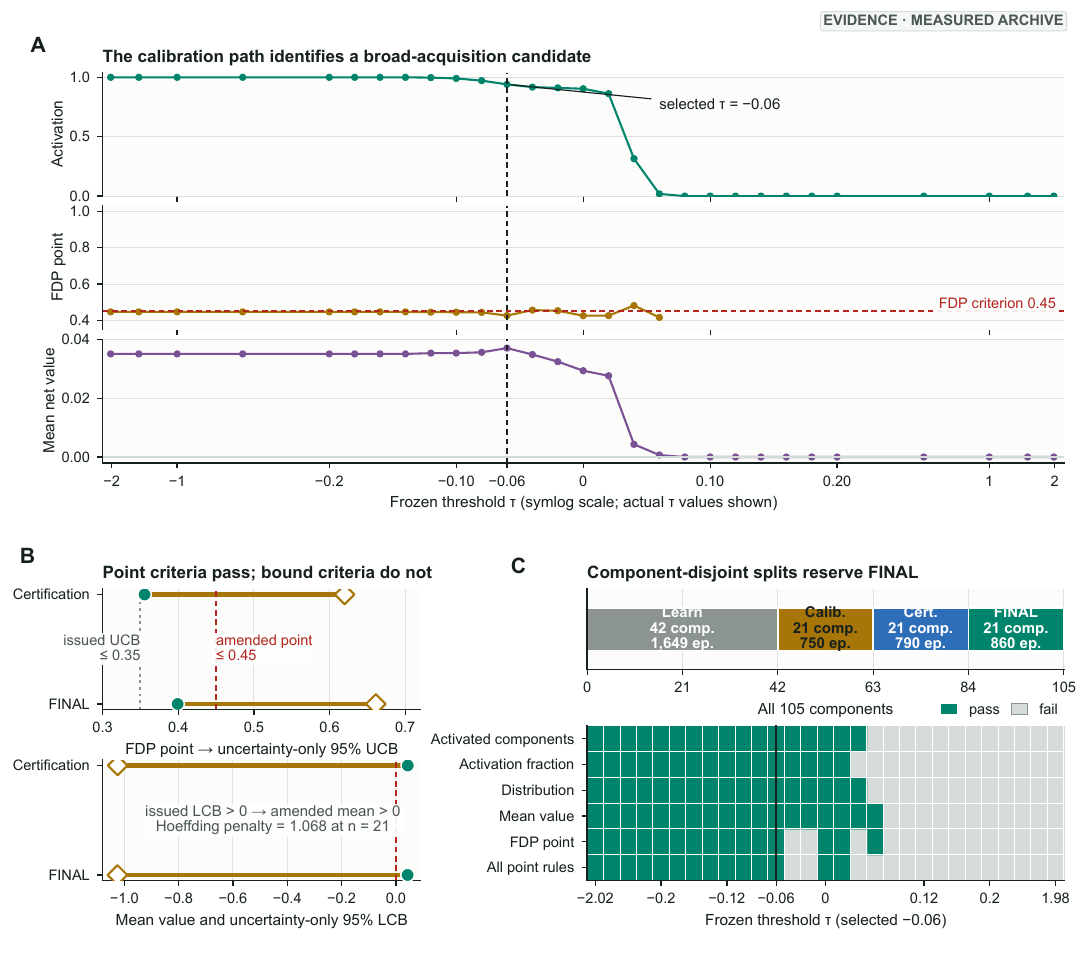}
\caption{\textbf{LINCS--LJP favors broad acquisition under development-set point criteria.} (\textbf{A}) All 29 calibration candidates plotted against their actual frozen threshold values on a symmetric-log axis, with activation, FDP and component-equal mean net value on native vertical scales. The black dashed vertical line marks the selected threshold, \(\tau=-0.06\), and the red dashed horizontal line marks the amended FDP point criterion, 0.45. FDP is omitted for the 11 candidates with zero activation because no FDP point estimate exists; their stored software sentinel is not plotted. Aggregate mean net value remains plotted at zero because it is defined over all episodes relative to fallback. (\textbf{B}) Component-equal certification and FINAL point estimates paired with uncertainty-only 95\% bounds. Circles mark the point statistics used by the complete amended rule; diamonds mark the original bound statistics. FDP point acceptance was fixed after CALIBRATION and before CERTIFICATION, and mean-value acceptance after CERTIFICATION and before FINAL. All 790 certification and 860 FINAL episodes were selected, and both mean-value point estimates were positive. The Hoeffding value bounds use the full four-unit prespecified range across 21 components, explaining their approximately \(-1.025\) endpoints; these and the FDP bounds remained uncertainty summaries. (\textbf{C}) Aggregate component-disjoint 42/21/21/21 split of all 105 components and 4,049 complete episodes, followed by five predicate rows and their all-point-rule conjunction for all 29 candidates; green and gray cells denote pass and fail.}
\label{fig:lincs_ljp}
\end{figure}

The uncertainty intervals were substantially wider than the acceptance region:
the 95\% FDP upper bound was 61.99\% in certification and 66.11\% in FINAL,
while the corresponding value lower bounds were negative. The Hoeffding value
bound uses the full prespecified range from \(-2.02\) to \(1.98\) across only 21
components, subtracting about 1.068 from each observed mean; this construction
explains why both lower bounds lie near \(-1.025\). The evidence
therefore justifies the stated point-estimate decision, not confidence-bound
control of FDP or a confidence-guaranteed positive value. The final rule also
acquired every eligible 24-h profile. Because it selected 790/790
CERTIFICATION episodes and 860/860 FINAL episodes, neither partition tests
within-partition ranking or selective savings. The positive mean point estimate
is therefore evidence for broad acquisition in this archive under the complete
amended rule, not evidence that a selective policy reduces measurement
burden.

\subsubsection*{CTRP falls back because the frozen score does not rank measured opportunity}
\label{sec:r6_ctrp}

CTRP v2 tested whether a frozen observable score could recover measured
opportunity in a family-disjoint archive of cancer-cell-line responses to
small-molecule perturbations
~\cite{seashoreludlow2015}. Before response access, we fixed a four-compound
panel and actions \(K\in\{1,2,4\}\), where \(K\) denotes the size of an offline
response bundle ranked from development data. Every eligible family had
outcomes for every compound, so each action value was a measured lookup rather
than an imputed counterfactual. Families were partitioned into development,
calibration and locked evaluation (Supplementary Table~\ref{tab:ctrp_actions} and Supplementary
Section~\ref{sec:sr_ctrp}).

The unconstrained threshold exceeded the 5\% NULL-risk limit, whereas
\(\tau=0.5\) passed calibration and was fixed before locked responses were
opened. Every locked score then fell below the threshold
(Fig.~\ref{fig:ctrp}, A). The final rule made
no false activations among the 244 NULL families in the 254-family locked set
(binding one-sided 99.5\% upper
bound, 2.15\%). However, it also activated none of the seven POSITIVE families
and had zero executed gain. More sharply, all seven POSITIVE families had a
frozen score of exactly zero, whereas every family with a nonzero frozen score
was NULL (Supplementary
Fig.~\ref{fig:ed_ctrp}, C to F).
The complete calibration and locked-evaluation results are reported in
Supplementary Table~\ref{tab:ctrp_locked_results}.

The absence of activation did not indicate an absence of pharmacological
opportunity. The four frozen compound structures and all 433 archive-reported
dose-level averages from 28 curves across the seven POSITIVE families are shown
in Fig.~\ref{fig:ctrp}, B and C. These complete measured dose--response
summaries showed mean panel-relative opportunity of 0.2054, whereas
the score--opportunity rank correlation was only \(9\times10^{-3}\)
(Supplementary Fig.~\ref{fig:ed_ctrp}, D to F). Recent analyses of cancer
drug-response benchmarks further emphasize the need to separate entities and
lock preprocessing and outcomes before evaluation
~\cite{asiaee2026cancerleakage}. Family-disjoint partitions, a frozen
four-compound panel and complete measured outcomes isolated the measurement
decision. The source-trained score
therefore failed to rank measured opportunity in the target strata. In this
archive, fallback was the indicated operational state.

\begin{figure}[!htbp]
\centering
\mainfigure{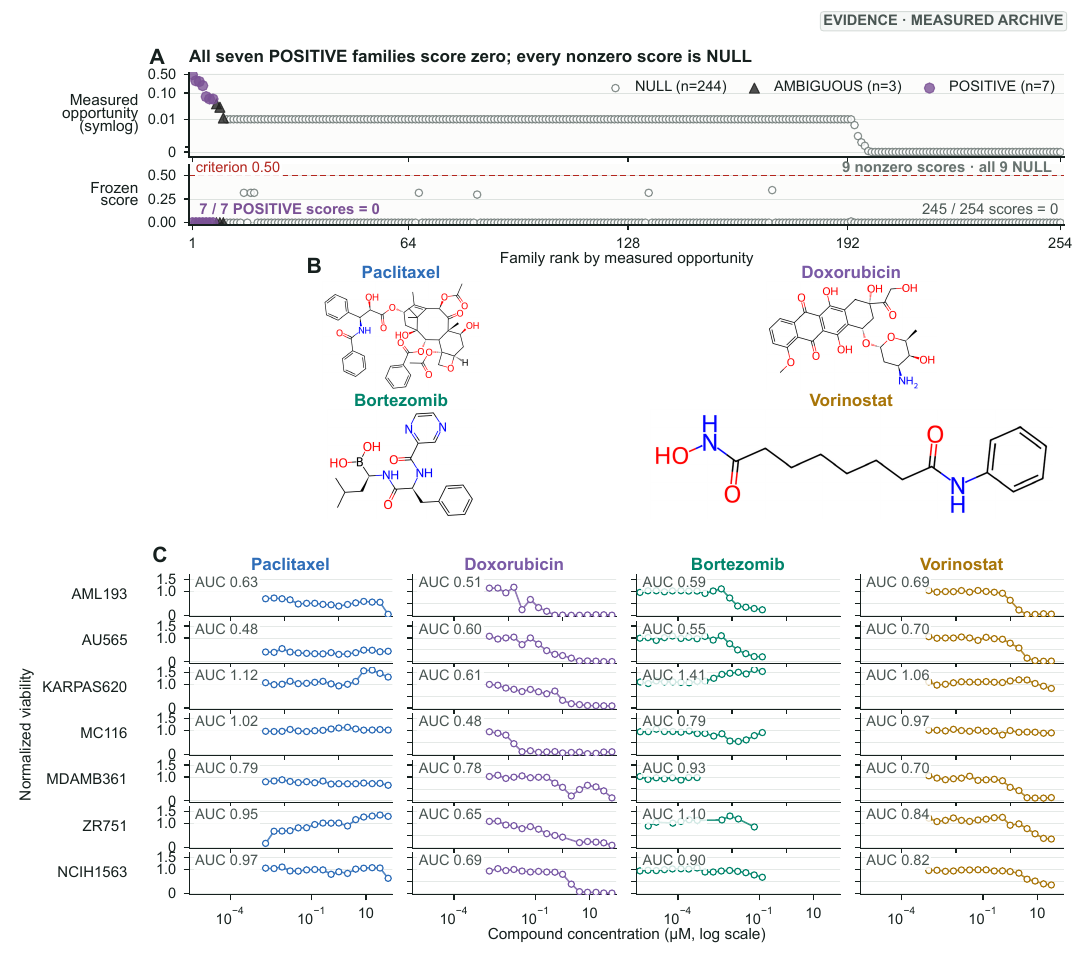}
\caption{\textbf{CTRP falls back because the frozen score does not rank measured opportunity.} (\textbf{A}) All 254 locked families ordered by measured opportunity, displayed on a symmetric-log vertical scale, with the frozen observable-only score aligned to the same family order. All seven POSITIVE families had score zero, whereas all nine families with a nonzero score were NULL; no score reached the fixed \(0.5\) threshold. (\textbf{B}) The four frozen compound structures, redrawn from staged SMILES in a two-by-two layout; compound colors are carried into panel C. (\textbf{C}) All 433 archive-reported dose-level averages from 28 compound--cell-line curves: four compounds by seven named POSITIVE families. Lines connect archive measurements as visual guides and are not fitted dose--response models; panel labels report normalized archive AUC values.}
\label{fig:ctrp}
\end{figure}

Together, the archive-specific evaluations resolve three empirical evidence
states with different experimental consequences rather than a common
performance ranking
(Table~\ref{tab:empirical_regimes}).

\begin{table}[!htbp]
\caption{Three measured archives resolve different evidence states and next-study implications.}
\label{tab:empirical_regimes}
\centering
\begin{threeparttable}
\footnotesize
\setlength{\tabcolsep}{3pt}
\renewcommand{\arraystretch}{1.10}
\begin{tabularx}{\textwidth}{@{}L{0.10\textwidth}L{0.12\textwidth}C{0.09\textwidth}L{0.18\textwidth}L{0.17\textwidth}Y@{}}
\toprule
Archive & Evaluation units & Selected & Acquisition-error result & Value result & State and next study \\
\midrule
Cell Painting & 11,265 compounds & 595 (5.28\%) & False-activation UCB 5.18\%; FDP 34.62\% (UCB 37.97\%) & Mean $1.948\times10^{-3}$; BCa LCB $1.229\times10^{-3}$ & Sparse candidate; revise and retest \\
LINCS--LJP & 860 FINAL episodes; 21 components & 860 (100\%) & Component-equal FDP 39.92\% (UCB 66.11\%) & Component-equal mean 0.04295; LCB $-1.02534$ & Broad acquisition; no selective saving \\
CTRP v2 & 254 families & 0 (0\%) & 0/244 NULL activated; 99.5\% UCB 2.15\% & Executed gain 0 & Fallback; improve score \\
\bottomrule
\end{tabularx}
\begin{tablenotes}[flushleft]
\item UCB denotes a one-sided 95\% upper confidence bound unless stated otherwise; LCB denotes lower confidence bound. ``Sparse candidate'' denotes the Cell Painting branch whose false-discovery uncertainty kept the fixed plan in place; its next test requires better active-set enrichment and a new, preferably block-disjoint, held-out evaluation. ``Broad acquisition'' denotes positive point-estimate value with all FINAL episodes selected; it does not imply selective measurement saving or confidence-bound control. ``Fallback'' denotes the locked CTRP rule, for which every score remained below threshold; a revised score must first demonstrate opportunity ranking. The amended LINCS--LJP point criteria used FDP $\leq0.45$ and mean value $>0$; its UCB and LCB were uncertainty summaries, not acceptance criteria. The FDP point criterion was set after CALIBRATION and before CERTIFICATION, and the value-mean criterion after CERTIFICATION and before FINAL; only FINAL evaluated the complete rule. Each state and implication is limited to its archive, endpoint, cost definition and split. Metrics are archive-specific and should not be compared as common biological effect sizes.
\end{tablenotes}
\end{threeparttable}
\end{table}

The formal and simulated studies answer a separate methodological question.
They define identification and finite-sample boundaries and show, under
constructed known-truth conditions, that risk control can preserve a
nontrivial active branch.

\FloatBarrier
\section*{DISCUSSION}\label{sec:discussion}

Across held-out final partitions from three biological archives, evaluation of frozen learned rules
resolved three evidence states. Cell Painting produced a lower-burden sparse
candidate with a positive mean and nominal compound-bootstrap lower endpoint,
but its false-discovery bound kept the fixed plan, and two sensitivity bounds
on value crossed zero. LINCS--LJP paired a positive mean point estimate with
acquisition of every final episode, indicating broad acquisition in the
observed setting rather than selective savings. CTRP fell back because
its frozen score did not rank measured opportunity. None of the three archive
evaluations therefore justified a measurement-saving selective policy.
Instead, the states imply different next experiments: improve and re-evaluate
the Cell Painting candidate, budget broad LINCS--LJP acquisition in the
observed setting, or revise the CTRP score before testing another selector.

Cell Painting shows why value and an evidence-qualified plan change can
diverge. Aggregate
expected value rewards gains accumulated across every selected compound,
whereas measurement burden accumulates with each added well. The highest-value
rule assigned 21,630 additional wells and re-imaged 96.01\% of the library. By
contrast, the sparse candidate assigned 1,190 wells, an 18.2-fold lower burden,
and met the nominal compound-bootstrap value criterion. The highest-value rule also had a 97.14\%
false-activation upper bound, while no competing active rule fell below a
37.94\% bound. Held-out qualification is therefore a different experimental target
from maximizing aggregate value.

Cell Painting was a confirmatory test of a development-selected near-boundary
rule, not an unexpected failure: none of 53,792 threshold combinations met all
seven development screens, and the selected rule entered final evaluation with
an FDP upper bound of 35.11\% against the 35\% screen. However, lower burden and
value positive under the nominal compound-bootstrap criterion were not
sufficient to change the fixed plan. The Cell Painting FDP point estimate was 34.62\%, below the 35\%
target, but its one-sided 95\% upper bound reached 37.97\%. At 595 selections,
the study would have needed at most 188 false activations, meaning that at
least 68.40\% of selections had to be non-false; 65.38\% were non-false under
the prespecified completion. At the same error rate, roughly 43,000 activations would be required before
the bound crossed 35\%. The priority is therefore a better-enriched active set
followed by a new held-out evaluation. Its 5.88\% sensitivity has the same
interpretation: the candidate identifies a high-confidence subset rather than
recovering most compounds with positive measured gain.

The other two archives imply different next experiments. LINCS--LJP points to
budgeting broad 24-h acquisition only within the evaluated endpoint, cost
definition, cell lines and batches; it provides no evidence of selective
measurement saving. In CTRP, the current frozen score should not alter the
fixed plan. The failure was sharper than a weak rank correlation: every
POSITIVE family had score zero, and every nonzero score belonged to a NULL
family. A revised score must first demonstrate opportunity ranking before
another selective acquisition rule is evaluated. These conclusions concern
archive-specific evidence and experimental consequences, not a common scale
of algorithmic success.

The plate-map analysis identifies a second level of Cell Painting evidence.
Compound-level resampling met the stated sensitivity and value criteria, but
resampling 80 assay blocks reduced the sensitivity lower endpoint from 5.41\%
to 4.86\%. The value endpoint remained positive under a pointwise percentile
analysis but became inconclusive under cluster BCa. Compounds and assay blocks
therefore answer different questions. More compounds from the existing layout
can improve precision within those blocks, but they cannot establish transport
to new blocks. That population claim requires block-disjoint development and
evaluation. Related Cell Painting benchmarks document technical variation
across plates, imaging runs, acquisition sites and instruments
~\cite{chandrasekaran2024jump,arevalo2024batch}.

A separate methodological layer explains when these archive-level decisions
can be justified. If the relationship between measurements and outcomes
changes across experimental settings, source outcomes and unlabeled target data
cannot establish whether the rule is suitable in the target setting without
transport assumptions. Target labels resolve that ambiguity, but a finite
campaign may still be too small. The exact finite-population boundary shows
when even an all-actionable pilot cannot justify a claim about its specific
unmeasured remainder. Target evidence addresses identification; experimental
design determines whether the intended use is certifiable.

The synthetic and reconstructed-execution studies establish complementary
method properties. The finite campaign isolates an evidence boundary that no
rule using the same pilot-count evidence, target claim, certificate class and
assumptions can cross. The held-out simulator shows that the risk gate can
preserve a nontrivial selective branch under constructed known-truth
conditions: it activated 342 of 500 positive fixtures and none of 400 NULL
fixtures. Together with the untouched LINCS--LJP FINAL result for broad
acquisition, this shows that the decision layer does not mechanically return
fallback. The complete joint rule has yet to support a measurement-saving
selective policy in measured biology, which now requires evaluation in a new
held-out archive or an instrument-silent prospective study.
The executed-value reconstruction shows
how comparator identity and information cost can reverse an apparent gain.
These are methodological results rather than additional biological acquisition
states.

Prospective adaptive screens provide the complementary success case. BATCHIE
explored approximately 4\% of 1.4 million possible drug-combination
experiments and validated prioritized combinations experimentally
~\cite{tosh2025batchie}. Its information-gain objective differs from the
authorization criteria studied here, but the comparison clarifies the next
evidence layer: a frozen acquisition rule should ultimately be tested through
instrument-executed decisions and measured experimental consequences.

Measurement burden also determines how long a policy remains valuable.
Across the five active Cell Painting rules, added-well burden fell from 22,530
to 1,190 as the zero-crossing cost rose from 0.02206 to 0.04320. Broad rules
lost value faster because each cost increase was charged for every selected
compound (Supplementary Table~\ref{tab:cpg_cost_crossovers}). LINCS--LJP
supplies the complementary empirical case: broad
acquisition yielded a positive component-equal mean point estimate after cost
in the final partition. The resulting design choices are concrete: improve
enrichment, acquire target labels, enlarge or redesign the campaign, acquire a
measurement broadly, or keep \(K_0\).

\OPAL{} complements methods that optimize, constrain or evaluate a
candidate decision rule after adaptive measurement has entered scope. It
addresses the earlier decision to depart from a fixed plan. The same decision
structure arises when a learned rule changes future data before its value is
observed, including high-throughput screens, beamline experiments and synthesis
campaigns. This conceptual scope is broader than the archive-based empirical
tests reported here. It also links experimental design to safe learning and
adaptation across settings. Data sufficient to learn a rule for a new setting are
not automatically sufficient to justify its use.

\paragraph{Study limitations.}
The three biological evaluations use recorded archives rather than prospective
instrument execution. Cell Painting provides one held-out compound partition,
and the block analysis does not establish transport to new assay blocks.
Its utility measures replicate-profile agreement after cost, not hit efficacy,
mechanism-of-action recovery or a biologically calibrated effect; the
\(5\times10^{-3}\) POSITIVE margin is study-specific. The displayed microscopy
examples are descriptive single-site fields, and the influence of the retained
ring-like object on the archive-supplied per-well profile was not separately
evaluated.
LINCS--LJP covers six cell lines and two batches; its grouping diagnostics do
not establish generalization to new cell lines, batches or perturbation
libraries. Its FDP criterion was amended after calibration and before
certification, and its value criterion after certification and before FINAL;
both amendments relaxed bound-based acceptance to point-estimate acceptance.
Only FINAL was untouched under the complete amended rule. Its decision criteria concern point
estimates, while the reported
FDP and value uncertainty bounds crossed 45\% and zero. CTRP evaluates an
author-defined response bundle, and its seven POSITIVE families provide
limited positive-stratum evidence. The Causal Chambers analysis
~\cite{gamella2025} shows that a measurement-rule computation can be
reconstructed from recorded trajectories (Supplementary
Fig.~\ref{fig:ed_constructibility}); it is constructibility evidence rather
than prospective validation. Clinical AI uses an analogous prospective staging
step: silent trials evaluate a model noninterventionally in its intended
setting before its outputs affect care~\cite{tikhomirov2026silent}. An
instrument-silent study could likewise log frozen acquisition decisions
alongside the fixed plan before allowing real-time experimental control.
Future studies should test frozen rules with instrument-executed actions,
block-disjoint data and measured facility costs.
Within these boundaries, \OPAL{} turns predicted acquisition value into an
empirical decision about whether a fixed plan should change. We conclude that
such a change requires evidence from the intended setting, or justified transport
assumptions, rather than an optimizer's recommendation alone.

\section*{MATERIALS AND METHODS}\label{sec:methods}

\subsection*{Study design and acquisition decision}\label{sec:m_problem}

The study separates archive-level evidence states from methodological
analyses. Cell Painting, LINCS--LJP and CTRP determine the state resolved in
each measured archive. The formal results, reconstructed-execution analyses
and simulations evaluate identification, finite-sample constraints and
nontrivial risk control rather than adding further biological states.
Supplementary Table~\ref{tab:notation} defines the core notation and the
evidentiary role of each quantity.

\OPAL{} separates two decisions that are usually combined: which experimental
action a learned rule would take, and whether target evidence justifies
departing from a fixed plan. In the formal development below, a policy is
simply a prespecified mapping from information available at assignment to one
of the allowed experimental actions. Decision unit \(i\) has precommitment information
\(X_i\), an optional probe \(Z_{ir}\), a payoff-relevant state \(M_i\) and an
action \(K\in\mathcal K\). Utility

\[
U_i(K)=Y_i(K)-C_i^{\mathrm{op}}(K)
\]

places scientific return and operational cost on one prespecified scale. The
executed fixed action \(K_0\) is the comparator and fallback. All inputs to the
rule must be available when the decision is made. Potential outcomes and oracle
mechanism variables are used only for evaluation. Calibration and evaluation
use the same scoring function, action semantics and decision-time information,
which we call deployment isomorphism (Supplementary
Section~\ref{sec:deployment_isomorphism}). In this paper, ``prespecified'',
``frozen'' and ``locked'' refer to internally timestamped study artifacts fixed
before the corresponding outcome stage, not to registration in an external
public registry. ``Registry'' denotes an internal record set, not an external
registration service.

For information set \(\mathcal I\), define

\[
V(\mathcal I)=
\mathbb E\!\left[\max_{k\in\mathcal K}
\mathbb E\{U(k)\mid\mathcal I\}\right],
\qquad
V_0=\max_{k\in\mathcal K}\mathbb E\{U(k)\},
\qquad
G(\mathcal I)=V(\mathcal I)-V_0.
\]

We report \(G_X=G\{\sigma(X)\}\),
\(G_{XZ}=G\{\sigma(X,Z_r)\}\) and simulator-only
\(G_M=G\{\sigma(M)\}\). Under payoff sufficiency,
\(\mathbb E\{U(k)\mid M,X,Z_r\}=\mathbb E\{U(k)\mid M\}\), these quantities
satisfy \(0\leq G_X\leq G_{XZ}\leq G_M\). Thus, decision-time information
cannot create opportunity absent from the payoff-relevant state
\cite{blackwell1953}. Complete conditions and proofs are given in Supplementary
Proposition~\ref{prop:opportunity}. Probe value is the increase in optimal
conditional utility after observing \(Z_r\), less its resource and latency
costs. A probe is acquired only when a frozen action-change model and a
net-gain lower bound both pass (Supplementary
Proposition~\ref{prop:option}).

Within the candidate active branch, a cross-fitted value model estimates the
best action and its paired uncertainty relative to a fold-specific scoring
reference. A finite, outcome-blind threshold family controls how much evidence
is required to depart from that reference. The complete family, tie rules,
residual-expert architecture and cross-fitting procedure are specified in
Supplementary Section~\ref{sec:architecture}. Frontier analyses report every
fixed threshold rather than selecting a winner from observed outcomes. The
nearest-method comparison contract, including which quantitative comparisons
share the required data and action semantics, is in Supplementary
Table~\ref{tab:authorization_comparator_contract}.

Authorization is evaluated for the composite policy that applies the active
action when the frozen gate \(A(X)=1\) and otherwise executes \(K_0\). If
\(\Gamma(X)=U\{a_1(X)\}-U(K_0)\), its executed value relative to the fixed plan
is \(\mu=\mathbb E\{A(X)\Gamma(X)\}\). The basic contract requires an upper
bound on NULL false activation to be below its limit, a lower bound on
\(\mu\) to be positive and a prespecified lower bound on non-triviality to
exceed its target. Study-specific contracts may add active-set error,
coverage and missingness guards. The decision is an intersection--union test:
all components must pass, and an empty or failed contract returns \(K_0\).
This excludes always-fallback from satisfying the positive-value and
non-triviality claims while controlling false authorization of the conjunction
at the largest component level, conditional on development and calibration
(Supplementary Proposition~\ref{prop:joint_authorization}).

Learn-then-Test and conformal risk-control procedures can calibrate a frozen
family of candidate rules against a population-risk criterion
\cite{angelopoulos2025,angelopoulos2024crc}. \OPAL{} uses that logic where it
applies, but the authorization target is not a prediction set or one scalar
risk. It is a joint action-space contract that also requires positive executed
value relative to the actual fallback and a nontrivial active branch; the
finite-campaign analysis additionally concerns the specific unobserved
deployment complement. The intersection--union decision controls false
authorization of this conjunction at the component level. Simultaneous
interval coverage across a displayed candidate family is a separate
inferential requirement. In Cell Painting, the probability, gain, and
out-of-distribution thresholds were combined into one rule during development
and frozen before final evaluation, rather than selected as multiple final-set
hypotheses. A single scalarized Bayes objective would instead impose an
exchange rate among avoidable acquisition, coverage, and scientific value; we
treat these quantities as separate criteria because that exchange rate is a
decision-maker choice rather than an empirical identity.

\subsection*{Acquisition statistics and finite-campaign inference}\label{sec:m_boundaries}

Authorization depends on target outcomes in the region selected by the frozen
gate. Source outcomes and unlabeled target covariates do not identify the
relevant target risk and value under unrestricted conditional outcome shift.
Supplementary Proposition~\ref{prop:target_shift} states and proves this
identification boundary; the target-labeled recovery design is given in
Supplementary Section~\ref{sec:target_calibrated_recovery}.

For a campaign of \(N\) units, a simple random pilot \(P_q\) has size \(q\) and
the policy acts on \(D_q=\{1,\ldots,N\}\setminus P_q\). The costed campaign
contrast is

\begin{equation}\label{eq:delta_campaign}
\Delta_{\mathrm{camp}}(q,I_q)
=\frac{1}{N}\left[
\sum_{i\in D_q}\{U_i(\pi_{I_q})-U_i(K_0)\}
-C_{\mathrm{info,total}}(q)
\right].
\end{equation}

Pilot outcomes, the value of the specific deployment complement and a
superpopulation mean are distinct estimands. We invert the exact
multivariate-hypergeometric likelihood to obtain a simultaneous confidence set
for the finite composition and observation-channel nuisance, then switch only
when the least-favorable value of Eq.~\eqref{eq:delta_campaign} is positive
(Supplementary Section~\ref{sec:finite_algorithm}). The binary noiseless
calculation evaluates \(h(N,q,\rho)\) and \(q^\dagger(N,\rho,\alpha)\) as
defined in Supplementary Proposition~\ref{prop:impossibility}. Its proof,
finite-\(N\) qualification surface and large-\(N\) limit are given there;
prespecified deployment-reuse scenarios are listed in Supplementary
Table~\ref{tab:h_scenarios}.

\subsection*{Measured-data archives and statistical analysis}\label{sec:m_empirical}

\paragraph{Cell Painting target evaluation.}
The flagship study used public cpg0012 Cell Painting profiles
\cite{wawer2014profiling,bray2017cellpaintingdata}. A design-only reconstruction
identified compounds and planned wells before profile access, producing 25,013
eligible compound clusters. Each cluster supplied a deterministically assigned
initial replicate, two gate-contingent additional replicates and one held-out
verifier. Utility was half-scaled cosine similarity to the verifier, and net
gain subtracted a normalized two-well cost \(c=0.02\). NULL, AMBIGUOUS and
POSITIVE corresponded to net gain \(\leq0\), between 0 and
\(5\times10^{-3}\), and \(\geq5\times10^{-3}\), respectively. Decision-time
inputs were initial morphology, stateless hashed structure features and
pre-imaging plate metadata.

This utility quantifies technical replicate-profile agreement relevant to
estimating a compound's morphological profile. It does not measure hit
confirmation, mechanism of action or biological efficacy. The
\(5\times10^{-3}\) POSITIVE margin was an author-set operational threshold, not
a biologically calibrated effect size.

We used 13,748 labeled target compounds to develop the multi-head ExtraTrees
model and gate. Development evaluated 53,792 candidate threshold combinations.
None met all seven development criteria. A total of 12,673 met the six non-FDP criteria,
and the unique rule with the smallest FDP-bound excess among them was frozen
before final access (development FDP UCB, 35.1117\%). The frozen rule required predicted NULL probability
\(\leq0.395\), predicted POSITIVE probability \(\geq0.633333\), predicted gain
\(\geq0.053625\) and OOD percentile \(\leq1\). It was applied to 11,265 final
base profiles and activated 595 compounds. Assignments, model, configuration
and evaluator were fixed before the additional and verifier profiles were
opened. The full protocol, split history and failure completions are
given in Supplementary Section~\ref{sec:sr_cpg0012}.

The empirical OOD percentile was scaled to \([0,1]\). Its fixed ceiling of 1.0
passed all 11,265 final units and was therefore nonbinding; the other three
thresholds determined the 595 activations. It should not be interpreted as a
final-partition OOD restriction.

The seven-part contract required the one-sided 95\% false-activation and
false-discovery Clopper--Pearson upper bounds to be at most 7.5\% and 35\%, the
one-sided sensitivity lower bound to be at least 5\%, a positive one-sided 95\%
BCa value lower endpoint, at least 100 activations, at least 5\% coverage and
at most 5\% final-outcome missingness. Least-favorable completion was defined
separately for each endpoint. The observed counts were false activation
206/4,453 and sensitivity 384/6,534. The all-compound worst-case mean gain was
\(1.948\times10^{-3}\). Binomial bounds have their exact conditional
interpretation under independent common-probability Bernoulli units. The value
endpoint used 20,000 compound resamples and is nominal under an iid
canonical-compound model \cite{clopper1934,efron1987bca}.
Error levels were study-specific and fixed within each operational contract:
Cell Painting used one-sided 95\% endpoints, the finite-campaign study used its
declared Bonferroni allocation, and the simulator and CTRP used 99.5\% risk
bounds. These levels are not a shared cross-archive evidence scale.
The development-only assurance scenario, target-opening chain and exact
conditional NULL-risk boundary are documented in Supplementary
Tables~\ref{tab:target_power_plan},
\ref{tab:target_calibrated_stages} and~\ref{tab:target_cp_boundary}.

We tested layout dependence post hoc by holding assignments and endpoint
completions fixed and resampling the 80 non-overlapping plate-map blocks, each
of which contains the four physical plates spanned by a compound's roles.
Twenty thousand shared-index block resamples produced pointwise one-sided 95\%
percentile endpoints. Endpoint-specific paired resamples produced a secondary
BCa sensitivity. Development and final compounds shared all 80 blocks, so this
analysis assesses dependence within the observed layout and does not establish
transport to new blocks. Seeds, block definitions and complete results are in
Supplementary Section~\ref{sec:sr_cpg0012}.

All six comparator assignments were fixed before final outcomes. A descriptive
risk-budget sweep changed only the admissible false-activation upper-bound
budget. A separate fixed-assignment cost analysis held the \(c=0.02\) truth
strata, assignments and thresholds fixed, changed only the two-well charge and
recomputed the mean and 20,000-resample nominal BCa endpoint. These analyses do
not reselect a rule at a new risk budget or cost (Supplementary
Section~\ref{sec:cpg_cost_sensitivity_methods}). The 7.5\% ceiling was an
internal study-specific tolerance introduced during tuning and frozen before
final outcomes. It was not derived from biological efficacy, monetary value or
a facility-capacity model.

The CTRP v2 analysis used measured dose--response summaries
\cite{seashoreludlow2015}. A metadata-only design step selected four compounds
before response access and included cell-line families with complete outcomes,
giving 125 development, 244 calibration and 254 held-out families. Observable
inputs were restricted to pre-response cell-line and assay metadata. Actions
selected the top \(K\in\{1,2,4\}\) compounds under a development-trained ranker.
Value was the best measured normalized response in the selected bundle minus
its prespecified cost. Complete panel coverage made each action value a
measured lookup rather than an imputed counterfactual. Counts are exact for the
fixed archive. Clopper--Pearson bounds are secondary repeated-sampling summaries
under independent common-probability Bernoulli family indicators. The complete
design is reported in Supplementary Section~\ref{sec:sr_ctrp}; the post-lock
full-information comparator reanalysis is separated in Supplementary
Table~\ref{tab:ctrp_retro_comparators}.

The LINCS--LJP analysis linked public L1000 perturbation profiles to matched
longer-term growth-rate responses
\cite{subramanian2017l1000,niepel2017lincs}. Each complete episode contained
978 landmark-gene features at 3 h, the corresponding optional 24-h profile and
a bounded 72-h growth-rate endpoint. The 4,049 complete-case episodes belonged to
105 perturbation-identity components. We assigned whole components to LEARN,
CALIBRATION, CERTIFICATION and FINAL in a 42/21/21/21 split so that episodes
from the same component could not cross partitions. Within LEARN, nested
component-disjoint cross-fitting trained deterministic ridge models for the
3-h-only and 3-h-plus-24-h prediction paths. Episode value was the reduction in
absolute error for the bounded 72-h endpoint after subtracting the fixed
acquisition cost \(c=0.02\).

For reproducibility, after the independently executed evaluation was complete,
we reconstructed a publication episode index and component partition map from
the public source identifiers and frozen matching and partition rules. These
indexes contain source metadata and the fixed partition assignments, but no
growth-rate outcomes, model predictions or held-out statistics. Their
construction did not repeat model fitting, rule selection or evaluation.

Calibration evaluated 29 fixed score thresholds. For each threshold, metrics
were first calculated within component and then averaged equally across
components. The selected threshold \(-0.06\) acquired 706 calibration episodes
and passed the support, activation, action-distribution and positive-value
screens. The initial FDP-UCB rule (UCB no greater than 0.35) was not met in
CALIBRATION and was relaxed to a component-equal FDP point estimate no greater
than 0.45 before CERTIFICATION. CERTIFICATION met the revised FDP rule but not
the value-LCB-above-zero rule. The value criterion was then relaxed using the
certification aggregate to component-equal mean value above zero. Accordingly,
LEARN, CALIBRATION and CERTIFICATION were development data. The complete final
rule therefore required FDP point estimate no greater than 0.45, positive mean
value, component-equal activation fraction greater than
0.5395375, all accepted metrics to be finite, representation of at least 14
components, and action support across both batches, all six cell lines and at
least eight of the 12 observed cell-line--batch groups. The influence screen
also required positive value after removing each group and a maximum absolute
group-contribution share no greater than 0.50. The support constants were
author-set operational floors rather than biologically calibrated thresholds.
This complete rule was fixed before FINAL
outcomes were opened; accordingly, FINAL is the sole untouched evaluation of
the complete amended rule.
The component-equal 95\% FDP upper bound and value lower bound were reported as
uncertainty summaries and were not acceptance statistics. Leave-one-group-out
values and contribution shares assessed dominance among the 12 observed
technical groups, not transport beyond them. Supplementary
Section~\ref{sec:sr_lincs} gives the complete matching, modeling,
thresholding and aggregation procedure.

\subsection*{Mechanism, constructibility and operational analyses}\label{sec:m_mechanism}

Three author-generated simulators isolate information opportunity, finite
campaign evidence and held-out risk calibration. They draw continuous
mechanism axes and categorical capacity demand from declared stress envelopes,
not estimated facility distributions. Models and fixed comparators were fitted
out of fold, candidate policies shared common random numbers within campaigns,
and inference was aggregated to the declared campaign or fixture unit.
Finite-policy comparisons used a common resample plan and multiplier/max-
\(t\) critical value. Sequential finite-campaign looks used simultaneous
coverage \cite{chernozhukov2018,chernozhukov2013,howard2021}. Exact populations,
models, allocations and seeds are given in Supplementary Sections
\ref{sec:architecture}--\ref{sec:study3_design}. The continuous mechanism
axes and prespecified calibration design are tabulated in Supplementary
Tables~\ref{tab:axes} and~\ref{tab:study3_design}.

The remaining studies use facility documents and Causal Chambers trajectories
to test scheduling, reuse and post-selection constructibility mechanisms. They
do not supply instrument-executed interventions or an identified physical
measurement-rule value.
Supplementary Table~\ref{tab:data_sources} records each source, inferential
unit and evidentiary role. Operational value includes declared reservation,
setup, idle, cancellation, spot, delay and information charges. Lifecycle
analyses amortize shared training and characterization cost over a specified
number of deployments. Full ledgers and adoption boundaries are in
Supplementary Section~\ref{sec:cost}.

\section*{SUPPLEMENTARY MATERIALS}
\noindent Supplementary Text; Materials and Methods; Figs. S1 to S13; Tables S1
to S32; and publication-facing aggregate source tables for the external-data
figure.

\bibliographystyle{sciencemag}
\bibliography{references}

\section*{ACKNOWLEDGMENTS}

We are grateful to colleagues at Diamond Light Source for discussions that
helped frame the motivating operational problem: when a learned selection rule
should be allowed to alter a fixed measurement plan if instrument
configurations, staffing and capacity must be committed before all
decision-relevant evidence becomes available. We thank Yufei Wang (School of
Electronics and Information, Northwestern Polytechnical University, Xi'an
710072, China; \href{mailto:yufei_wang@mail.nwpu.edu.cn}{yufei\_wang@mail.nwpu.edu.cn})
for independently executing the LINCS--LJP analysis on a separately controlled
host and returning the aggregate outputs used here. The authors specified the
scientific question and analysis protocol, set each acceptance criterion before
the corresponding held-out partition was opened, and interpreted the results.

\noindent\textbf{AI-assisted technologies:} AI-assisted tools were used during
manuscript preparation for language editing, proofreading and quality-control
support. They did not generate or modify primary data, perform the reported
numerical analyses, select models or decision rules, or create images or
multimedia. The authors verified all AI-assisted output and take responsibility
for the submitted work.

\noindent\textbf{Funding:} This work was supported by the CNPC Innovation Fund
(grant no. 2024DQ02-0501), the Royal Society (grant no. IECNSFC233444),
Intelligent Manufacturing Longcheng Laboratory (grant no. CJ20254004) and the
Youth Science and Technology Talent Promotion Project of Jiangsu Province
(grant no. JSTJ-2025-137).

\noindent\textbf{Author contributions:} J.B., S.P. and C.Z. conceived the
study, specified the scientific questions and acceptance criteria, and
interpreted the results. J.B. developed the methods and software, performed
the Cell Painting, CTRP and simulation analyses, integrated the LINCS--LJP
aggregate outputs, and drafted the manuscript. C.Z. supervised the work. All
authors reviewed and revised the manuscript.

\noindent\textbf{Competing interests:} The authors declare no competing
interests.

\noindent\textbf{Data, code and materials availability:} The cpg0012 profiles and
raw microscopy are available from the Cell Painting Gallery collection
\texttt{cpg0012-wawer-bioactivecompoundprofiling}
~\cite{wawer2014profiling,bray2017cellpaintingdata}. LINCS L1000 profiles are
available from GEO accession
\href{https://www.ncbi.nlm.nih.gov/geo/query/acc.cgi?acc=GSE70138}{GSE70138},
and the matched growth-rate resource is described and distributed with Niepel
\emph{et al.}~\cite{niepel2017lincs}. CTRP v2 dose--response data are available
from the \href{https://portals.broadinstitute.org/ctrp.v2.1/}{Broad Institute
CTRP portal} and the \href{https://ocg.cancer.gov/programs/ctd2/data-portal}{NCI
CTD\({}^{2}\) Data Portal}~\cite{seashoreludlow2015}. The simulation studies use
author-generated synthetic campaign ledgers. Author-generated source tables,
aggregate evidence, figure-building materials and reproducibility records are
archived at
\href{https://doi.org/10.5281/zenodo.21859443}{doi: 10.5281/zenodo.21859443}.
The cited record contains all 29 LINCS--LJP calibration candidates,
certification and FINAL aggregate results, the 4,049-row episode index, the
105-component identity and partition map, source manifests, and the
deterministic frame-construction and analysis pipeline. The episode and
component indexes contain public source identifiers, metadata and frozen
partition assignments; they do not redistribute the underlying L1000
expression profiles, growth-response values, model predictions or held-out
statistics. The LINCS--LJP evaluation was executed independently on a separate
host, and the manuscript authors originally received aggregate outputs only.
After that evaluation was complete, the two publication indexes were
reconstructed from the cited public metadata and frozen matching and partition
rules solely to support reproducibility; model fitting, rule selection and
outcome aggregation were not repeated. The deposited aggregates reproduce
Fig.~\ref{fig:lincs_ljp} and its numerical summaries. Full episode-level
reanalysis requires reacquiring the cited third-party inputs and rerunning the
deposited pipeline. Code for
\OPAL{}, the simulation studies, CTRP and Cell Painting analyses is publicly
accessible for inspection at
\href{https://github.com/JIABI/mard_autonomous}{github.com/JIABI/mard\_autonomous};
the latest manuscript source, figure-building code and post-hoc score-sweep
tables accompany this submission. Author-generated data tables and
documentation in the Zenodo record are released under CC BY 4.0.
Author-generated software in the submission package is released under the MIT
License; the public repository is released under the same MIT License.
Third-party materials retain their original licenses. No physical materials
were generated in this work, and no patient or human-participant data were used.

\clearpage

\clearpage
\hypersetup{pageanchor=false}
\setcounter{secnumdepth}{3}
\setcounter{tocdepth}{2}
\renewcommand{\thesection}{S\arabic{section}}
\renewcommand{\thesubsection}{\thesection.\arabic{subsection}}
\renewcommand{\thesubsubsection}{\thesubsection.\arabic{subsubsection}}
\captionsetup{labelsep=period,hypcap=false}
\renewcommand{\topfraction}{0.95}
\renewcommand{\bottomfraction}{0.95}
\renewcommand{\textfraction}{0.06}
\renewcommand{\floatpagefraction}{0.80}
\setlength{\textfloatsep}{16pt plus 3pt minus 3pt}
\setlength{\floatsep}{14pt plus 3pt minus 3pt}

\begin{center}
{\LARGE\bfseries Supplementary Materials for\par}
\vspace{0.5em}
{\Large Held-out evidence resolves follow-up measurement decisions in biological screens\par}
\vspace{1.0em}
{\large Jia Bi$^{*}$, Samuel Pinilla, Chenyang Zhu$^{*}$\par}
\vspace{0.8em}
{\small $^{*}$Corresponding authors. Email: jia.bi@stfc.ac.uk; c.zhu@soton.ac.uk\par}
\end{center}

\vspace{1.5em}
\noindent\textbf{This PDF file includes:}
\begin{itemize}
  \item Supplementary Results
  \item Supplementary Methods
  \item Figs. S1 to S13
  \item Tables S1 to S32
  \item References cited in the manuscript reference list
\end{itemize}
\newpage

\begingroup
\small
\tableofcontents
\clearpage
\listoffigures
\listoftables
\endgroup

\renewcommand{\thefigure}{S\arabic{figure}}
\renewcommand{\thetable}{S\arabic{table}}
\renewcommand{\theequation}{S\arabic{equation}}
\setcounter{figure}{0}
\setcounter{table}{0}
\setcounter{equation}{0}

\section{SUPPLEMENTARY RESULTS}\label{sec:supp_results}

Three biological archives resolved three evidence states with different
experimental consequences. In Cell Painting, a lower-burden sparse candidate
had positive executed value under the compound-level analysis, but
false-discovery uncertainty kept the fixed plan in place. In LINCS--LJP, positive
mean point-estimate value accompanied acquisition of every final episode,
indicating broad 24-h profiling within the observed archive rather than
selective measurement saving. In CTRP, the frozen score did not rank measured
panel-relative opportunity, and every locked family fell back.
Together, these outcomes form a decision map rather than a performance
ranking.

A separate evidence layer evaluates the behavior and limits of \OPAL{} itself.
Formal results define information and finite-population boundaries; held-out
simulations show that risk control can preserve a nontrivial active branch
under constructed known-truth conditions; and retrospective selector and
Causal Chambers analyses diagnose executed-value alignment and
constructibility. Supplementary Table~\ref{tab:study_design} summarizes the
independent units and distinguishes the empirical archive states from these
methodological analyses.

Here \OPAL{} denotes the opportunity-aware protocol for authorizing learned
measurement rules. A policy is a prespecified mapping from information
available at assignment to one of the allowed experimental actions.

\begin{table}[!t]
\centering
\begin{threeparttable}
\caption{\textbf{Evidence layers and inferential roles.} Units are the independent units used for each scientific statement, not raw ledger rows.}
\label{tab:study_design}
\fontsize{7.2}{8.1}\selectfont
\renewcommand{\arraystretch}{1.03}
\rowcolors{2}{tablerow}{white}
\begin{tabularx}{\textwidth}{@{}L{2.65cm}L{2.55cm}L{2.45cm}L{2.45cm}Y@{}}
\specialrule{\heavyrulewidth}{0pt}{0pt}
\rowcolor{tablehead}
Study & Scientific question & Independent units & Evidence layer & Scope \\
\specialrule{\lightrulewidth}{0pt}{0pt}
Target-calibrated morphology & Can a development-selected gate activate nontrivially while controlling false activation and maintaining positive executed value? & 13,748 labeled development and 11,265 previously unopened final compound clusters & Empirical archive: sparse candidate & Lower-burden nontrivial candidate with positive executed value; binding FDP uncertainty kept the fixed plan in place \\
Temporal transcriptomics and growth & Does acquiring a 24-h L1000 profile add net value for predicting a 72-h growth-rate endpoint after a 3-h profile? & 105 perturbation-identity components and 4,049 complete-case episodes; component split 42/21/21/21 & Empirical archive: broad acquisition & Positive mean point estimate under broad acquisition; no selective measurement saving or new-batch transport \\
Pharmacogenomic exact outcomes & Can a frozen observable rule limit false activation and recover measured bundle value? & 125 development, 244 calibration and 254 locked cell-line families & Empirical archive: fallback & Frozen score did not rank measured opportunity; all locked families fell back to \(K_0\) \\
Executed-value mechanism & Does a learned selector exceed the correctly executed fixed rule after information cost? & 720 in-support units across 45 cells; 810 units for probe diagnostics & Retrospective mechanism diagnostic & Baseline, action and information cost evaluated on development data \\
Finite-campaign certification & Can a pilot certify positive value on the unobserved complement? & 7,442 campaigns at \(q\in\{2,4,\ldots,14\}\) & Exact methodological boundary & No superpopulation substitution for the finite complement \\
Held-out simulator risk & Does a frozen threshold control false activation under known simulator truth? & 140 NULL tuning, 500 NULL lock, 400 NULL validation and 500 positive fixtures & Held-out methodological validation & \(S^{\mathrm{eval}}\) uses potential outcomes and is not deployable \\
Physical-archive constructibility & Can a proxy score--action--fallback path be computed from measured trajectories? & 13 scalar archives and 196 sessions & Constructibility only & Author-defined proxies without a native causal action or risk-control endpoint \\
\specialrule{\heavyrulewidth}{0pt}{0pt}
\end{tabularx}
\begin{tablenotes}[flushleft]\footnotesize
\item Sensitivity in the held-out simulator study was an estimation target. CTRP assignments were sealed before locked-response access. Its fixed-archive counts are exact empirical quantities, whereas reported Clopper--Pearson bounds have a secondary repeated-sampling interpretation under a model in which the relevant family-level event indicators are independent Bernoulli trials with a common probability.
\item The cpg0012 rule and all 11,265 final gate assignments were sealed before final outcome access. Both labeled pre-final partitions are treated as development. Only the previously unopened final partition supports the single-held-out final evaluation.
\item In LINCS--LJP, LEARN, CALIBRATION and CERTIFICATION were used to establish the final operational rule and criteria and are therefore development evidence. The FINAL\_LOCK action map and criteria were fixed before its 24-h profiles or 72-h outcomes were opened.
\end{tablenotes}
\end{threeparttable}
\end{table}

\subsection{Executed-value diagnostics separate predicted opportunity from experimental value}\label{sec:sr_policy_diagnostics}

We separated the evaluated estimand, fixed comparator, selector family, probe
contribution and model architecture to determine whether the apparent value of
adaptation persisted after reconstructing the action that was actually
executed and charged.

\subsubsection{Selector-frontier reconciliation}\label{sec:sr_frontier}

We retrospectively reconstructed the executed comparison from the frozen
development ledgers. This descriptive analysis did not select a frontier
member. We distinguished three baselines. The fold-learned \(K\)-only baseline
was \(K=1\). The legacy best-fixed arm was \(K=2\) and included learned-model
cost. The factorial fixed-\(K=2\), no-probe policy was unlearned and carried no
model cost; it is therefore the executed comparator.

Table~\ref{tab:sr_gap_decomposition} gives the reconciliation shown in Supplementary Fig.~\ref{fig:ed_estimand}B at full precision. The floating-point residual confirms numerical closure of the four-step identity. The frozen ledger did not retain enough information to reconstruct the information-set, executor/exercise and action-cost subcomponents separately, so they remain bundled in the value-source term.

\begin{figure}[!htbp]
\centering
\suppfigure{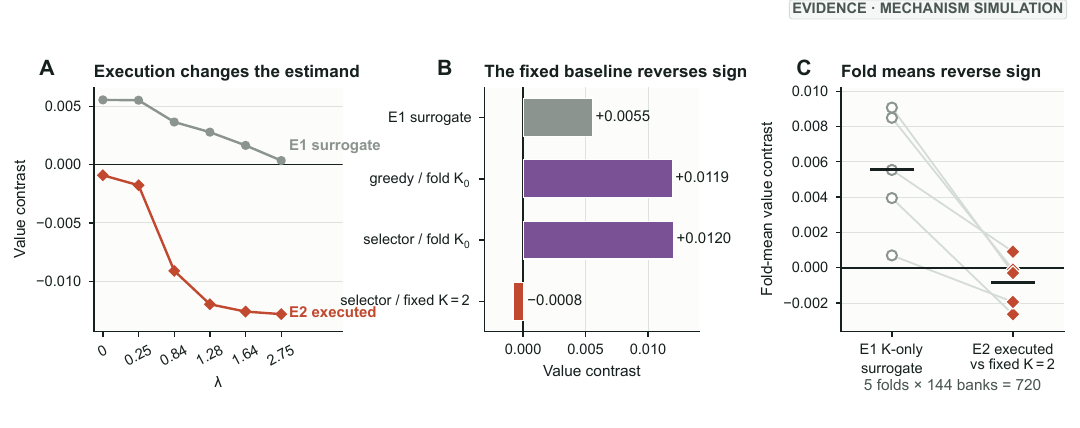}
\caption{\textbf{Estimand and comparator reconciliation in the retrospective
selector analysis.}
\textbf{A}, \(K\)-only surrogate and executed contrasts over the same frozen
\(\lambda\) family. The surrogate uses label-derived value against
\(K_{\mathrm{score},-f}=1\); the executed estimand uses operational utility
against \(K_0=2\). \textbf{B}, Four intermediate contrasts that isolate
value-source, selection-rule and comparator effects. \textbf{C}, Paired
surrogate and executed fold means for the five frozen outer folds (144 banks
per fold; 720 total); horizontal bars show the corresponding overall means.}
\label{fig:ed_estimand}
\end{figure}

We replayed all six \(\lambda\) members from cross-fitted per-bank, per-\(K\)
executed utilities and recovered both factorial endpoints exactly. The
\(K\)-only surrogate, executed contrasts and comparator reconciliation used
the same resamples and critical value, \(3.158018\). The six \(\lambda\) values
were fixed independently of these results. The bank-level score used
\(K_{\mathrm{score},-f}=1\) in every outer fold, whereas the operational
comparison and cohort-level fallback used \(K_0=2\). At \(\epsilon=0\), no
learned whole-policy member was eligible. Proposition~\ref{prop:safe}
therefore applies to the cohort wrapper, not to bank-level reversion from a
learned action to \(K_{\mathrm{score},-f}\).

Using the same contract, we reconstructed
\[
\widehat G^{\mathrm{aligned}}_M
=\frac1{720}\sum_i\{\max_{k\in\{0,1,2\}}U_i^{\mathrm{exec}}(k)
-U_i^{\mathrm{exec}}(K_0)\}
=2.689906\times10^{-3}
\]
from the same counterfactual action paths used by the executed frontier. Its
descriptive cell-stratified bootstrap interval was
\([2.253257,3.161963]\times10^{-3}\). Executor (MRAD), exercise rule (microgate),
commitment mode (fully reversible), \(\phi=1\), action-cost function,
zero-probe information cost and \(K_0=2\) were identical.
Every bank--policy gain was bounded by its bank-level maximum. The recovered
fraction was undefined because all attained frontier contrasts were
non-positive.

The oracle quantities \(G_X\) and \(G_{XZ}\) require conditional optimal values
under their respective information sets and cannot be identified from the
frozen ledger. A separate cross-fitted oracle evaluation stream yielded the
previously reported descriptive full-state routing quantity \(0.0198977\), with
interval \([0.0173493,0.0228224]\). Although this quantity is on a related
operational scale, it comes from a different evaluation-fold and seed stream
than the executed-estimand counterfactual path. We therefore do not use it as a
strict numerical ceiling for the executed frontier. Complete executed
intervals appear in Table~\ref{tab:sr_frontier_full}.

\begin{table}[ht]
\centering
\caption{\textbf{Exact selector-frontier gap decomposition.} Values
reconstruct the sign change from the \(K\)-only surrogate against
\(K_{\mathrm{score},-f}=1\) to executed operational value against the fixed
comparator \(K_0=2\).}
\label{tab:sr_gap_decomposition}
\small
\renewcommand{\arraystretch}{1.13}
\rowcolors{2}{tablerow}{white}
\begin{tabularx}{\textwidth}{@{}L{4.6cm}C{2.4cm}Y@{}}
\specialrule{\heavyrulewidth}{0pt}{0pt}
\rowcolor{tablehead}
Quantity & Value & Interpretation \\
\specialrule{\lightrulewidth}{0pt}{0pt}
Greedy K-only surrogate versus \(K_{\mathrm{score},-f}=1\) & \(5.5489699\times10^{-3}\) & Surrogate starting point \\
Executed greedy versus \(K_{\mathrm{score},-f}=1\) & \(11.9462080\times10^{-3}\) & Adds executed value source \\
Executed selector versus \(K_{\mathrm{score},-f}=1\) & \(12.0472049\times10^{-3}\) & Adds the actual selector rule \\
Executed selector versus \(K_0=2\) & \(-0.8143707\times10^{-3}\) & Correct operational contrast \\
\(\Delta_{\mathrm{value\ source}}\) & \(6.3972381\times10^{-3}\) & Label-to-execution bundle \\
\(\Delta_{\mathrm{policy\ rule}}\) & \(0.1009968\times10^{-3}\) & Greedy-to-actual selector \\
\(\Delta_{\mathrm{baseline}}\) & \(-12.8615755\times10^{-3}\) & \(K_{\mathrm{score},-f}=1\) to \(K_0=2\) \\
Residual & \(6.51\times10^{-19}\) & Numerical closure \\
\specialrule{\heavyrulewidth}{0pt}{0pt}
\end{tabularx}
\end{table}

\begin{table}[ht]
\centering
\caption{\textbf{Full executed selector frontier against fixed \(K=2\).} All six points use one score estimator and one 2,000-replicate simultaneous max-\(t\) family.}
\label{tab:sr_frontier_full}
\small
\renewcommand{\arraystretch}{1.10}
\rowcolors{2}{tablerow}{white}
\begin{tabularx}{\textwidth}{@{}C{1.4cm}Y C{2.1cm}C{2.1cm}C{2.1cm}@{}}
\specialrule{\heavyrulewidth}{0pt}{0pt}
\rowcolor{tablehead}
\(\lambda\) & Policy label & \shortstack{Mean\\(\(\times10^{-3}\))} & \shortstack{Lower bound\\(\(\times10^{-3}\))} & \shortstack{Upper bound\\(\(\times10^{-3}\))} \\
\specialrule{\lightrulewidth}{0pt}{0pt}
0 & Greedy & \(-0.915\) & \(-2.184\) & \(0.353\) \\
0.25 & Legacy threshold & \(-1.773\) & \(-2.939\) & \(-0.607\) \\
0.8416 & \(\alpha=0.20\) critical value & \(-9.112\) & \(-10.223\) & \(-8.000\) \\
1.2816 & \(\alpha=0.10\) critical value & \(-11.979\) & \(-13.069\) & \(-10.888\) \\
1.6449 & \(\alpha=0.05\) critical value & \(-12.609\) & \(-13.625\) & \(-11.592\) \\
2.7478 & \(\alpha=3\times10^{-3}\) critical value & \(-12.827\) & \(-13.800\) & \(-11.853\) \\
\specialrule{\heavyrulewidth}{0pt}{0pt}
\end{tabularx}
\end{table}

\subsubsection{Probe and expert diagnostics}\label{sec:sr_option}

Table~\ref{tab:sr_factorial} reports the executed factorial values. The no-probe fixed policy performed best among the nine cells. Full-panel policies were strongly negative because they purchased every probe coordinate at every bank, making this primarily an information-cost result rather than a failure of the scientific-success simulator. The active, decision-focused probe cost substantially less, but neither its fixed nor learned policy exceeded fixed \(K\) without a probe.

\begin{table}[!t]
\centering
\caption{\textbf{Executed selector-frontier factorial ablation.} Contrasts use fixed \(K\), no probe as the common comparator.}
\label{tab:sr_factorial}
\small
\renewcommand{\arraystretch}{1.10}
\rowcolors{2}{tablerow}{white}
\begin{tabularx}{\textwidth}{@{}L{2.3cm}L{2.1cm}C{1.8cm}C{2.0cm}Y@{}}
\specialrule{\heavyrulewidth}{0pt}{0pt}
\rowcolor{tablehead}
Value component & Information component & Operational value & Contrast & \shortstack[l]{Model-\\dependent} \\
\specialrule{\lightrulewidth}{0pt}{0pt}
Fixed \(K\) & No probe & 0.705246 & 0 & no \\
Fixed \(K\) & Decision probe & 0.704732 & \(-0.515\times10^{-3}\) & no \\
Fixed \(K\) & Full panel & 0.362332 & \(-0.342915\) & no \\
Global learned \(K\) & No probe & 0.704432 & \(-0.814\times10^{-3}\) & yes \\
Global learned \(K\) & Decision probe & 0.703846 & \(-1.401\times10^{-3}\) & yes \\
Global learned \(K\) & Full panel & 0.343310 & \(-0.361937\) & yes \\
Soft residual experts & No probe & 0.704432 & \(-0.814\times10^{-3}\) & yes; \(E=0\) \\
Soft residual experts & Decision probe & 0.703846 & \(-1.401\times10^{-3}\) & yes; \(E=0\) \\
Soft residual experts & Full panel & 0.343310 & \(-0.361937\) & yes; \(E=0\) \\
\specialrule{\heavyrulewidth}{0pt}{0pt}
\end{tabularx}
\end{table}

In the probe-option ledger, 10 of 810 banks purchased a probe and two changed capacity. Mean gross option value was \(0.0228668\), whereas mean net option value after the prespecified 0.12 resource charge was \(-0.0971332\). The common-random-number realized contribution across 720 in-support banks was \(-0.5862\times10^{-3}\), with standard error \(0.3336\times10^{-3}\) and simultaneous interval \([-1.1724,0]\times10^{-3}\). The two action-changing banks had mean realized contribution \(-0.14798\), but this sample is too small for population inference. Candidate identities, actions before and after probing, and recomputed resource charges all agreed with the source ledger.

The K selector produced distribution \(K=0:92\), \(K=1:139\), \(K=2:489\), compared with oracle distribution \(K=0:112\), \(K=1:52\), \(K=2:556\). There were 222 mismatches among 720 banks. Nested cross-validation selected \(E=0\) in all five folds; consequently, the soft-expert and global learned-\(K\) factorial rows are identical by construction rather than by post hoc deletion.

Supplementary Fig.~\ref{fig:ed_architecture} displays the foldwise expert
selection, selector--oracle action matrix and probe outcomes underlying these
diagnostics.

\begin{figure}[!htbp]
\centering
\suppfigure{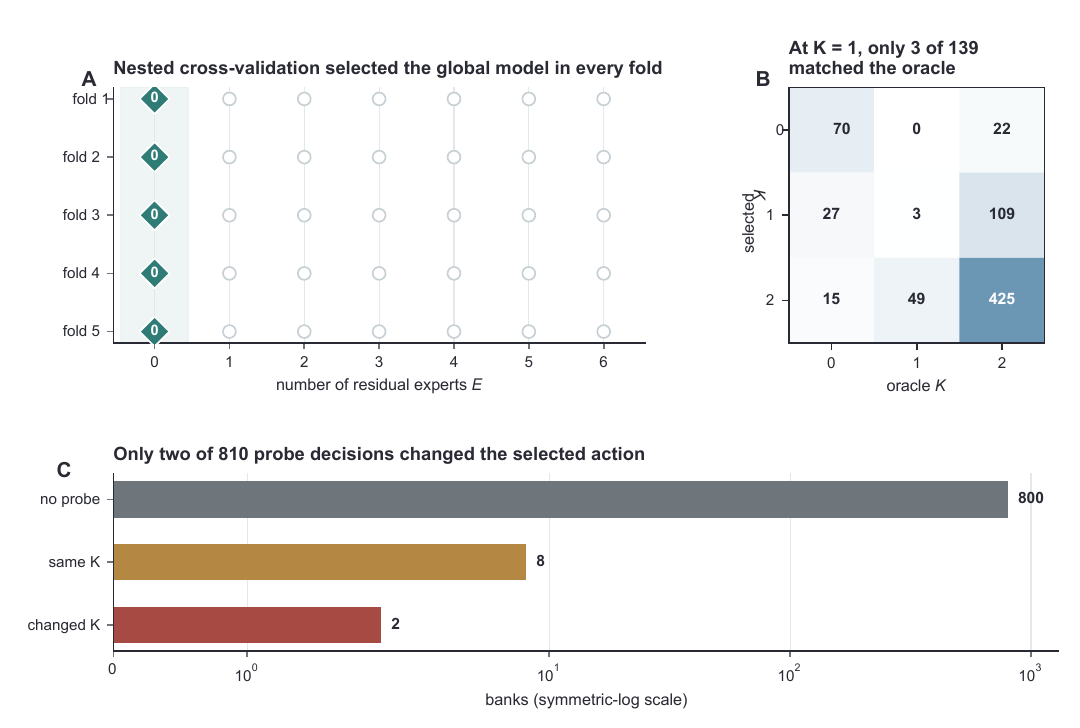}
\caption{\textbf{Model-complexity and action-change diagnostics.}
\textbf{A}, Candidate residual-expert counts \(E=0,\ldots,6\) in each outer
fold. Filled diamonds mark the nested-cross-validation choice under the
one-standard-error rule; all five folds selected the global model \(E=0\).
\textbf{B}, Confusion matrix for the learned capacity and information oracle
over 720 in-support banks. The selector chose \(K=1\) 139 times, of which three
matched the oracle. \textbf{C}, Probe outcomes over 810 banks. Counts are
displayed on a symmetric-log scale: 800 no-purchase decisions, eight purchases
without an action change and two action-changing purchases.}
\label{fig:ed_architecture}
\end{figure}

\FloatBarrier

\subsection{Finite-population evidence boundaries limit certifiability}\label{sec:sr_finite_group}

The next analysis treats the deployment complement, rather than a
superpopulation mean, as the target. We report the exhaustive certification
calculation first and then retain the same campaign denominator when
information cost and reuse are evaluated.

\subsubsection{Finite-campaign certification}\label{sec:sr_finite_exact}

The finite population contained \(N=16\) banks. Our primary analysis used the
three-class observation-and-decision model. We enumerated
all 1,113 possible signal compositions over the 21 prespecified \((q,p)\) cells,
reconstructed every exact confidence-pair set and least-favorable
composition, and found no actionable or active observation
(Proposition~\ref{prop:full_likelihood_impossibility} and
Table~\ref{tab:full_likelihood_certificate}). This certificate was computed
from the prespecified likelihood and utility contract without regenerating the
22,861,824-row simulator ledger.

\begin{table}[!htbp]
\centering
\caption{\textbf{Full three-class likelihood certification envelope.} The number of signal vectors is per channel parameter; all three prespecified \(p\) values give the same envelopes and state outcome.}
\label{tab:full_likelihood_certificate}
\small
\renewcommand{\arraystretch}{1.10}
\rowcolors{2}{tablerow}{white}
\begin{tabularx}{\textwidth}{@{}C{0.9cm}C{1.8cm}C{1.8cm}C{2.0cm}C{2.1cm}Y@{}}
\specialrule{\heavyrulewidth}{0pt}{0pt}
\rowcolor{tablehead}
\(q\) & Signals per \(p\) & \(B_{q,p}\) & \(B^\Delta_{q,p}\) & \shortstack{Actionable/\\active} & State set \\
\specialrule{\lightrulewidth}{0pt}{0pt}
2 & 6 & 0 & \(-0.06875\) & 0/0 & Indeterminate \\
4 & 15 & 0 & \(-0.13125\) & 0/0 & Indeterminate \\
6 & 28 & 0 & \(-0.19375\) & 0/0 & Indeterminate \\
8 & 45 & 0 & \(-0.25625\) & 0/0 & Indeterminate \\
10 & 66 & 0 & \(-0.31875\) & 0/0 & Indeterminate \\
12 & 91 & 0 & \(-0.38125\) & 0/0 & Indeterminate \\
14 & 120 & 0 & \(-0.44375\) & 0/0 & Indeterminate \\
\specialrule{\heavyrulewidth}{0pt}{0pt}
\end{tabularx}
\end{table}

To interpret this result and guide campaign-size design, we separately evaluated the
binary all-actionable event \(X=q\) under the largest total actionable count
\(M_0=q+\lfloor(16-q)/2\rfloor\) whose complement still lacked a strict
actionable majority. Table~\ref{tab:sr_n16_enumeration} gives the integer
enumeration underlying Proposition~\ref{prop:impossibility}. This
binary corollary complements rather than replaces the three-class algorithm.
Neither calculation uses a normal or independent-bank approximation.

\begin{table}[!t]
\centering
\caption{\textbf{Exact all-success enumeration for the prespecified 16-bank
campaign.} For each pilot size \(q\), \(M_0\) is the largest total actionable
count whose unobserved complement still lacks a strict actionable majority.}
\label{tab:sr_n16_enumeration}
\small
\renewcommand{\arraystretch}{1.10}
\rowcolors{2}{tablerow}{white}
\begin{tabularx}{\textwidth}{@{}C{1.0cm}C{1.6cm}C{1.6cm}C{2.5cm}C{2.5cm}Y@{}}
\specialrule{\heavyrulewidth}{0pt}{0pt}
\rowcolor{tablehead}
\(q\) & \(16-q\) & \(M_0\) & \(\binom{M_0}{q}\) & \(\binom{16}{q}\) & Exact probability \\
\specialrule{\lightrulewidth}{0pt}{0pt}
2 & 14 & 9 & 36 & 120 & 0.300000 \\
4 & 12 & 10 & 210 & 1,820 & 0.115385 \\
6 & 10 & 11 & 462 & 8,008 & 0.057692 \\
8 & 8 & 12 & 495 & 12,870 & 0.038462 \\
10 & 6 & 13 & 286 & 8,008 & 0.035714 \\
12 & 4 & 14 & 91 & 1,820 & 0.050000 \\
14 & 2 & 15 & 15 & 120 & 0.125000 \\
\specialrule{\heavyrulewidth}{0pt}{0pt}
\end{tabularx}
\end{table}

Exact enumeration covered all 21 \(q\times p\) combinations without omission
or imputation. Pair coverage ranged from 0.998212 to 1.000000, above the
prespecified per-look minimum of 0.997143. The campaign registry contained 7,442
campaigns, and the potential-outcome ledger contained 22,861,824
common-random-number rows. An independent implementation reproduced the pilot
selection, observation channel, decisions, campaign values and potential
outcomes.

Supplementary Fig.~\ref{fig:ed_coverage} summarizes the exact pair-coverage
shortfall and the denominator available to each operating-characteristic
endpoint.

\begin{figure}[!htbp]
\centering
\suppfigure{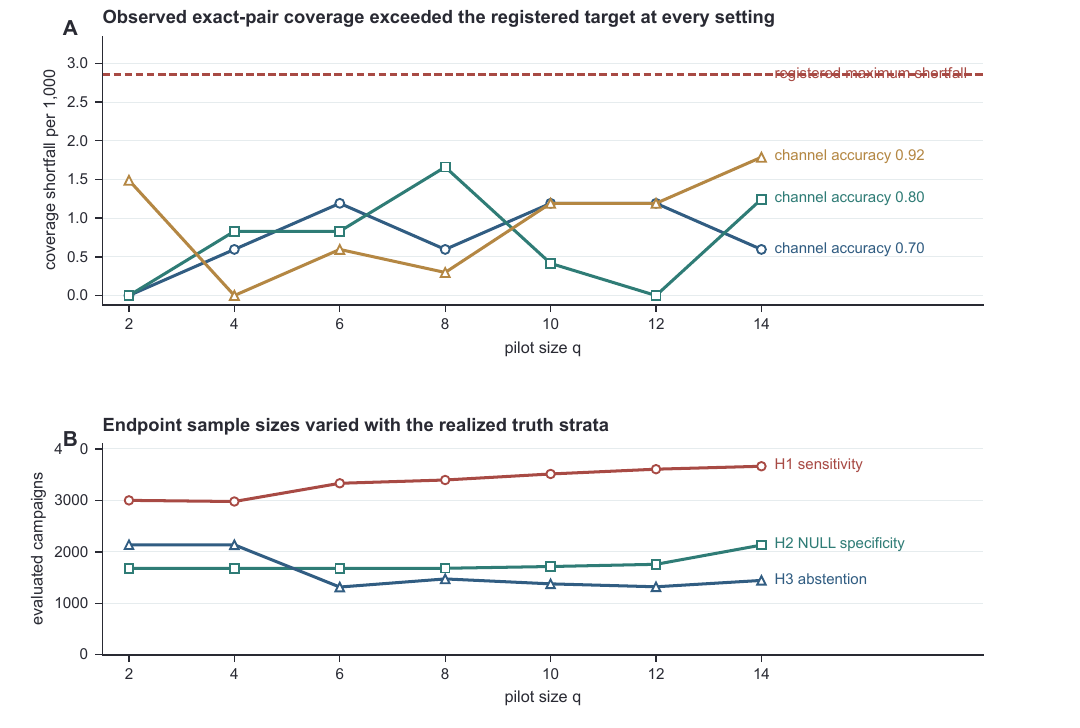}
\caption{\textbf{Exact finite-campaign coverage analysis.}
\textbf{A}, Empirical joint-coverage shortfall per 1,000 for the exact
composition--channel confidence set over 21 combinations of pilot size and
channel accuracy. The dashed line is the prespecified maximum shortfall
corresponding to coverage 0.99714; observed coverage was
0.9982--1.0000. \textbf{B}, Numbers of campaigns contributing to each
operating-characteristic endpoint as a function of \(q\). Counts vary because
truth strata are assigned from the finite realized campaign and deployable
opportunity.}
\label{fig:ed_coverage}
\end{figure}

Actionable sensitivity was zero at every \(q\). Table~\ref{tab:sr_finite_oc} places this result in context by reporting each denominator and its exact simultaneous one-sided upper limit. NULL specificity and nonactionable abstention use the corresponding lower limits. At every look, actionable upper limits were below \(2.5\times10^{-3}\), whereas both abstention lower limits exceeded 0.9945.

\begin{table}[!t]
\centering
\caption{\textbf{Finite-campaign operating-characteristic denominators and simultaneous exact limits.}}
\label{tab:sr_finite_oc}
\fontsize{6.3}{7.2}\selectfont
\renewcommand{\arraystretch}{1.08}
\setlength{\tabcolsep}{1.5pt}
\rowcolors{2}{tablerow}{white}
\begin{tabularx}{\textwidth}{@{}C{0.55cm}C{1.35cm}C{1.5cm}C{1.1cm}C{1.5cm}C{1.35cm}C{1.5cm}Y@{}}
\specialrule{\heavyrulewidth}{0pt}{0pt}
\rowcolor{tablehead}
\(q\) & \shortstack{Actionable\\\(n\)} & \shortstack{Active UCB\\(\(\times10^{-3}\))} & \shortstack{NULL\\\(n\)} & \shortstack{NULL spec.\\LCB} & \shortstack{Nonact.\\\(n\)} & \shortstack{Nonact.\\abstain LCB} & Result \\
\specialrule{\lightrulewidth}{0pt}{0pt}
2 & 3,001 & 2.411 & 1,678 & 0.995692 & 2,138 & 0.996617 & \shortstack[l]{No active\\qualification} \\
4 & 2,979 & 2.429 & 1,678 & 0.995692 & 2,137 & 0.996616 & \shortstack[l]{No active\\qualification} \\
6 & 3,333 & 2.171 & 1,680 & 0.995697 & 1,316 & 0.994510 & \shortstack[l]{No active\\qualification} \\
8 & 3,397 & 2.130 & 1,681 & 0.995700 & 1,474 & 0.995097 & \shortstack[l]{No active\\qualification} \\
10 & 3,514 & 2.059 & 1,715 & 0.995785 & 1,378 & 0.994757 & \shortstack[l]{No active\\qualification} \\
12 & 3,607 & 2.006 & 1,758 & 0.995888 & 1,321 & 0.994531 & \shortstack[l]{No active\\qualification} \\
14 & 3,664 & 1.975 & 2,132 & 0.996608 & 1,445 & 0.994999 & \shortstack[l]{No active\\qualification} \\
\specialrule{\heavyrulewidth}{0pt}{0pt}
\end{tabularx}
\end{table}

\subsubsection{Fallback value and finite-campaign qualification}\label{sec:sr_finite_value}

At every fixed \(q\), the policy retained \(K_0\) for every campaign. The
operational contrast therefore had zero sampling standard error and equalled
the negative information bill. The prespecified feasible sets
\(Q_{\mathrm{OC}}\), \(Q_{\mathrm{net}}\) and \(Q_{\mathrm{adopt}}\) were
empty. The sequential policy carried all 7,442 campaigns to the last
prespecified look, with zero activation, mean stopping point \(q=14\), mean
contrast \(-0.44375\) and no active harm events. Because no switch was
exercised, this result does not measure the performance of active switching.

We evaluated lifecycle value at
\(H\in\{250,1000,3000,10000,100000\}\). Because
\(\Delta_{\mathrm{camp}}\leq0\) before shared cost, every lifecycle value was
negative and no finite \(H^\star\) existed. Information-recovery ratios were
undefined because their opportunity denominator was negligible or non-positive.

The complete evidence trajectories, value summaries and cost--reuse dependence
are shown together in Supplementary Figs.~\ref{fig:finite_campaign_diagnostics},
\ref{fig:finite_campaign_values} and~\ref{fig:ed_finite_economics}.

\begin{figure}[!htbp]
\centering
\suppfigure{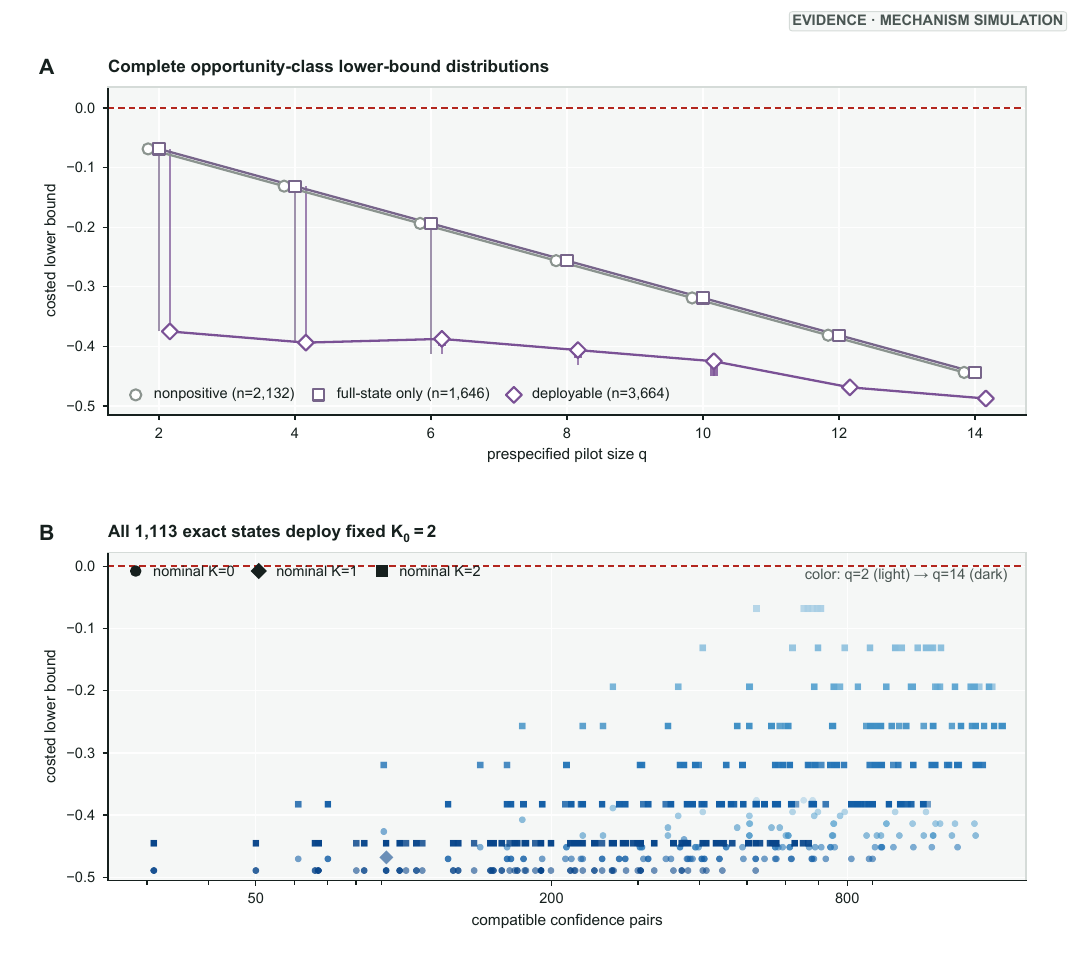}
\caption{\textbf{Complete finite-campaign evidence trajectories and exact states.}
\textbf{A}, Pilot-size-dependent 5th--95th percentile ranges, interquartile
ranges and medians for nonpositive, positive full-state-only and deployable
campaigns. This complete three-class view supports the simplified main-text
display without hiding the overlapping central trajectories. \textbf{B}, All
1,113 exact observation states at \(N=16\), colored by prespecified pilot size;
every state selected the fallback \(K_0=2\) and every costed lower bound was
negative.}
\label{fig:finite_campaign_diagnostics}
\end{figure}

\begin{figure}[!htbp]
\centering
\suppfigure{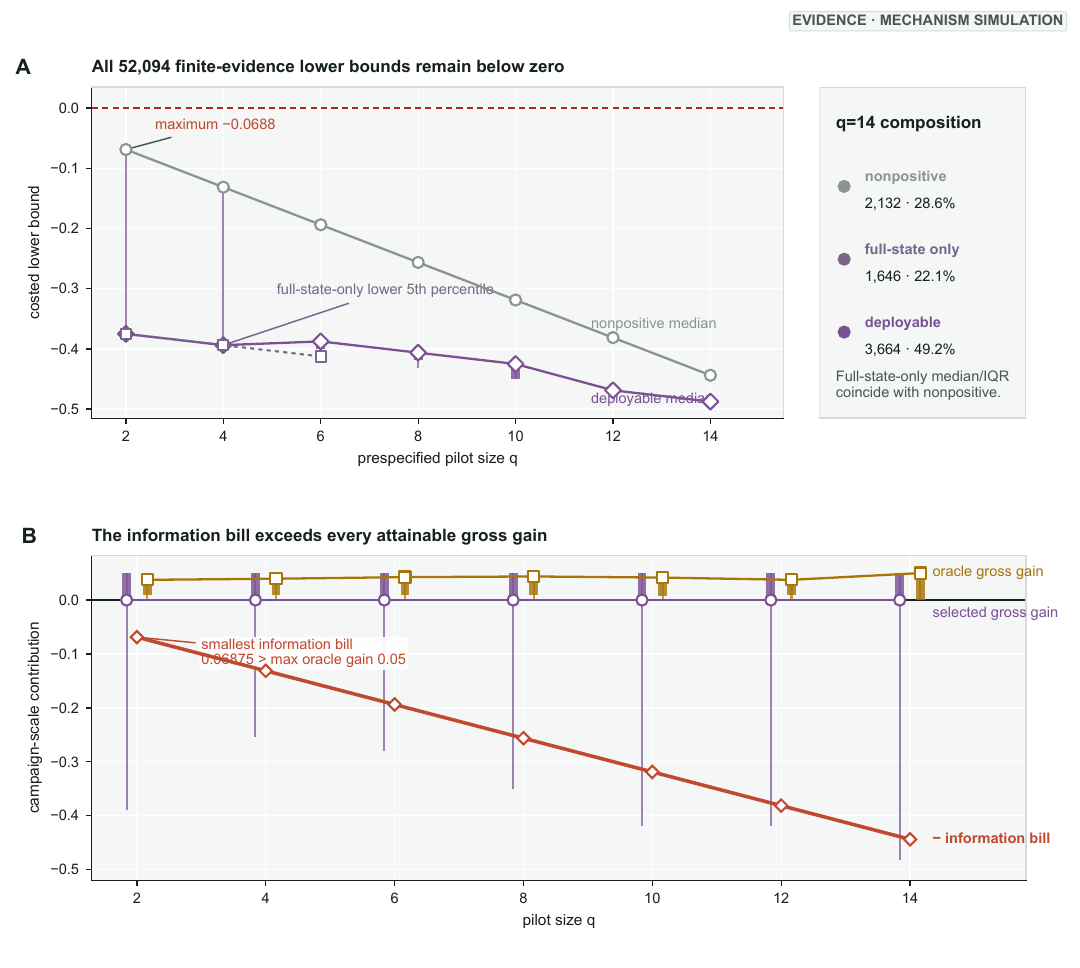}
\caption{\textbf{Finite-evidence lower bounds and information cost across all campaign looks.}
\textbf{A}, Pilot-size-dependent 5th--95th percentile ranges, interquartile
ranges and medians for all 52,094 campaign--look lower bounds. Every bound
remained below zero. The inset gives the exact final-look composition of 7,442
campaigns: 3,664 deployable-positive, 1,646 positive full-state-only and 2,132
nonpositive campaigns. \textbf{B}, Selected-action and oracle gross-gain
distributions together with the campaign-normalized information bill across
104,188 campaign--look gain values. The smallest information bill, 0.06875,
exceeded the maximum oracle gross gain, 0.05, and increased with pilot size.}
\label{fig:finite_campaign_values}
\end{figure}

\begin{figure}[!htbp]
\centering
\suppfigure{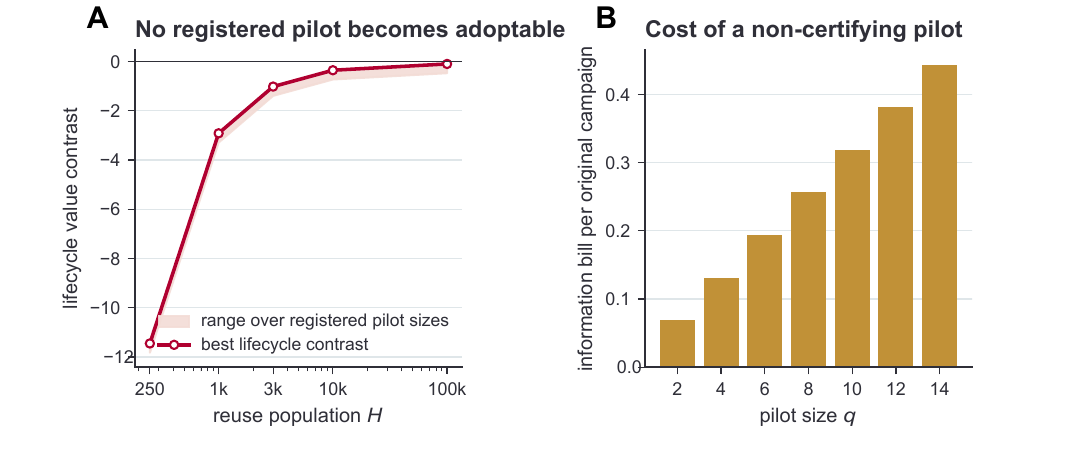}
\caption{\textbf{Finite-campaign information cost and lifecycle scenarios.}
\textbf{A}, Best lifecycle contrast and its range over all prespecified pilot
sizes as reuse population \(H\) increases. The entire range remains below zero
because operational contrast is already non-positive before shared cost is
added. \textbf{B}, Information bill per original 16-bank campaign. No finite
positive \(H^\star\) exists at any displayed \(q\). The \(H=100{,}000\)
column is a unit-relaxed stress scenario, not a campaign-level adoption
estimate.}
\label{fig:ed_finite_economics}
\end{figure}

\FloatBarrier

\subsection{Held-out simulation shows that risk control need not force universal fallback}\label{sec:sr_risk_group}

The prespecified simulator study comprised NULL-only threshold selection, an
independent risk lock, held-out validation and a retrospective comparator
battery. The binding false-activation claim applies only to the originally
locked rule. Sensitivity is an estimation target, and the value decomposition
and comparator analyses are descriptive.

\subsubsection{Threshold calibration and risk lock}\label{sec:sr_d6_calibration}

The held-out continuous-population study used four mutually disjoint
populations. The NULL tuning population contained 140 fixtures and generated
the order-statistic candidate class. It selected
\(\tau=0.0180806098\) without access to positive or lock observations. The
independent risk-lock population contained 500 fixtures and no activations,
giving a one-sided 99.5\% Clopper--Pearson upper bound of
\(0.0105406882\). This result fixed the original threshold; the prespecified
second-threshold option was not used.

Tuning scores ranged from \(-0.123126\) to \(0.018081\), with mean
\(-0.042451\); independent-lock scores ranged from \(-0.161873\) to
\(7.961\times10^{-3}\), with mean \(-0.044585\). The separation between their
maxima is descriptive. Release of the frozen threshold to validation rested
on the exact binomial upper bound.

\subsubsection{Held-out risk validation}\label{sec:sr_d6_validation}

The final run contained 400 NULL and 500 positive fixtures, with one 200-bank
campaign per fixture. The tuning, lock and validation partitions had no
fixture, campaign, bank, parameter or seed overlap. Supplementary
Fig.~\ref{fig:risk_validation}, A to C, shows the complete score distributions,
score--opportunity relation and primary risk result; Table~\ref{tab:sr_study3_results}
gives the exact numerical summaries.

\begin{table}[!t]
\centering
\caption{\textbf{Held-out risk calibration and simulator validation results.}}
\label{tab:sr_study3_results}
\small
\renewcommand{\arraystretch}{1.13}
\rowcolors{2}{tablerow}{white}
\begin{tabularx}{\textwidth}{@{}L{2.4cm}C{1.2cm}C{1.8cm}L{3.0cm}Y@{}}
\specialrule{\heavyrulewidth}{0pt}{0pt}
\rowcolor{tablehead}
Partition & \(n\) & Activations & Exact result & Evidentiary role \\
\specialrule{\lightrulewidth}{0pt}{0pt}
NULL tuning & 140 & 0 above selected maximum & \(\tau=0.0180806098\) & Threshold selection \\
NULL risk lock & 500 & 0 & 99.5\% UCB 0.0105407 & Independent threshold lock \\
NULL validation & 400 & 0 & 99.5\% UCB 0.0131585 & Binding claim supported \\
Positive validation & 500 & 342 & 0.684; 95\% CI [0.641249, 0.724559] & Estimation only \\
\specialrule{\heavyrulewidth}{0pt}{0pt}
\end{tabularx}
\end{table}

NULL validation scores ranged from \(-0.172296\) to
\(3.260\times10^{-3}\), with mean \(-0.046505\). Positive scores ranged from
\(-0.075173\) to \(0.165675\), with mean \(0.035189\). The threshold therefore
exceeded every NULL score but not every positive score. Direct reconstruction
confirmed the reported units, score and raw-to-summary calculations. The
potential-outcome score remains an evaluation statistic rather than a
deployable one.

\medskip\noindent\textbf{Retrospective value decomposition.}
After the primary risk result was known, we used the sealed
positive-validation potential outcomes for a descriptive value decomposition
(Table~\ref{tab:sr_d6_value}). The frozen policy applies the learned action when
\(S^{\mathrm{eval}}>\tau\) and otherwise returns \(K_0\). Because the
activated-fixture mean conditions on this selection, we report it alongside the
unconditional gated mean, executed fallback mean and force-activate comparator.
The force-activate arm applies the same learned bank-level action to all 500
positive fixtures, separating the gate from the quality of the fixtures that
it selected. Because \(S^{\mathrm{eval}}\) is also the fixture-level paired
potential-outcome contrast used in this decomposition, the analysis is
diagnostic rather than an independent confirmation of policy value. At the
fixture level, gated minus force-activate value was
\(8.6026\times10^{-5}\), with a percentile bootstrap 95\% interval of
\([-6.7823\times10^{-4},\,8.7498\times10^{-4}]\). The interval spans zero. The
gate removed negative forced contrasts in 13.6\% of fixtures but also declined
small positive contrasts, yielding no detectable mean-value advantage
(Supplementary Fig.~\ref{fig:risk_validation}, D).

\begin{table}[!t]
\centering
\caption{\textbf{Evaluation-only value decomposition on positive simulator fixtures.} Intervals are percentile intervals from 10,000 fixture bootstrap resamples with a fixed analysis seed. The analysis was specified after the primary risk result and is descriptive.}
\label{tab:sr_d6_value}
\fontsize{7.4}{8.2}\selectfont
\renewcommand{\arraystretch}{1.05}
\setlength{\tabcolsep}{2pt}
\rowcolors{2}{tablerow}{white}
\begin{tabularx}{\textwidth}{@{}L{2.05cm}C{0.65cm}C{1.15cm}C{2.05cm}C{1.05cm}Y@{}}
\specialrule{\heavyrulewidth}{0pt}{0pt}
\rowcolor{tablehead}
Policy & \(n\) & Mean & Bootstrap 95\% interval & Harm fraction & Interpretation \\
\specialrule{\lightrulewidth}{0pt}{0pt}
Gated policy, all positive & 500 & 0.035275 & [0.032382, 0.038196] & 0 & Learned action if active; otherwise \(K_0\) \\
Activated fixtures & 342 & 0.051572 & [0.048692, 0.054587] & 0 & Conditional on \(S^{\mathrm{eval}}>\tau\) \\
Fallback fixtures & 158 & 0 & [0, 0] & 0 & Executed \(K_0\) contrast \\
Force activate, all positive & 500 & 0.035189 & [0.032181, 0.038262] & 0.136 & Learned action without the campaign gate \\
Gated minus force activate & 500 & \(0.086\times10^{-3}\) & \([-0.678,\,0.875]\)\newline\(\times10^{-3}\) & 0.180 & Paired difference; no distinguishable mean advantage \\
\specialrule{\heavyrulewidth}{0pt}{0pt}
\end{tabularx}
\end{table}

\begin{figure}[!htbp]
\centering
\suppfigure{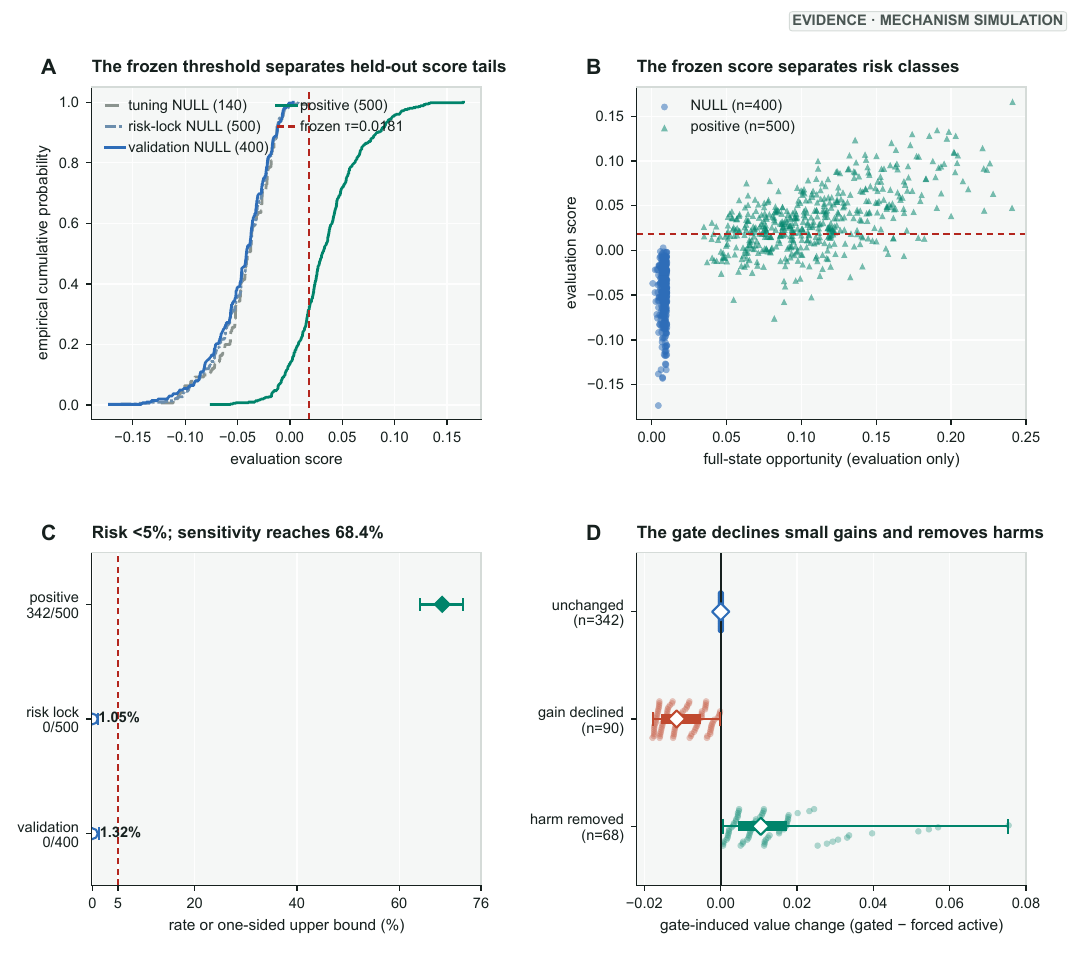}
\caption{\textbf{Held-out simulation shows that risk control need not force
universal fallback.} \textbf{A}, Score distributions in disjoint NULL tuning,
NULL risk-lock, NULL validation and positive-validation sets, with the frozen
threshold \(\tau=0.0180806\). \textbf{B}, Held-out score--opportunity
relationship for all 900 validation fixtures. \textbf{C}, Exact risk bounds
and sensitivity: 0/400 NULL activations (one-sided 99.5\% UCB, 1.316\%) and
342/500 positive activations (68.4\%; 95\% confidence interval,
64.1--72.5\%). \textbf{D}, Fixture-level gated-minus-forced-active value for
all 500 positive fixtures: 342 active fixtures were unchanged, 68 harmful
forced-active values were removed and 90 small positive effects were declined.
This mechanism simulation is methodological evidence and not a measured
biological archive.}
\label{fig:risk_validation}
\end{figure}

The tuning-to-lock transition and positive-population heterogeneity are shown
together in Supplementary Fig.~\ref{fig:ed_d6_calibration_heterogeneity}.

\begin{figure}[!htbp]
\centering
\suppfigurefit{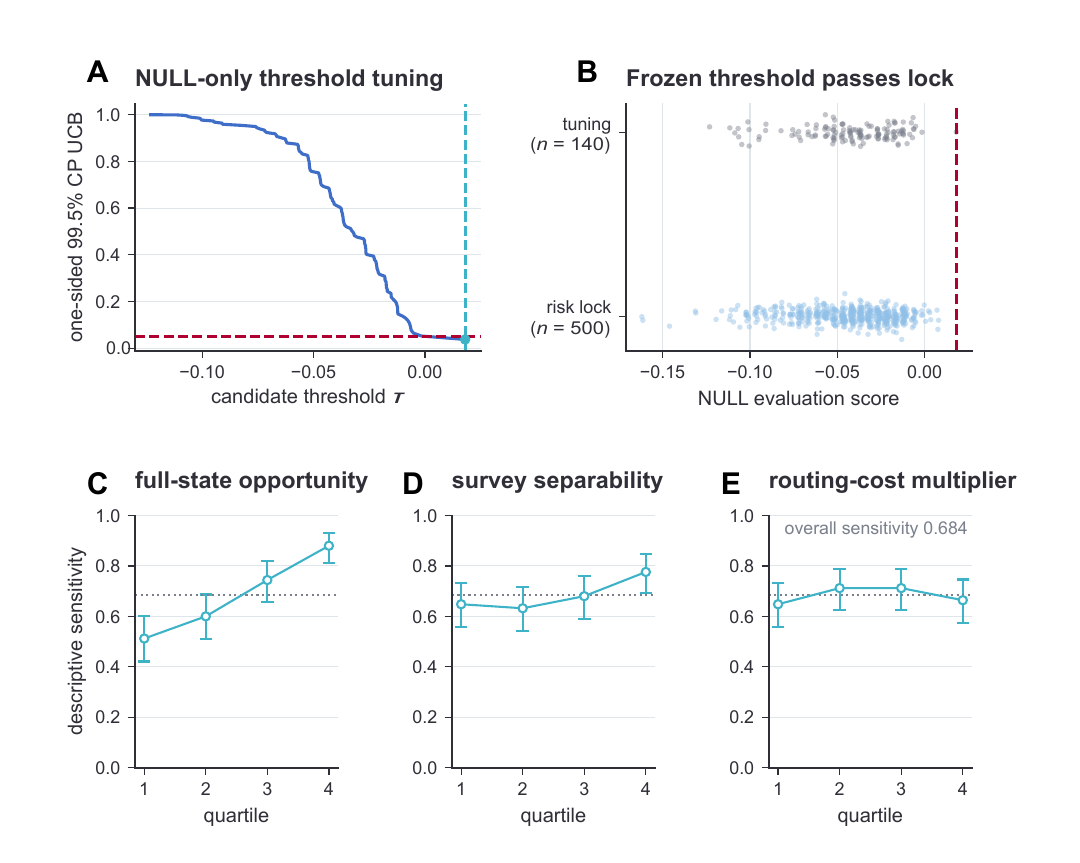}
\caption{\textbf{NULL-only risk calibration and descriptive
positive-population heterogeneity.} \textbf{A}, One-sided 99.5\%
Clopper--Pearson upper risk bound across order-statistic thresholds computed
from 140 NULL tuning fixtures. The horizontal line is the 0.05 target and the
vertical line is the selected \(\tau=0.0180806\). \textbf{B}, Tuning and
independent risk-lock scores with the frozen threshold. Jitter is used only
for display. The risk-lock set contained 500 disjoint fixtures and produced
zero activations; no positive or validation outcome entered threshold
selection. Positive-validation fixtures were divided into quartiles of
\textbf{C}, full-state opportunity \(G_M\); \textbf{D}, survey separability;
and \textbf{E}, routing-cost multiplier. Points are activation fractions and
error bars are two-sided 95\% exact Clopper--Pearson intervals. The dotted line
is the overall sensitivity, 0.684. These post hoc strata were not used to
modify \(\tau\).}
\label{fig:ed_d6_calibration_heterogeneity}
\end{figure}

\subsubsection{Shared-score gate ablation}\label{sec:sr_d6_comparators}

After observing the primary validation result, we specified a seven-rule
threshold ablation before calculating its summaries. It included always-abstain
and always-activate endpoints, a fixed zero-score threshold, a Gaussian 95\%
plug-in threshold, the empirical 95th percentile of the 140 NULL tuning scores,
the original risk-controlled threshold and a non-deployable oracle regime-label
ceiling. The zero, Gaussian and empirical rules shared the frozen evaluation
score and differed only in threshold construction. No threshold was fitted to
validation outcomes.

On the 400 NULL/500 positive validation fixtures, the zero and Gaussian rules
each activated one NULL fixture and 432 and 427 positive fixtures,
respectively. Their sensitivities were 0.864 and 0.854, and their one-sided
99.5\% NULL upper bounds were both 0.01843. The empirical 95th-percentile rule
activated 21 NULL and 458 positive fixtures, giving sensitivity 0.916 but a
NULL upper bound of 0.08819. Relative to the original risk-controlled rule,
paired sensitivity differences were 0.180 (95\% bootstrap interval
[0.146,0.214]), 0.170 [0.138,0.202] and 0.232 [0.194,0.268], while
false-activation differences were \(2.5\times10^{-3}\), \(2.5\times10^{-3}\) and 0.0525. The complete
battery is Supplementary Fig.~\ref{fig:ed_d6_comparators} and
Table~\ref{tab:sr_d6_comparator_battery}. These unadjusted comparisons are
descriptive: only the original rule passed through the prespecified
tuning-and-lock procedure required for a binding claim.

\begin{figure}[!htbp]
\centering
\suppfigure{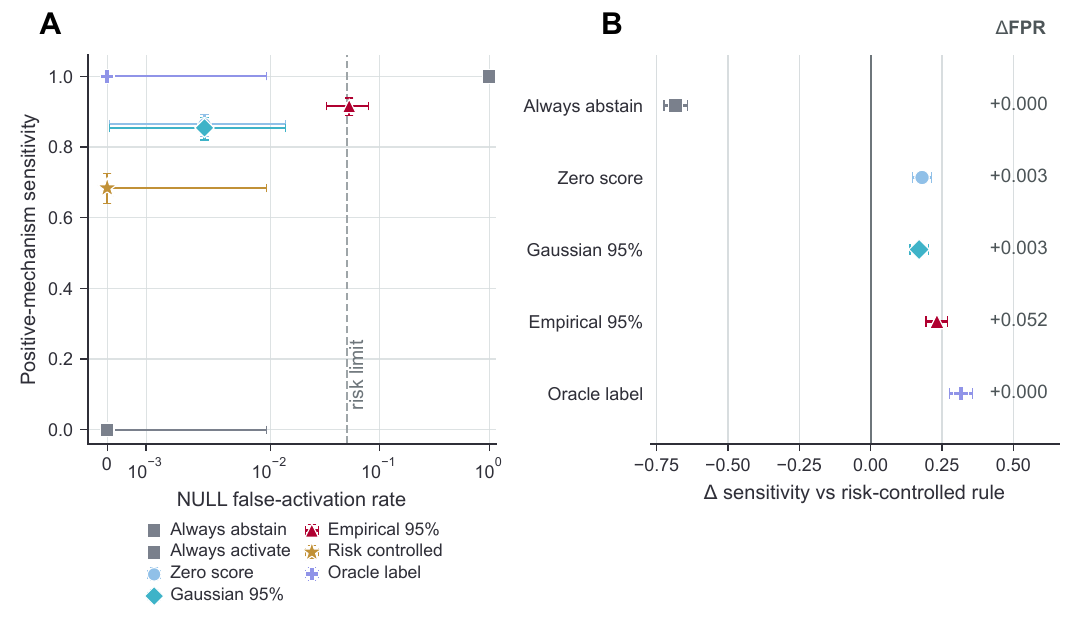}
\caption{\textbf{Descriptive shared-score gate ablation for the held-out
simulator.} \textbf{A}, NULL false-activation rate against
positive-mechanism sensitivity for seven rules evaluated on the same 400 NULL
and 500 positive validation fixtures. Error bars are two-sided 95\% exact
Clopper--Pearson intervals; the vertical line is the original 0.05 risk limit.
The horizontal axis uses a symmetric-log transform to retain zero-rate rules
and the always-activate endpoint. \textbf{B}, Paired sensitivity differences
from the false-activation-controlled rule with unadjusted 95\%
fixture-bootstrap intervals. The right column gives the corresponding
false-activation-rate difference. The oracle label is non-deployable. This
retrospective ablation reuses one oracle-assisted evaluation score and changes
only its threshold; it does not modify the original binding claim.}
\label{fig:ed_d6_comparators}
\end{figure}

\begin{table}[!t]
\centering
\caption{\textbf{Frozen retrospective shared-score gate ablation on the independent-validation fixtures.} Exact intervals are two-sided 95\% intervals for sensitivity; NULL values in parentheses are one-sided 99.5\% upper bounds. All non-oracle rules reuse the same evaluation-only score and therefore do not constitute full implementations of neighbouring authorization methods.}
\label{tab:sr_d6_comparator_battery}
\fontsize{7.4}{8.3}\selectfont
\renewcommand{\arraystretch}{1.1}
\setlength{\tabcolsep}{2.3pt}
\rowcolors{2}{tablerow}{white}
\begin{tabularx}{\textwidth}{@{}L{2.3cm}C{1.7cm}C{2.0cm}C{2.5cm}L{2.4cm}Y@{}}
\specialrule{\heavyrulewidth}{0pt}{0pt}
\rowcolor{tablehead}
Rule & Threshold & NULL activated & Positive activated & Sensitivity (95\% CI) & Role \\
\specialrule{\lightrulewidth}{0pt}{0pt}
Always abstain & \(+\infty\) & 0/400 (0.0132) & 0/500 & 0 (\(0\)--\(7\times10^{-3}\)) & Negative control \\
Always activate & \(-\infty\) & 400/400 (1) & 500/500 & 1 (0.993--1) & Positive endpoint \\
Zero score & 0 & 1/400 (0.0184) & 432/500 & 0.864 (0.831--0.893) & Secondary \\
Gaussian 95\% & \(1.824\times10^{-3}\) & 1/400 (0.0184) & 427/500 & 0.854 (0.820--0.884) & Secondary \\
Empirical 95\% & \(-7.643\times10^{-3}\) & 21/400 (0.0882) & 458/500 & 0.916 (0.888--0.939) & Secondary \\
Risk-controlled rule & 0.018081 & 0/400 (0.0132) & 342/500 & 0.684 (0.641--0.725) & Original binding rule \\
Oracle regime label & truth label & 0/400 (0.0132) & 500/500 & 1 (0.993--1) & Non-deployable ceiling \\
\specialrule{\heavyrulewidth}{0pt}{0pt}
\end{tabularx}
\end{table}

\subsection{Measured archive states and trajectory constructibility}\label{sec:sr_external_group}

The core measured-archive evaluations resolve three evidence states. Cell
Painting tests a target-calibrated adaptive-replicate rule, LINCS--LJP asks
whether a matched 24-h transcriptomic profile adds net value for predicting a
72-h growth-rate endpoint after a 3-h profile, and CTRP provides a
pre-partitioned locked evaluation in which every retained capacity action can
be valued by measured lookup. Causal Chambers has a separate role: its
post-selection analysis asks whether an author-defined proxy computation can
be run on measured trajectories. It is constructibility evidence rather than
an additional policy evaluation.

\subsubsection{Compatibility with measured physical-system archives}\label{sec:sr_ds03_audit}

We pinned the Causal Chambers repository at commit
\texttt{6da4e646\allowbreak{}cf62bc90\allowbreak{}2f87911c\allowbreak{}866298a8\allowbreak{}c98ad2eb}. We downloaded all 14
non-image archives listed at that commit and verified their official MD5
checksums. Six image archives were excluded before
measurement inspection because they require a separate representation and
compute contract. One named experiment CSV was counted as one candidate
physical session; timestamp rows within a CSV were treated as repeated
measurements.

Thirteen scalar datasets met the five compatibility criteria: ordered timestamps, at least two
candidate scalar measurement channels, an actuator--response schema, a
deterministic experiment- or descriptive time-level split and a
byte-reproducible raw-to-manifest transform. Randomized intervention archives
did not contain a logged per-row action propensity, and the remaining
interventional archives were descriptive. The largest individual dataset contained 59
experiment sessions. None met the frozen requirement of at least 104
independent NULL validation sessions, derived from
\(1-(5\times10^{-3})^{1/n}<0.05\), together with propensity support.

These properties restricted the analysis to physical constructibility. No
prospective execution of the complete measurement rule on an instrument was
conducted. The complete task
registry contains 20 released datasets and the excluded metadata template.
The session manifest records file identities, row counts, timestamp checks and
deterministic split assignments.

\subsubsection{Physical-archive constructibility}\label{sec:sr_study4p}

The physical-archive constructibility analysis followed the compatibility
assessment. Its design was fixed after an initial analysis that included the
locked-suffix predictive comparisons. The reported results are therefore
post-selection and descriptive. They do not provide prospective validation of
an instrument-executed measurement rule or independent evidence of risk control.

The physical measurements came from the published wind-tunnel and light-tunnel
Causal Chambers archives~\cite{gamella2025}. Individual response, actuator and
candidate measurement columns were checked against the pinned scalar schemas.
We defined the capacity construction as follows: \(K_0\) used mandatory
actuator/state variables and lagged response history, \(K=1\) added one optional
scalar measurement family, and \(K=2\) added two such families. For wind-tunnel
data the response was downstream pressure and the two added families grouped
auxiliary pressure and acoustic channels. For light-tunnel data the response
was visible-light sensor 1 and the added families grouped infrared photodiode
and infrared thermal channels. These groupings are archive-backed predictive
measurement-capacity proxies, not documented physical actions or randomized
treatments.

Within a session, the pilot was the strict time prefix
\[
n_{\mathrm{pilot}}
=
\min\{500,\max[100,\lfloor0.2n_{\mathrm{rows}}\rfloor]\}.
\]
For each proxy capacity, a standardized ridge model produced rolling-origin
next-step response errors within the pilot. The descriptive score was the
normalized \(K_0\) loss minus the smaller normalized loss under \(K=1\) or
\(K=2\). A greedy rule selected the lowest pilot loss; the reported
one-standard-error rule deviated from \(K_0\) only when the pilot improvement
exceeded its session-level standard error. Missing or insufficient pilot data,
out-of-distribution features, degenerate uncertainty, prediction failure or the
absence of positive pilot improvement returned \(K_0\). These rules define
archive-level proxy actions rather than physical interventions.

All 13 scalar datasets supported the score--action--fallback computation
(Table~\ref{tab:sr_study4p_datasets}). The archive contained 196 sessions, of
which 131 produced a finite pilot score. The one-standard-error rule selected
\(K_0\) in 138 sessions, \(K=1\) in 24 and \(K=2\) in 34. The 138 fallback
decisions comprised 29 insufficient-pilot cases, 73 cases without positive
pilot improvement and 36 out-of-distribution cases. Six datasets supported
experiment-level session summaries; the others retained descriptive time
splits and were not relabeled as independent experiments.

\begin{table}[!t]
\centering
\caption{\textbf{Archive-level physical constructibility.} Fractions are over archived sessions. ``Exp.'' denotes an experiment-level split; ``time'' denotes descriptive time segmentation only. Activity refers to the author-constructed proxy action.}
\label{tab:sr_study4p_datasets}
\fontsize{7.2}{8.1}\selectfont
\renewcommand{\arraystretch}{1.06}
\rowcolors{2}{tablerow}{white}
\begin{tabularx}{\textwidth}{@{}L{2.9cm}C{1.0cm}C{1.0cm}C{1.25cm}C{1.3cm}C{1.3cm}C{0.9cm}@{}}
\specialrule{\heavyrulewidth}{0pt}{0pt}
\rowcolor{tablehead}
Archive & System & Sessions & Split & Finite score & Proxy active & OOD \\
\specialrule{\lightrulewidth}{0pt}{0pt}
WT change-points & wind & 10 & Exp. & 0.900 & 0.500 & 0.100 \\
WT walks & wind & 28 & Exp. & 0.893 & 0.750 & 0.107 \\
LT walks & light & 3 & time & 0.667 & 0 & 0.333 \\
LT test & light & 4 & time & 0.750 & 0.750 & 0.250 \\
WT intake impulse & wind & 5 & time & 0.200 & 0 & 0.800 \\
LT interventions & light & 59 & Exp. & 0.831 & 0.186 & 0.170 \\
LT Malus & light & 12 & Exp. & 0.833 & 0.750 & 0.167 \\
LT validation & light & 29 & Exp. & 0.103 & 0.103 & 0 \\
WT validation & wind & 28 & Exp. & 0.643 & 0.036 & 0.250 \\
WT Bernoulli & wind & 3 & time & 0.333 & 0 & 0.667 \\
WT test & wind & 10 & time & 0.600 & 0.300 & 0.400 \\
WT pressure control & wind & 1 & time & 1.000 & 0 & 0 \\
LT SPICE & light & 4 & time & 0.750 & 0.500 & 0.250 \\
\specialrule{\heavyrulewidth}{0pt}{0pt}
\end{tabularx}
\end{table}

We recomputed every action after replacing or permuting suffix response and
candidate-measurement values. Across 196 sessions, scores, actions and fallback
reasons were unchanged, and every replay was deterministic. The implemented
proxy action therefore depends only on prefix data. This data-flow property
does not show that the action improves a physical outcome.

Locked-suffix predictive loss was available for 30 experiment-level sessions,
and comparator intervals overlapped broadly (Table~\ref{tab:sr_study4p_loss}).
The mean normalized loss was \(0.34765\) for \(K_0\), \(0.34417\) for the
greedy proxy and \(0.34433\) for the one-standard-error proxy. Even the
non-deployable archive oracle, selected after observing suffix prediction
loss, had mean \(0.34257\). Thus, the archive showed no descriptive predictive
advantage over \(K_0\). This endpoint is prediction loss on archived
trajectories rather than reward, scientific success, treatment effect or policy
value.

\begin{table}[!t]
\centering
\caption{\textbf{Descriptive archival predictive loss.} Intervals are session-bootstrap intervals over 30 locked, experiment-level sessions. The archive oracle observes suffix prediction loss and is non-deployable.}
\label{tab:sr_study4p_loss}
\small
\renewcommand{\arraystretch}{1.10}
\rowcolors{2}{tablerow}{white}
\begin{tabularx}{\textwidth}{@{}L{4.3cm}C{2.4cm}C{2.4cm}Y@{}}
\specialrule{\heavyrulewidth}{0pt}{0pt}
\rowcolor{tablehead}
Comparator & Mean normalized loss & Session-bootstrap interval & Role \\
\specialrule{\lightrulewidth}{0pt}{0pt}
\(K_0\) & 0.34765 & 0.15771--0.58760 & Fixed proxy baseline \\
Fixed \(K=1\) & 0.34364 & 0.15480--0.57915 & Descriptive comparator \\
Fixed \(K=2\) & 0.34288 & 0.15477--0.57900 & Descriptive comparator \\
Greedy proxy & 0.34417 & 0.15463--0.58210 & Pilot-prefix descriptive rule \\
One-standard-error proxy & 0.34433 & 0.15475--0.58233 & Pilot-prefix fallback rule \\
Archive oracle & 0.34257 & 0.15458--0.57870 & Suffix-informed; non-deployable \\
\specialrule{\heavyrulewidth}{0pt}{0pt}
\end{tabularx}
\end{table}

All 13 proxy groupings were author-constructed, and none represented a
documented physical \(K\) action. The largest archive contained 59 independent
sessions; no dataset met the 104-session NULL-validation requirement or
provided logged per-action propensities. The evidence therefore establishes
schema-level constructibility, prefix-only replay and the absence of a
descriptive archival predictive advantage. Prospective validation of causal
policy value and false-activation control requires a physical study with
documented actions, propensities and enough independent sessions
(Supplementary Fig.~\ref{fig:ed_constructibility}).

\begin{figure}[!htbp]
\centering
\suppfiguretall{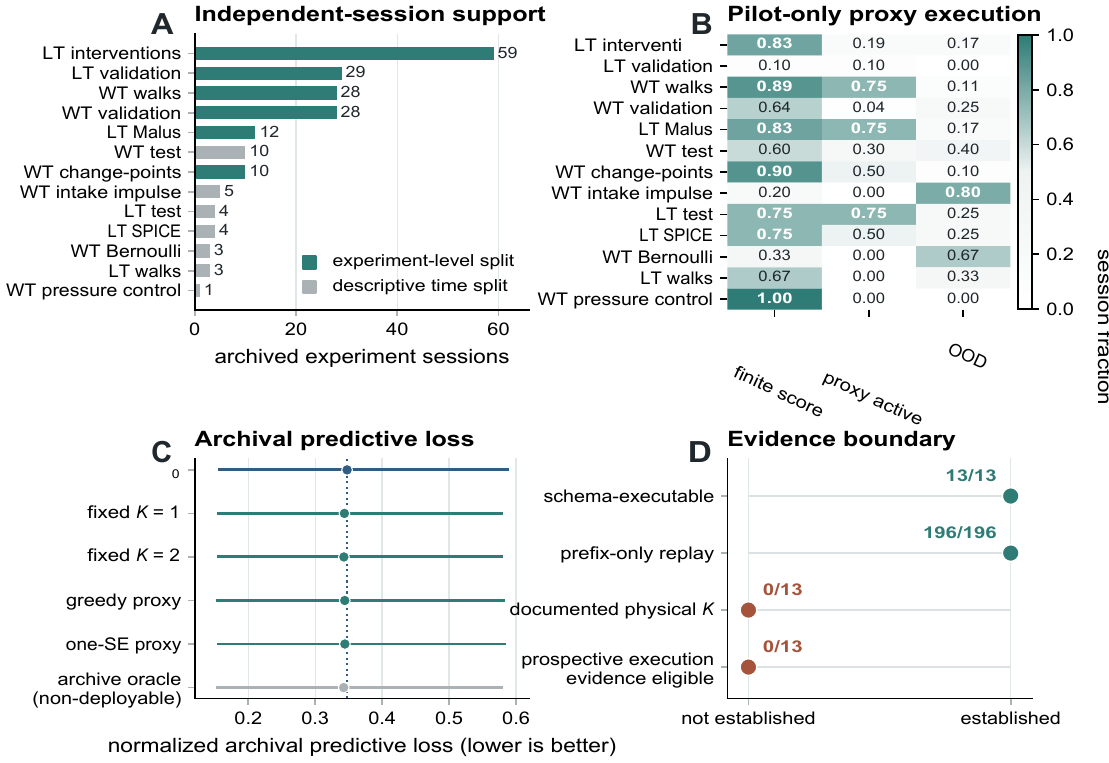}
\caption{\textbf{Constructibility of the score--action--fallback path on
public physical-system archives.} \textbf{A}, Numbers of archived experiment
sessions in 13 scalar Causal Chambers datasets. Dark bars denote datasets with
experiment-level splits eligible for session-level summaries; light bars
denote descriptive time splits. \textbf{B}, Fractions of sessions with a
finite pilot-prefix score, an active proxy action and an out-of-distribution
flag. All 13 schemas admitted score--action--fallback execution, but activity
is not a physical treatment effect. \textbf{C}, Mean normalized archival
predictive loss over 30 locked descriptive sessions. Lines are
session-bootstrap intervals. The archive oracle uses suffix outcomes and is
non-deployable. Intervals overlap broadly, and no descriptive predictive
advantage over \(K_0\) was observed. \textbf{D}, Suffix mutation produced zero
score, action, fallback-reason or deterministic-replay mismatches across 196
sessions, establishing a property of the code data flow. The proxies were
author-defined; none of the 13 mappings was a documented physical \(K\) action,
and no archive supplied the independent sessions and propensities required for
prospective execution of the complete measurement rule on an instrument.}
\label{fig:ed_constructibility}
\end{figure}

\FloatBarrier

\subsubsection{CTRP falls back because the frozen score misses measured opportunity}\label{sec:sr_ctrp}

\noindent\textbf{Source, unit and frozen partitions.}
The locked pharmacogenomic evaluation used exact outcomes from CTRP v2, which
reports measured small-molecule sensitivity across cancer cell
lines~\cite{seashoreludlow2015}. The independent unit was the
\texttt{master\_ccl\_id} family. Only experiments labeled
\texttt{SNP-matched-reference} and families with post-quality-control
measurements for all four frozen compounds were eligible. A metadata-only
salted hash assigned families to DEVELOPMENT below 0.20, CALIBRATION
from 0.20 to below 0.60 and LOCKED\_EVALUATION at or above 0.60 in the
namespace \texttt{OPAL\_HTE03\_CTRP\_FAMILY\_SPLIT\_V1}. The design snapshot
contained 126, 244 and 259 families, respectively. Outcome parsing retained
125 development and 254 locked families after the frozen complete-panel-coverage
and finite-value checks; calibration retained all 244. No family crossed
a partition.

The target of inference is the fixed locked archive, so its activation and
paired-value counts are exact finite-archive quantities. Protocol-specified
Clopper--Pearson intervals provide repeated-sampling summaries under a model
of independent Bernoulli family-level indicators with a common event
probability within each denominator. They do not quantify uncertainty about
the observed archive or transport across tissue lineages, assay batches or
future CTRP-like populations.

The candidate panel was selected before response-value access from compounds
labeled FDA and standard-of-care, with one compound from each of four
mechanistic axes and maximum joint family coverage as the selection rule
(Table~\ref{tab:ctrp_actions}). Panel selection used annotations and binary
availability only. It did not rank compounds or cell lines by response.

\begin{table}[!t]
\centering
\caption{\textbf{Frozen CTRP compound panel and capacity contract.} Compound identities and the nested \(K\) actions were fixed before response-value access. The panel is an assay-capacity abstraction and is not a treatment recommendation.}
\label{tab:ctrp_actions}
\small
\renewcommand{\arraystretch}{1.10}
\rowcolors{2}{tablerow}{white}
\begin{tabularx}{\textwidth}{@{}L{2.2cm}C{1.5cm}L{2.6cm}L{2.1cm}Y@{}}
\specialrule{\heavyrulewidth}{0pt}{0pt}
\rowcolor{tablehead}
Compound & CTRP ID & Mechanistic axis & Eligibility & Role in action \\
\specialrule{\lightrulewidth}{0pt}{0pt}
Paclitaxel & 26956 & Microtubule & FDA; standard-of-care & Eligible for top-\(K\) rank \\
Doxorubicin & 36599 & Topoisomerase & FDA; standard-of-care & Eligible for top-\(K\) rank \\
Bortezomib & 411737 & Proteasome & FDA; standard-of-care & Eligible for top-\(K\) rank \\
Vorinostat & 56554 & Histone deacetylase & FDA; standard-of-care & Eligible for top-\(K\) rank \\
\specialrule{\heavyrulewidth}{0pt}{0pt}
\end{tabularx}
\end{table}

\medskip\noindent\textbf{Decision observables, actions and value.}
The policy used only metadata available before the four panel outcomes:
primary site, broad histology, histology subtype, growth mode, culture medium,
median \(\log(1+\text{cells per well})\), median
\(\log(1+\text{baseline signal})\), and missingness indicators. Categorical
aggregation used the modal value with lexicographic tie breaking; unseen
development categories mapped to \texttt{<OTHER>}. We excluded experiment
date, run identifier, curve-quality statistics, response values, truth strata
and measured opportunity.

A development-only response ranker ordered the four compounds for each family.
The nested actions \(K=1,2,4\) tested the top one, top two or all four ranked
compounds. Each compound corresponds to a 16-concentration series in
duplicate, so the resource counts are 32, 64 and 128 compound wells. Repeated
family--compound curves were aggregated by the prespecified median. With
measured efficacy \(R_{ij}=1-\mathrm{AUC}_{ij}\), the value was
\[
V_i(K;\lambda)
=
\max_{j\in\operatorname{top}K_i}R_{ij}-\lambda(K-1).
\]
The primary normalized cost was \(\lambda=0.01\), with
\(\lambda=0,5\times10^{-3},0.02\) used for descriptive sensitivity analyses.
The resulting values are normalized assay quantities rather than monetary,
patient-benefit or facility-level scientific utilities.

\medskip\noindent\textbf{Development model and score.}
The response ranker was a global ridge model with compound identity and
compound-by-cell-feature interactions. Five outer family folds generated
out-of-fold predictions; four inner folds selected
\(\alpha\in\{0.1,1,10,100\}\) by mean squared error and a one-standard-error
rule favoring the larger penalty. Predictions were clipped to \([0,1]\).
Separate global ridge models predicted \(V_i(K;0.01)\). Under the
one-standard-error rule, which favored smaller \(K\), the best fixed action in
out-of-fold development predictions was \(K_0=2\). This external study did not
use a mixture-of-experts or probe component.

Five hundred family-bootstrap refits estimated the standard error of each
predicted contrast. Let \(\widehat k_i=\arg\max_K\widehat V_i(K)\), with ties
to smaller \(K\). The observable score was
\[
S_i=
\frac{\widehat V_i(\widehat k_i)-\widehat V_i(K_0)}
{\max(\widehat{\mathrm{se}}_i,10^{-6})}.
\]
If \(\widehat k_i=K_0\), or if a value or uncertainty was
degenerate, the rule assigned score zero and fallback.

\medskip\noindent\textbf{Calibration and locked assignment.}
We defined measured opportunity as
\[
G_i=\max_K V_i(K;0.01)-V_i(K_0;0.01),
\]
This quantity measures panel-relative opportunity under the frozen
four-compound, three-capacity and normalized-cost contract, rather than
unrestricted pharmacological or treatment opportunity. We defined the strata
with NULL \(G_i\leq0.01\), AMBIGUOUS \(0.01<G_i<0.05\), and POSITIVE
\(G_i\geq0.05\). We computed these labels only for evaluation; they never
entered the observable score. The threshold grid was
\(\{0,0.5,1,1.5,2,2.5,3,4\}\). Calibration selected the smallest threshold
whose one-sided exact Clopper--Pearson upper bound at \(\alpha=5\times10^{-3}\) was at
most 0.05, provided at least 104 NULL families were present.

The calibration partition contained 236 NULL, three AMBIGUOUS and five POSITIVE families.
At \(\tau=0\), six NULL families activated and the risk upper bound was
0.06501. Threshold \(\tau=0.5\) was the first eligible member, with zero NULL
activations and upper bound 0.02220. We then hash-bound the fitted encoder,
ranker, value models, fixed action and threshold. The assignment table for all
locked families was generated from observables and sealed before parsing the
locked-response table. Locked responses were opened once, after verification
of the assignment hash.

\medskip\noindent\textbf{Locked inference and outcome.}
The external claim required a one-sided 99.5\% exact NULL false-activation
upper bound at most 0.05 and a one-sided 97.5\% paired lower confidence bound
for mean executed value relative to fixed \(K=2\) strictly above zero.
Sensitivity in POSITIVE families carried
no minimum threshold and used a two-sided 95\% exact interval. All four action
values were exact lookups from retained measured responses; neither a
propensity model nor outcome extrapolation was used.

The locked set contained 244 NULL, three AMBIGUOUS and seven POSITIVE
families. Nine families had model-predicted candidate action \(K=4\), but the
largest score was 0.34542, below the frozen threshold. All 254 assignments
therefore fell back to \(K=2\). The NULL false-activation result was 0/244,
with one-sided 99.5\% upper bound 0.02148, meeting the protocol-defined
repeated-sampling risk component. Because no assignment activated, the
locked assignments were identical to an always-\(K_0\) policy.
Executed paired gain was zero for every family, and its one-sided 97.5\% lower
bound was also zero; the joint claim was therefore not supported. POSITIVE sensitivity was
0/7, with exact 95\% interval \(0\)--0.40962. Mean measured opportunity was 0.01322
overall and 0.20542 among POSITIVE families. The post hoc Spearman rank
correlation between the frozen score and measured opportunity was
\(8.69\times10^{-3}\). All seven POSITIVE families and all three AMBIGUOUS
families had score zero, whereas all nine families with a nonzero score were
NULL. This
near-zero association shows poor transfer of the development-trained value
model to the outcome-defined locked strata, without ruling out other
representations of the same decision-time metadata. The seven POSITIVE
families gave a 95\% sensitivity interval extending to 0.40962, leaving
positive-stratum sensitivity imprecisely estimated
(Table~\ref{tab:ctrp_locked_results}).

\begin{table}[!t]
\centering
\caption{\textbf{Pre-partitioned CTRP calibration and locked-evaluation results.} Counts are exact for the fixed archive. Clopper--Pearson bounds have a secondary repeated-sampling interpretation under a model of independent Bernoulli family-level indicators with a common event probability. Sensitivity is estimation-only.}
\label{tab:ctrp_locked_results}
\small
\renewcommand{\arraystretch}{1.10}
\rowcolors{2}{tablerow}{white}
\begin{tabularx}{\textwidth}{@{}L{1.85cm}L{2.85cm}C{1.9cm}C{2.1cm}Y@{}}
\specialrule{\heavyrulewidth}{0pt}{0pt}
\rowcolor{tablehead}
Stage & Quantity & Estimate & Exact/paired bound & Interpretation \\
\specialrule{\lightrulewidth}{0pt}{0pt}
Calibration, \(\tau=0\) & NULL false activation & 6/236 & 99.5\% UCB 0.06501 & Ineligible \\
Calibration, \(\tau=0.5\) & NULL false activation & 0/236 & 99.5\% UCB 0.02220 & Selected threshold \\
Locked & NULL false activation & 0/244 & 99.5\% UCB 0.02148 & Risk threshold met; identical to always-fallback \\
Locked & POSITIVE sensitivity & 0/7 & 95\% CI 0--0.40962 & Estimation only \\
Locked & Paired gain versus \(K=2\) & 0 & 97.5\% lower bound 0 & Value component not supported \\
Locked & Measured opportunity (panel-relative) & 0.01322 overall; 0.20542 POSITIVE & Descriptive & Outcome heterogeneity exists within the frozen panel \\
\specialrule{\heavyrulewidth}{0pt}{0pt}
\end{tabularx}
\end{table}

The molecular panel, measured response example, calibration boundary and
locked score--opportunity result are shown in Supplementary
Fig.~\ref{fig:ed_ctrp}.

\begin{figure}[!htbp]
\centering
\makebox[\linewidth][c]{\includegraphics[width=0.98\textwidth,trim=0 260bp 0 0,clip]{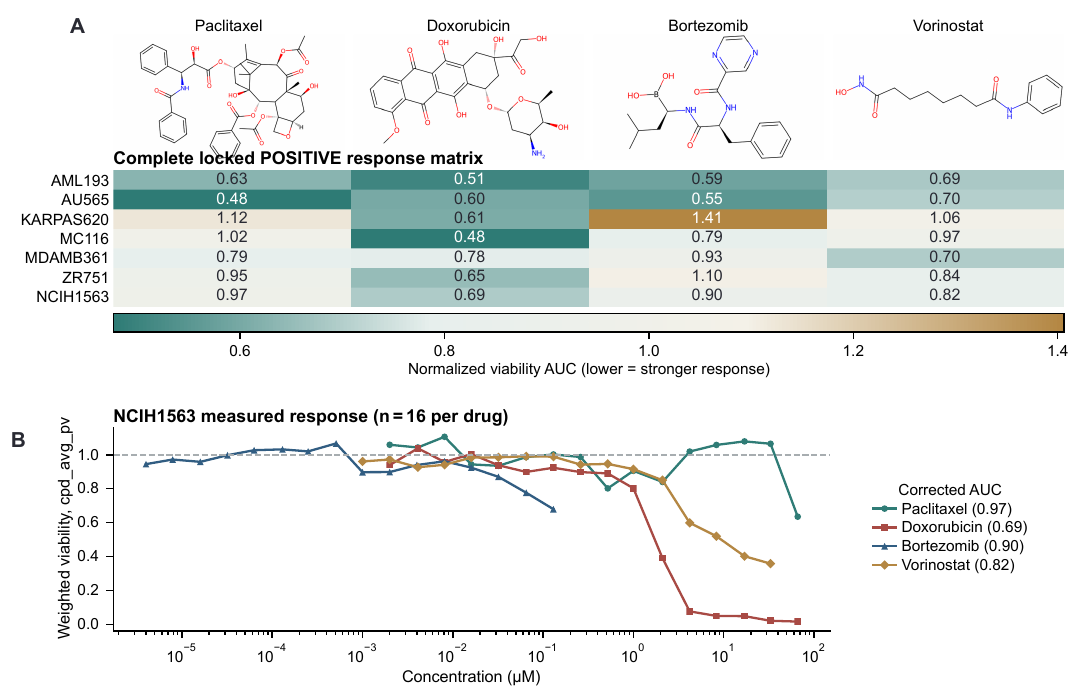}}
\caption{\textbf{A risk-calibrated source-trained score misses measured
opportunity in CTRP.} \textbf{A}, Frozen paclitaxel, doxorubicin, bortezomib
and vorinostat structures above the normalized-viability-AUC matrix for seven
locked POSITIVE families. Structures were not score inputs. POSITIVE denotes
measured opportunity of at least 0.05 under the frozen contract. Each cell has one
curve repeat, and the scale spans 0.47606--1.40590 without clipping.
\textbf{B}, Post-lock NCIH1563 example (family 155421; experiment 793),
selected as the POSITIVE family at median measured opportunity, 0.19019. Symbols
are 16 archive-reported weighted viability summaries per compound; lines guide
the eye, and AUC labels are fitted common-range integrals. Panels \textbf{A}
and \textbf{B} connect the aggregate result to the measured pharmacological
responses.}
\label{fig:ed_ctrp}
\end{figure}

\begin{figure}[!htbp]
\ContinuedFloat
\centering
\makebox[\linewidth][c]{\includegraphics[width=0.98\textwidth,trim=0 0 0 330bp,clip]{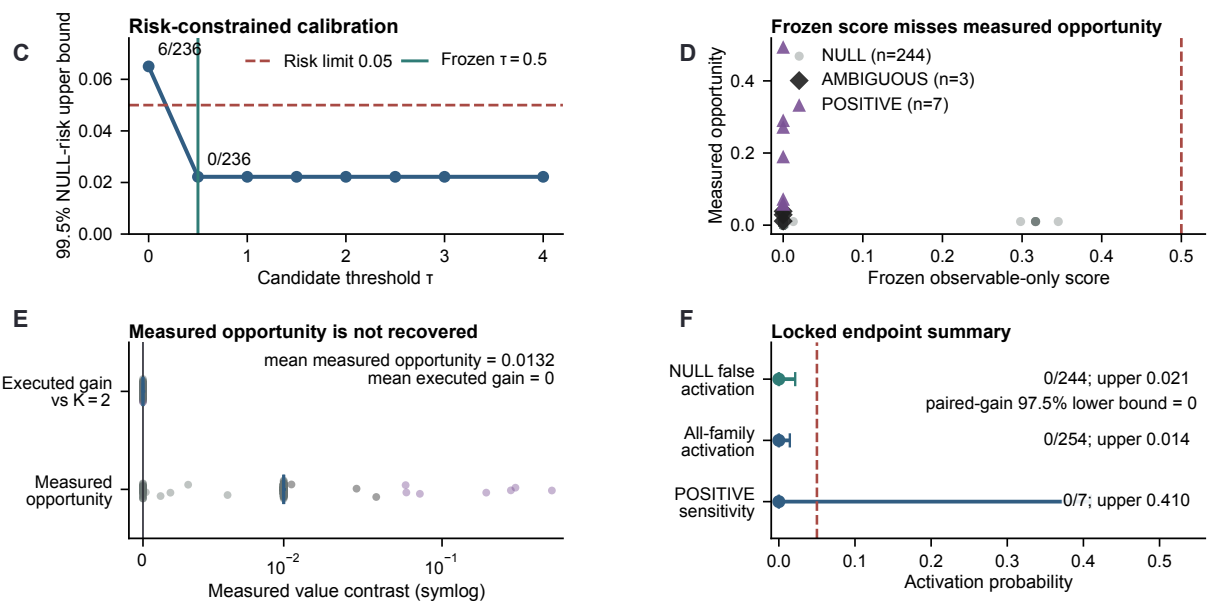}}
\caption{\textbf{A risk-calibrated source-trained score misses measured
opportunity in CTRP (continued).} \textbf{C}, Calibration NULL-risk bounds.
At \(\tau=0\), 6/236 activated (99.5\% upper bound, 0.0650); frozen
\(\tau=0.5\) gave 0/236 (upper bound, 0.0222). \textbf{D}, Locked score versus
measured opportunity. The maximum score was
\(0.3454<0.5\), including all seven POSITIVE families; colors and markers are
evaluation-only strata. \textbf{E}, Measured opportunity and executed gain over
254 locked families. Gain was zero because every assignment retained
\(K_0=2\). \textbf{F}, Exact locked results: NULL false activation 0/244
(99.5\% upper bound, \(0.0215<0.05\)); POSITIVE sensitivity 0/7 (95\%
interval, 0--0.410); and paired-gain 97.5\% lower bound, zero. This is an
offline compound-bundle evaluation.}
\end{figure}

Reproduction revealed two floating-point-only comparison differences:
\(7.1\times10^{-13}\) for a \(t\)-critical value and
\(5.6\times10^{-17}\) for a Clopper--Pearson endpoint. The validator now uses
a \(10^{-12}\) comparison tolerance. Inputs, assignments, outputs, claims and
the one-run seal were unchanged; no scientific analysis was rerun.

The CTRP lock therefore resolved to fallback for the current frozen score. The
score met the NULL-risk threshold through universal fallback, while measured
panel-relative opportunity remained and was not ranked. This result does not
establish exhaustion of decision-time information. It indicates that a revised
score must first demonstrate opportunity ranking in an independent evaluation
before selective acquisition is reconsidered.

\subsubsection{Cell Painting yields a lower-burden sparse candidate but retains the fixed plan}
\label{sec:sr_cpg0012}

\noindent\textbf{Study population and outcome separation.}
The target-calibrated morphology study used the public
\texttt{cpg0012-wawer-bioactivecompoundprofiling} archive, which contains
Cell Painting images and morphological profiles for approximately 30,000
small-molecule treatments~\cite{wawer2014profiling,bray2017cellpaintingdata}.
Before downloading or inspecting any profile, we used design metadata to
reconstruct 30,616 planned compound identifiers, 405 mapped assay plates
and 127,904 planned treatment wells. Requiring four planned wells and
excluding exact provided-structure duplicates yielded 25,013 non-overlapping
canonical-compound inferential units. A salted compound-level hash assigned 2,444
clusters to \texttt{TARGET\_TUNING}, 11,304 to
\texttt{TARGET\_CERTIFICATION} and 11,265 to \texttt{FINAL\_LOCK}; no
replicate or exact provided structure crossed partitions. This procedure
provided author-controlled self-blinding rather than independent broker
blinding.

Four planned replicates received deterministic hashed roles: one initial
decision-time replicate, two gate-contingent additional replicates
and one evaluation-only verifier. For action \(a\), utility was
\[
U_i(a)=
\frac{1+\cos\{\widehat z_i(a),z_{i,\mathrm{verify}}\}}{2}.
\]
The fallback estimate used the initial replicate and the active estimate used
the mean of the initial plus two additional replicates. Net gain was
\(\Gamma_i=U_i(\mathrm{ADD\_TWO})-U_i(\mathrm{FALLBACK})-0.02\).
NULL meant \(\Gamma_i\leq0\), POSITIVE meant
\(\Gamma_i\geq5\times10^{-3}\), and measured values between these limits were
AMBIGUOUS. These labels and the additional and verifier profiles were
evaluation-only.

This utility measures technical replicate-profile agreement for estimating a
compound's morphological profile. It is not a measure of hit efficacy,
mechanism-of-action recovery or biological response. The
\(5\times10^{-3}\) POSITIVE margin was an author-set study-specific operational
threshold rather than a biologically calibrated effect size.

\medskip\noindent\textbf{Model and threshold development.}
Decision-time predictors included initial-replicate morphology, stateless
hashed SMILES features, compound source, plate map, well position, plate
prefix, solvent and planned concentration. The selected
\texttt{MULTI\_EXTRA\_LEAF2} model used three 400-tree ExtraTrees heads for
NULL probability, POSITIVE probability and net gain, with minimum leaf size
two and square-root feature subsampling. The fitted model used the
\texttt{TARGET\_TUNING} partition only.

We chose the model and thresholds on a labeled development population of
13,748 compounds from the two pre-final partitions, before the final test. The
predictor fit remained restricted to \texttt{TARGET\_TUNING}; candidate-rule
selection used the labeled development population. Accordingly, neither
pre-final partition is used as independent confirmatory evidence. The
remaining \texttt{FINAL\_LOCK} provided the single held-out final test. Final
outcomes did not enter model fitting, threshold selection, preprocessing or
assignment.

The amended development search treated the seven quantities as screening
objectives rather than confirmatory tests. None of 53,792 threshold
combinations met all seven. Among 12,673 candidates meeting the six non-FDP
requirements, the unique rule minimizing the FDP-UCB excess above 35\% was
fixed before final outcomes. It
activated a compound when predicted NULL probability was at most 0.395, predicted
POSITIVE probability at least 0.633333, predicted gain at least 0.053625 and
OOD percentile at most 1.0. On the 13,748-compound development population it
activated 722 compounds (5.2517\%), with 473 POSITIVE, 232
false-or-unevaluable and 17 AMBIGUOUS activations. Its FDP point estimate was
32.1330\%; its one-sided 95\% UCB was 35.1117\% against the 35\% screen. The
rule met the six non-FDP screening criteria, including a nominal
compound-bootstrap executed-value lower endpoint of
\(1.667\times10^{-3}\). Development selected but did not confirm the rule; the
still-unopened final partition supplied its sole test.

The OOD score was an empirical percentile on \([0,1]\). The frozen ceiling of
1.0 passed all 11,265 FINAL units and was therefore nonbinding in the final
assignment; the NULL-probability, POSITIVE-probability and predicted-gain
thresholds determined the 595 activations. The OOD term should not be read as a
restrictive final-partition protection in this execution.

\medskip\noindent\textbf{Pre-outcome final assignment.}
Using only final initial-replicate profiles and other permitted observables,
the frozen rule assigned 595/11,265 compounds to
\texttt{ADD\_TWO}. Coverage was 5.2818\% and decision-time feature
missingness was 1.1895\%, meeting the pre-outcome requirements of at least 100
activations, at least 5\% coverage and at most 5\% missingness. The assignment
table, model, configuration, protocol, evaluator and encrypted final outcomes
were sealed before the final key was used.

\medskip\noindent\textbf{Final authorization contract and result.}
The final authorization contract was conjunctive. Four inferential components
used marginal one-sided 95\% bounds: a false-activation UCB at most 7.5\%, an
active-set FDP UCB at most 35\%, a nominal compound-bootstrap BCa executed-value
lower endpoint strictly above zero and a POSITIVE-sensitivity LCB at least 5\%.
Three finite-frame guards required at least 100 activations, at least 5\%
coverage and no more than 5\% missing or unevaluable final outcomes. Their
conjunction gives a nominal intersection--union decision under an iid
canonical-compound model and bootstrap regularity. Conditional on the observed
denominators, the three Clopper--Pearson bounds are exact under their binomial
models; the BCa component has nominal level 0.05. The component bounds do not
form a distribution-free finite-sample guarantee or a simultaneous 95\%
confidence set.

These requirements have operational rather than biological meanings. The
7.5\% NULL-risk limit was introduced in the tuning-stage amendment and fixed
before certification and final-outcome access. It constrains unnecessary
additional profiling at the upper confidence bound. It was not derived from a
biological efficacy threshold, monetary valuation or calibrated
facility-capacity model. The 35\% FDP limit requires a non-false share of at
least 65\% among activations, and the 5\% sensitivity floor excludes negligible activation at
archive scale. The value requirement asks whether gain remains positive after
the prespecified cost. Under the prespecified worst-case completion, all 263
technical failures received a gain of \(-1.02\), the theoretical lower bound
of the utility scale. The nine activated failures counted as false and entered
executed value at that bound; the 254 inactive failures retained the fallback
and entered the worst-case POSITIVE denominator.
Measured AMBIGUOUS compounds were reported separately and did not count as
NULL.

The stricter \(5\%\), \(25\%\), \(25\%\) and \(5\times10^{-3}\) thresholds in
Table~\ref{tab:target_power_plan} belong to an earlier, unexecuted three-way
assurance scenario that required a separate certification partition. Once
both pre-final labeled partitions were assigned to development, that scenario
could not supply confirmatory evidence. The v5 single-held-out contract and
11,265 pre-outcome assignments were sealed before any final outcome was
decrypted and remained unchanged thereafter.

\begin{table}[!htbp]
\centering
\caption{\textbf{Original development-only assurance scenario for a
three-way target-calibrated design.} Rates and values include prespecified
worst-case missingness recoding. They are assumptions informed by opened
development resources, not target estimates or confirmatory evidence.}
\label{tab:target_power_plan}
\small
\renewcommand{\arraystretch}{1.08}
\rowcolors{2}{tablerow}{white}
\begin{tabularx}{\textwidth}{@{}L{3.4cm}L{4.2cm}Y@{}}
\specialrule{\heavyrulewidth}{0pt}{0pt}
\rowcolor{tablehead}
Object & Prespecified development scenario & Consequence \\
\specialrule{\lightrulewidth}{0pt}{0pt}
Joint contract & \(R_{\mathrm{wc}}\leq0.05\), \(F_{\mathrm{wc}}\leq0.25\), \(S_{\mathrm{wc}}\geq0.25\), \(L_\mu\geq5\times10^{-3}\) & The allocation gives simultaneous marginal coverage of at least 0.95 and an IUT false-authorization bound of 0.025; finite-frame guards add no population claim \\
Assurance inputs & NULL prevalence \(\geq0.40\), POSITIVE prevalence \(\geq0.30\), coverage \(\geq0.25\), \(R_{\mathrm{wc}}\leq0.01\), \(F_{\mathrm{wc}}\leq0.10\), \(S_{\mathrm{wc}}\geq0.40\), \(\mu_{\mathrm{wc}}\geq4.8\times10^{-2}\), standard deviation \(\leq0.14\), missingness \(\leq0.01\) & Random NULL, POSITIVE and active denominators are integrated rather than replaced by their expectations \\
Bounded-value cushion & Confirmatory margin \(5\times10^{-3}\); design LCB target \(7.5\times10^{-3}\) on hypothetical development support \([-0.02,1]\) & Avoids selecting a numerical knife-edge at the scientific limit \\
Selected partition size & 1,253 independent units in certification and 1,253 in final & Fixed-rule joint-assurance lower bound 0.965 under the stated assumptions; tuning-rule selection success is excluded \\
Outcome-blind split size & 10\% tuning, 45\% certification, 45\% final; minima 150/1,253/1,253 & 2,900 registry units give simultaneous split-count assurance at least 0.950 by marginal binomial bounds and a union bound; the candidate rounds to 3,000 \\
Value sensitivity & Mean lower bound \(0.04\)--\(0.06\), standard-deviation upper bound \(0.10\)--\(0.20\) & Value-component-only requirement ranges from 713 to 2,444 units \\
\specialrule{\heavyrulewidth}{0pt}{0pt}
\end{tabularx}
\end{table}

The one-shot final result contained 384 POSITIVE, 197 measured NULL, nine
technical-failure and five AMBIGUOUS activations. The nine failures plus 197
NULL activations gave 206 worst-case false activations. Their risk denominator
was 4,453 (4,444 measured NULL compounds plus nine active failures), whereas
the worst-case sensitivity denominator was 6,534 (6,280 measured POSITIVE
compounds plus 254 inactive failures). Thus the component counts were
\(206/4453\) and \(384/6534\). Among the 595 activations, 384 were measured
POSITIVE (64.5378\%); measured POSITIVE compounds comprised 6,280 of all
11,265 final units (55.7479\%). Their ratio gives a descriptive 1.15767-fold
active-set enrichment while retaining AMBIGUOUS and unevaluable
units in both denominators. The false-discovery bound nevertheless exceeded
its target, retaining the fixed plan
(Table~\ref{tab:cpg0012_stage_results}).

These ratios and the completed all-unit mean are exact properties of the locked
archive. The Clopper--Pearson bounds have a repeated-sampling interpretation,
and the BCa endpoint~\cite{efron1987bca} a nominal bootstrap interpretation,
under an iid canonical-compound sampling model for the stated frame. Removing
exact structural duplicates does not establish this independence assumption,
and we did not model plate or batch dependence. The BCa endpoint is therefore
not a distribution-free finite-sample guarantee. As a stress test, the
empirical-Bernstein lower bound was
\(-1.099\times10^{-3}\), whereas the fixed-archive mean after the prespecified
least-favorable missing-outcome completion was
\(1.948\times10^{-3}>0\).

\medskip\noindent\textbf{Post-lock microscopy display QC.}
The two Fig.~2 compounds were already fixed as the active POSITIVE and active
measured-NULL units nearest their respective group-median realized gains. For
display only, we enumerated sites 1--6 for each of their four frozen role wells
and applied one common site index to all eight images. A fixed pixel-only QC
rule lexicographically minimized the worst 8-connected near-black component,
then the worst outer-edge bright-or-blue artifact fraction and then site index.
It selected site 5. This post hoc display rule used neither predicted nor
realized gain and supplied no biological evidence. Every selected TIFF is
bound to its public URL, S3 metadata and SHA-256 in
\texttt{main\_fig2\_selected\_raw\_asset\_manifest.csv}; all candidate and
per-image QC values are reported in the accompanying site-selection manifests.
For display, channel-specific 0.20th and 99.80th percentile limits were pooled
across all 40 bound site-5 W1--W5 TIFFs. Intensities were clipped to those
limits, scaled to \([0,1]\) and transformed with a fixed gamma exponent of
0.82, identically within channel. The decision-time and verifier composites
show the full fields; the five single-channel decision-time panels use a
centered 680-by-680-pixel detail crop. No local adjustment, object removal or
stitching was applied.

\begin{table}[!t]
\centering
\caption{\textbf{Target-calibrated cpg0012 development and one-shot final
results.} The contract binds on the confidence bounds, which are reported
separately from point estimates. Final assignments were fixed before outcome
access. Binomial bounds have a repeated-sampling interpretation and BCa
endpoints a nominal bootstrap interpretation under the stated iid-compound
sampling model.}
\label{tab:cpg0012_stage_results}
\small
\renewcommand{\arraystretch}{1.08}
\rowcolors{2}{tablerow}{white}
\begin{tabularx}{\textwidth}{@{}L{3.1cm}C{2.2cm}C{2.2cm}C{2.2cm}Y@{}}
\specialrule{\heavyrulewidth}{0pt}{0pt}
\rowcolor{tablehead}
Quantity & Pooled development & Final & Internal final requirement & Final interpretation \\
\specialrule{\lightrulewidth}{0pt}{0pt}
Units; activations & 13,748; 722 & 11,265; 595 & \(\geq100\) active & Met \\
Activation coverage & 5.2517\% & 5.2818\% & \(\geq5\%\) & Met \\
POSITIVE activations & 473 & 384 & --- & Observed positive-gain activations \\
False activations & 232 & 206 & --- & NULL plus technical failure \\
AMBIGUOUS activations & 17 & 5 & Report separately & Not counted as NULL \\
False-activation point; 95\% UCB & 4.1667\%; 4.6343\% & 4.6261\%; 5.1776\% & UCB \(\leq7.5\%\) & Met \\
FDP point; 95\% UCB & 32.1330\%; 35.1117\% & \textbf{34.6218\%; 37.9660\%} & UCB \(\leq35\%\) & \textbf{Not met} \\
Positive sensitivity; 95\% LCB & 6.0416\%; 5.6045\% & 5.8770\%; 5.4054\% & LCB \(\geq5\%\) & Met \\
Mean executed value; nominal 95\% BCa lower endpoint & \(2.210\times10^{-3}\); \(1.667\times10^{-3}\) & \(1.948\times10^{-3}\); \(1.229\times10^{-3}\) & lower endpoint \(>0\) & Met \\
Outcome missingness & 2.0221\% & 2.3347\% & \(\leq5\%\) & Met \\
\specialrule{\heavyrulewidth}{0pt}{0pt}
\end{tabularx}
\end{table}

The final FDP point estimate, \(206/595=34.6218\%\), was below the 35\% target,
but its exact one-sided 95\% UCB was 37.9660\%, 2.966 percentage points above
the target. The gate nevertheless activated nontrivially, captured 384
POSITIVE compounds and had a
fixed-archive worst-case mean of \(1.948\times10^{-3}>0\), with a nominal
one-sided 95\% compound-bootstrap BCa lower endpoint of
\(1.229\times10^{-3}>0\). Because the prespecified conjunctive rule binds on
the confidence bound rather than the point estimate, the complete contract
kept the fixed plan in place.

At the fixed final size of 595 activations, the upper-bound criterion permits
at most 188 false activations, hence at least 407 non-false activations. This corresponds
to a non-false share of at least 68.4034\%, compared with 65.3782\% observed. Holding
the 206 false activations fixed, the one-sided 95\% Clopper--Pearson upper bound
first satisfies the 35\% criterion at 647 total activations. This requires 52
additional activations, all non-false; 646 total activations still fail. If the
observed FDP is instead held exactly at (206/595), the first integer multiple
that satisfies the criterion contains 43,435 activations and 15,038 false
activations. These counterfactual calculations provide design guidance under
the prespecified independent-compound model and do not alter the final result.

\medskip\noindent\textbf{Post-hoc score-ranking diagnostic.}\label{sec:cpg_score_sweep}
We separately asked how much POSITIVE enrichment either frozen scalar score
could achieve under a simple rank cutoff. This exploratory calculation joined
the 11,002 compounds with measured outcomes to their pre-outcome assignments
and ranked them by either predicted gain or predicted POSITIVE probability.
The measured POSITIVE baseline was \(6{,}280/11{,}002=57.08\%\). Across every
integer cutoff from 50 to 6,000 compounds, predicted gain reached a maximum
enrichment of 1.2376-fold at \(n=218\) (154 POSITIVE), whereas predicted
POSITIVE probability reached 1.1812-fold at \(n=132\) (89 POSITIVE). Their
full-population AUCs were 0.5419 (Hanley--McNeil 95\% interval,
0.5311--0.5527) and 0.5465 (0.5357--0.5573), respectively.

The prespecified \OPAL{} assignment was not a top-\(n\) cutoff on either score:
it combined multiple frozen criteria, activated 595 compounds and included
586 measured outcomes, of which 384 were POSITIVE. Its measured-only
enrichment was therefore 1.1480-fold. These post-hoc sweeps characterize two
individual scores in this archive; they neither replace the prespecified gate
nor establish a general representation ceiling. The peak cutoffs
(\(n=218\) for predicted gain and \(n=132\) for predicted POSITIVE
probability) were selected after outcome access, and any deployable top-\(n\)
rule would require an acquisition budget fixed before evaluation. The complete integer sweeps
and deterministic summary are supplied as
\texttt{cpg0012\_posthoc\_score\_sweep.csv} and
\texttt{cpg0012\_posthoc\_score\_summary.csv}.

\medskip\noindent\textbf{Post-hoc plate-map dependence sensitivity.}
The 320 physical assay plates form 80 non-overlapping
\texttt{plate\_map\_name} blocks of four plates. Each compound's initial,
two incremental and verifier roles occupy four distinct plates within one
block, so resampling individual plates would split the inferential unit. We
therefore held the final gate, assignments and endpoint-specific
least-favorable completions fixed and resampled 80 complete blocks with
replacement 20,000 times using one shared index set (seed 2026082901) for the
percentile analysis. The paired cluster-BCa sensitivity used 20,000
endpoint-specific paired block resamples with deterministic seeds recorded in
the deposited code (2026083402--2026083405 for the four \OPAL{} endpoints).

We report pointwise one-sided 95\% percentile endpoints: upper
95th percentiles for false activation and false discovery, and lower 5th
percentiles for sensitivity and executed value
(Table~\ref{tab:cpg0012_plate_map_sensitivity}). False activation remained
below 7.5\%, false discovery remained above 35\%, and the executed-value lower
endpoint remained positive. The sensitivity lower endpoint was 4.8592\%, just
below its 5\% floor. A secondary paired cluster-BCa calculation gave the same
risk and sensitivity directions but placed the value lower endpoint slightly
below zero. This method sensitivity was concentrated around a block with 137
of 146 final outcomes missing and all seven active rows recorded as technical
failures.

All 80 plate-map blocks contain both development and final compounds, and plate
metadata entered the decision-time representation. The calculation therefore
tests dependence of the fixed-policy endpoints within the observed layout; it
does not create an evaluation held out by plate-map block, repair development--final
block sharing or replace the prespecified compound-level endpoints. A future
population claim across assay blocks requires block-disjoint splitting or
cross-fitting before model selection and outcome access.

\begin{table}[!t]
\centering
\caption{\textbf{Post-hoc cpg0012 plate-map dependence sensitivity.}
Primary endpoints use the prespecified iid-compound model. Plate-map
endpoints resample 80 four-plate blocks and are pointwise, descriptive and
not adjusted for multiplicity. Paired cluster BCa is reported alongside the
percentile endpoints as a method-sensitivity calculation. Neither
changes the prespecified seven-component decision.}
\label{tab:cpg0012_plate_map_sensitivity}
\small
\renewcommand{\arraystretch}{1.08}
\rowcolors{2}{tablerow}{white}
\begin{tabularx}{\textwidth}{@{}Y C{1.55cm} C{2.0cm} C{2.0cm} C{1.85cm} C{1.35cm}@{}}
\specialrule{\heavyrulewidth}{0pt}{0pt}
\rowcolor{tablehead}
Endpoint & Point & Primary compound-level endpoint & Block percentile & Block BCa & Target \\
\specialrule{\lightrulewidth}{0pt}{0pt}
False activation & 4.6261\% & UCB 5.1776\% & UCB 5.6525\% & UCB 5.6923\% & \(\leq7.5\%\) \\
False discovery & 34.6218\% & UCB 37.9660\% & UCB 37.8165\% & UCB 37.8739\% & \(\leq35\%\) \\
POSITIVE sensitivity & 5.8770\% & LCB 5.4054\% & LCB 4.8592\% & LCB 4.9246\% & \(\geq5\%\) \\
Executed value & \(1.948\times10^{-3}\) & LCB \(1.229\times10^{-3}\) & LCB \(5.319\times10^{-4}\) & LCB \(-9.083\times10^{-5}\) & \(>0\) \\
\specialrule{\heavyrulewidth}{0pt}{0pt}
\end{tabularx}
\end{table}

\medskip\noindent\textbf{Common-contract comparator battery.}
The prespecified applicable comparators used identical final compounds,
decision-time scores, action, endpoint and normalized cost
(Table~\ref{tab:cpg0012_comparators}). Always-fallback was the explicit
non-triviality negative control. The expected-net-benefit gate had the largest
reported held-out value mean and BCa lower endpoint, but activated 96.01\% of
compounds and had a 97.14\% false-activation UCB. Forced activation, the simple
uncertainty gate and the risk-only gate likewise achieved broader coverage and
higher value than \OPAL{} while crossing the false-activation limit. \OPAL{}
retained a positive value lower endpoint at 5.28\% coverage and was the only
active rule below the 7.5\% limit. The battery therefore measures a
risk--value frontier rather than an unconditional performance ranking. HCPI,
SPIBB and CSPI-MT require a logged behavior policy and action propensities,
which this fixed archive does not contain. Conservative bandit and
Bayesian-optimization baselines require online interaction. We therefore
report these native data requirements without constructing oracle-assisted
surrogates.

\begin{table}[!t]
\centering
\caption{\textbf{cpg0012 final comparator battery under common observables,
action and cost.} Binomial bounds are one-sided 95\%; values are all-unit
worst-case executed gains with nominal compound-bootstrap BCa
lower endpoints. FDP is undefined
for zero activation; the software value of one is a conservative gate
convention, not a confidence bound. No method met the complete joint contract.}
\label{tab:cpg0012_comparators}
\fontsize{8.1}{9.0}\selectfont
\renewcommand{\arraystretch}{1.08}
\rowcolors{2}{tablerow}{white}
\begin{tabularx}{\textwidth}{@{}Y C{1.35cm}C{1.8cm}C{1.45cm}C{1.8cm}C{1.8cm}C{1.25cm}@{}}
\specialrule{\heavyrulewidth}{0pt}{0pt}
\rowcolor{tablehead}
Method & Coverage & False-activation UCB & FDP UCB & Sensitivity LCB & Value BCa LCB & Added wells \\
\specialrule{\lightrulewidth}{0pt}{0pt}
Always-fallback & 0.00\% & 0.07\% & NA\(^\dagger\) & 0.00\% & 0 & 0 \\
Forced activation & 100.00\% & 100.00\% & 42.55\% & 99.95\% & \(2.064\times10^{-3}\) & 22,530 \\
Simple uncertainty gate & 81.86\% & 83.74\% & 41.65\% & 80.88\% & \(1.4554\times10^{-2}\) & 18,442 \\
ENB gate & 96.01\% & 97.14\% & 41.64\% & 95.13\% & \(1.4829\times10^{-2}\) & 21,630 \\
Risk-only gate & 38.83\% & 37.94\% & 39.06\% & 39.58\% & \(7.750\times10^{-3}\) & 8,748 \\
\OPAL{} final gate & 5.28\% & \textbf{5.18\%} & \textbf{37.97\%} & 5.41\% & \(1.229\times10^{-3}\) & 1,190 \\
\specialrule{\heavyrulewidth}{0pt}{0pt}
\end{tabularx}
\end{table}

\medskip\noindent\textbf{Post-hoc risk-budget and cost sensitivity.}
We first varied only the false-activation risk budget at the prespecified cost
\(c=0.02\), retaining the six pre-outcome assignments and their final-test
statistics. An active rule was risk-feasible when its one-sided 95\%
false-activation UCB did not exceed the candidate budget. Under this
definition, \OPAL{} was the only risk-feasible active rule for every
\(r_0\in[5.1776105886\%,37.9356104264\%)\). At the upper endpoint, the risk-only gate
also became feasible. This descriptive threshold sensitivity concerns the
false-activation component alone, rather than the complete seven-component
contract.

We next varied the normalized incremental cost of the active two-well
bundle while holding all six assignments fixed (Supplementary
Fig.~\ref{fig:cpg_real_data_validation}, D, and Supplementary
Fig.~\ref{fig:cpg_cost_sensitivity}; crossings in Supplementary
Table~\ref{tab:cpg_cost_crossovers}). The nominal one-sided 95\% BCa lower
endpoint crossed zero first for forced activation, followed by the ENB,
simple-uncertainty, risk-only and \OPAL{} rules. The corresponding bundle costs
were 0.02206398, 0.03543449, 0.03778481, 0.03997444 and 0.04320122. Thus,
approximately over \(c\in[0.03997444,0.04320122)\), \OPAL{} was the sole active
rule with a positive nominal BCa lower endpoint. The sequence reflects the
cost exposure of the frozen decisions: the other active rules assigned
8,748--22,530 additional wells, whereas the frozen \OPAL{} gate assigned 1,190. This
fixed-assignment analysis does not compare policies retrained or reselected at
the alternative costs.

The score--gain relationship, selected and nonselected gain distributions,
plate-map dependence analysis and fixed-assignment cost response are assembled
in Supplementary Fig.~\ref{fig:cpg_real_data_validation}.

\begin{figure}[!htbp]
\centering
\makebox[\linewidth][c]{\includegraphics[width=0.98\textwidth]{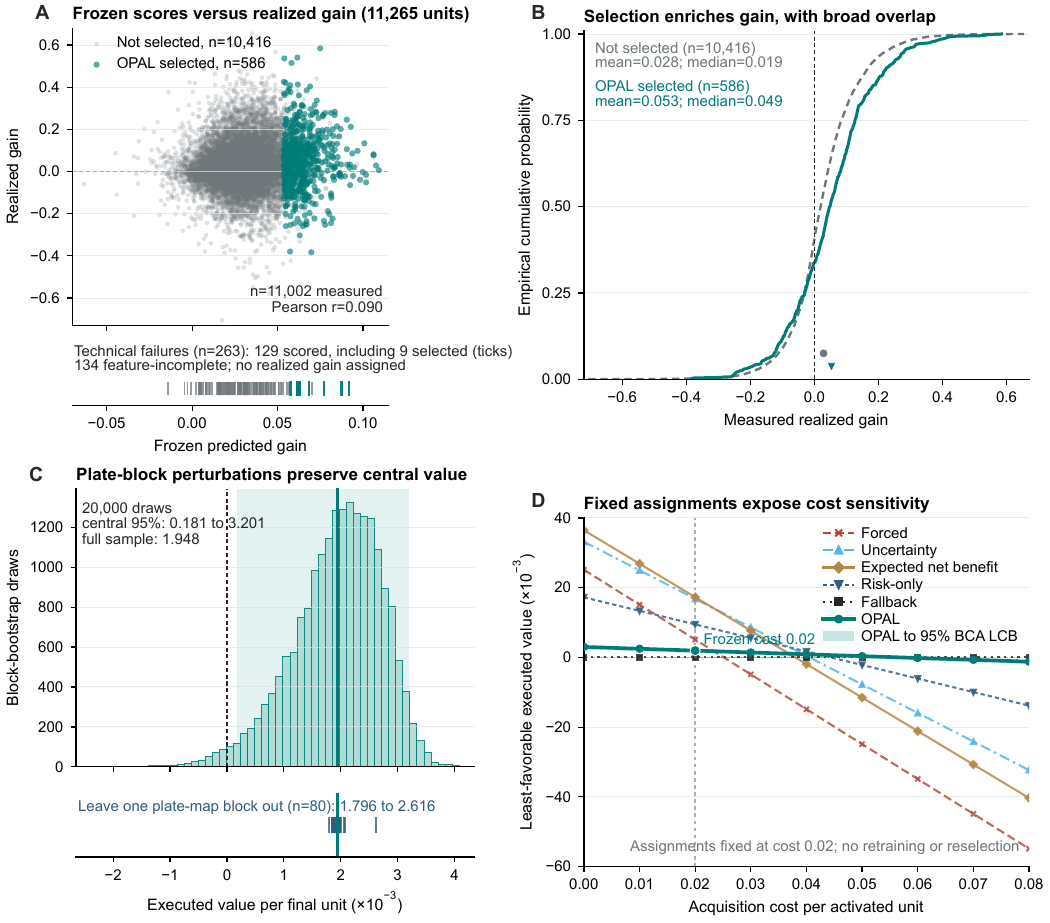}}
\caption{\textbf{Frozen Cell Painting scores, measured gains, dependence and
cost sensitivity.} \textbf{A}, Frozen predicted gain versus realized gain for
all 11,002 final compounds with a measured outcome. Of these, 586 were selected
by \OPAL{} and 10,416 were not. The lower strip accounts separately for 129
compounds with a score but a technical outcome failure and 134
feature-incomplete compounds; the full frozen rule selected 595 of 11,265
compounds. \textbf{B}, Empirical cumulative distributions of measured realized
gain for the same selected and nonselected compounds. The distributions overlap
substantially despite the shift in their centers. \textbf{C}, Distribution of
20,000 four-plate-block bootstrap values under the frozen assignments, with the
central 95\% range and full-sample estimate, and the 80 leave-one-plate-map-block
values shown below. These are post hoc dependence analyses of the observed
layout. \textbf{D}, Least-favorable executed value for all six frozen rules as
the per-activation acquisition cost changes. Assignments and the truth strata
defined at cost 0.02 remain fixed; no rule is retrained or reselected.}
\label{fig:cpg_real_data_validation}
\end{figure}

\begin{figure}[!htbp]
\centering
\suppfigure{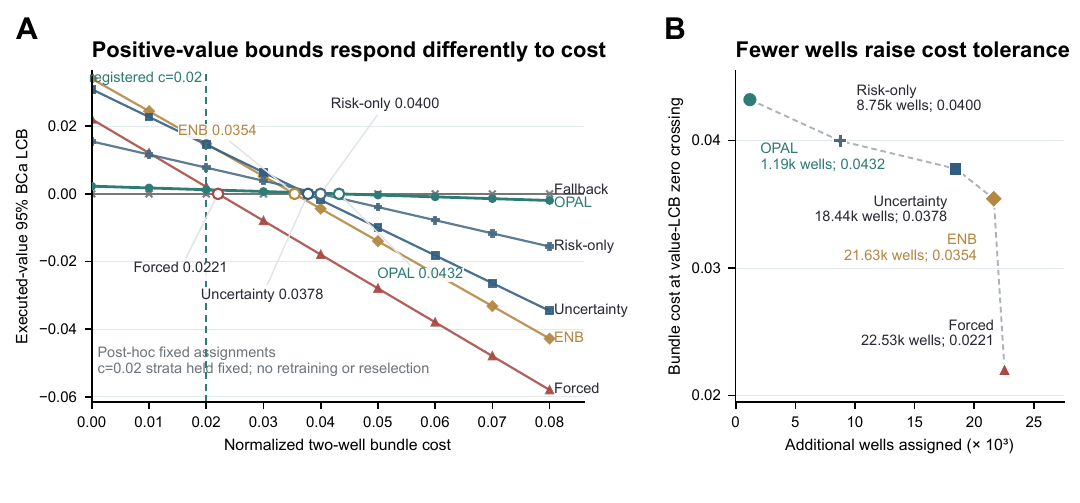}
\caption{\textbf{Fixed-assignment value bounds reveal the cost of activation
volume.} \textbf{A}, Post-hoc cost sensitivity for the six assignments frozen
at the prespecified normalized two-well bundle cost, \(c=0.02\). Assignments and
the prespecified evaluation strata were held fixed; only the common bundle cost
in executed gain changed. Curves show the nominal one-sided 95\% BCa lower
endpoint recomputed with 20,000 compound-level bootstrap resamples at each of
33 costs. The vertical line marks the prespecified cost and open symbols mark
the recomputed lower-bound zero crossings: 0.02206 for forced activation,
0.03543 for expected net benefit, 0.03778 for simple uncertainty, 0.03997 for
the risk-only gate and 0.04320 for \OPAL{}. \textbf{B}, Lower-bound
zero-crossing cost against the number of additional wells assigned by each
active frozen rule. The dashed line guides the eye and is not a fitted
relation; always-fallback is omitted because its value is identically zero and
has no positive-bound crossing. This analysis neither retrains nor reselects a
gate and does not replace the final evaluation at \(c=0.02\).}
\label{fig:cpg_cost_sensitivity}
\end{figure}

\begin{table}[!t]
\centering
\caption{\textbf{Fixed-assignment cpg0012 cost sensitivity.} The zero crossing
is the normalized two-well bundle cost at which the recomputed nominal
one-sided 95\% compound-bootstrap BCa lower endpoint for all-unit executed
value reached zero. Assignments and outcome strata remained fixed at their
prespecified \(c=0.02\) definitions; only net executed value was re-evaluated at
the candidate cost. Crossings were obtained numerically and are descriptive.}
\label{tab:cpg_cost_crossovers}
\small
\renewcommand{\arraystretch}{1.08}
\rowcolors{2}{tablerow}{white}
\begin{tabularx}{\textwidth}{@{}Y C{2.3cm} C{2.3cm} C{3.2cm}@{}}
\specialrule{\heavyrulewidth}{0pt}{0pt}
\rowcolor{tablehead}
Frozen rule & Active compounds & Additional wells & BCa lower-endpoint zero crossing \\
\specialrule{\lightrulewidth}{0pt}{0pt}
Always-fallback & 0 & 0 & Identically zero \\
Forced activation & 11,265 & 22,530 & 0.02206398 \\
ENB gate & 10,815 & 21,630 & 0.03543449 \\
Simple uncertainty gate & 9,221 & 18,442 & 0.03778481 \\
Risk-only gate & 4,374 & 8,748 & 0.03997444 \\
\OPAL{} final gate & 595 & 1,190 & 0.04320122 \\
\specialrule{\heavyrulewidth}{0pt}{0pt}
\end{tabularx}
\end{table}

\FloatBarrier
Under the compound-level analysis, Cell Painting therefore resolved to a
lower-burden sparse candidate rather than a rule ready to replace the fixed
plan. Its 595 activations
required 1,190 additional wells and had positive executed value, but
modest active-set enrichment left the FDP upper bound above its internal
limit. The post-hoc block analysis further showed that a population claim
across assay blocks requires block-disjoint evidence. The next empirical step
is to improve active-set enrichment and test the revised rule in a new
block-disjoint held-out evaluation.

\subsubsection{LINCS--LJP indicates broad temporal-profile acquisition with positive mean value}
\label{sec:sr_lincs}

The LINCS--LJP analysis joined the public GSE70138 Level 5 MODZ L1000
profiles with the published LJP growth-rate measurements
~\cite{subramanian2017l1000,niepel2017lincs}. For each exact
batch--cell--compound--dose episode, \(X\) was the 3-h profile over the 978
directly measured landmark genes, \(Z\) was the matched optional 24-h profile,
and \(M\) was the 72-h growth-rate endpoint. The complete-case frame contained
4,049 episodes from 105 non-overlapping perturbation-identity components, six cell lines
and two LJP batches. Components, rather than episode rows, were assigned to
LEARN, CALIBRATION, CERTIFICATION and FINAL\_LOCK in counts of 42, 21, 21 and
21. The corresponding episode counts were 1,649, 750, 790 and 860.

LEARN, CALIBRATION and CERTIFICATION contributed to the complete final rule and
are therefore development evidence rather than independent confirmation. The
initial FDP-UCB rule (UCB no greater than 0.35) was not met in CALIBRATION and
was relaxed to component-equal FDP point estimate no greater than 0.45 before
CERTIFICATION. At 21 components, the original KL--Hoeffding UCB criterion
required a point FDP no greater than 0.11675; the amended rule therefore made
a substantively weaker point-estimate claim. CERTIFICATION met that FDP rule
but not the original value-LCB-above-zero rule. Value acceptance was then
relaxed to component-equal mean net value above zero. Accordingly,
CERTIFICATION remained development data, and only FINAL\_LOCK evaluated the
complete rule. The support and technical-influence screens remained in force,
and the FDP upper bound and value lower bound remained uncertainty summaries.
The FINAL\_LOCK action map was fixed from 3-h information and permitted
covariates before any FINAL\_LOCK 24-h profile or 72-h outcome was opened.

Among 29 calibration thresholds, the selected threshold was \(-0.06\). It
acquired 706 of 750 calibration episodes across all 21 components, giving a
component-equal activation fraction of 0.94071. The calibration FDP point
estimate was 0.42666 and mean net value was 0.03704. On CERTIFICATION, the
rule acquired all 790 episodes across all 21 components. Its FDP point
estimate was 0.35548 and mean net value was 0.04358. The leave-one-technical-
group-out value estimates ranged from 0.03981 to 0.04825.

The previously unopened FINAL\_LOCK partition reproduced this broad action:
all 860 episodes and all 21 components were activated, with no unevaluable
selected episode. The component-equal FDP point estimate was 0.39917, 5.08
percentage points below the 0.45 operational limit, and the component-equal
mean net value was 0.04295. All 12 batch-by-cell technical groups were
represented. The leave-one-group-out value remained positive throughout,
ranging from 0.03968 to 0.04615, and the largest absolute group-contribution
share was 0.16465. The corresponding FDP upper bound was 0.66111 and the value
lower bound was \(-1.02534\); these wide bounds quantify component-level
uncertainty but did not enter the point-estimate decision rule. In particular,
the Hoeffding value construction uses the complete four-unit prespecified range
\([-2.02,1.98]\) across only 21 equal-weight components. Its one-sided 95\%
penalty is approximately 1.068, which places the lower bound near \(-1.025\)
despite an observed mean near 0.043.

The LINCS--LJP result therefore resolves to broad acquisition within the
observed archive. Positive mean value and stable leave-one-group-out summaries
indicate acquisition of the 24-h profile for all final episodes under the endpoint,
cost, criteria and action map fixed before FINAL. Because activation was
860/860, the result provides no evidence for selective measurement saving.
The wide component-level uncertainty summaries and recurrence of
batch-by-cell groups across partitions limit transport beyond this setting
(Supplementary Table~\ref{tab:lincs_ljp_stage_results}).

\begin{sidewaystable}[p]
\centering
\begin{minipage}{0.93\textheight}
\centering
\begin{threeparttable}
\caption{\textbf{LINCS--LJP calibration, certification and final aggregate
results.} FDP and value are equal-component aggregates. The acceptance rule
used FDP point estimate \(\leq0.45\) and mean net value \(>0\), together with
the support and technical-influence screens.}
\label{tab:lincs_ljp_stage_results}
\fontsize{7.7}{8.7}\selectfont
\renewcommand{\arraystretch}{1.10}
\begin{tabularx}{\linewidth}{@{}L{1.65cm}C{2.2cm}C{2.15cm}C{1.85cm}C{1.5cm}C{1.5cm}C{1.5cm}C{1.5cm}Y@{}}
\toprule
Stage & Selected / total episodes & Activated components & Activation fraction & FDP point & FDP UCB\tnote{a} & Value mean & Value LCB\tnote{a} & LOO value range \\
\midrule
Calibration & 706 / 750 & 21 / 21 & 0.94071 & 0.42666 & 0.68606 & 0.03704 & ---\tnote{b} & ---\tnote{b} \\
Certification & 790 / 790 & 21 / 21 & 1.00000 & 0.35548 & 0.61988 & 0.04358 & \(-1.02470\) & 0.03981--0.04825 \\
FINAL\_LOCK & 860 / 860 & 21 / 21 & 1.00000 & 0.39917 & 0.66111 & 0.04295 & \(-1.02534\) & 0.03968--0.04615 \\
\bottomrule
\end{tabularx}
\begin{tablenotes}[flushleft]\footnotesize
\item[a] One-sided 95\% component-level uncertainty summaries; neither bound was an acceptance criterion for the final operational rule.
\item[b] These calibration-stage summaries were not part of the released aggregate table.
\end{tablenotes}
\end{threeparttable}
\end{minipage}
\end{sidewaystable}

\subsection{Evidence states and experimental consequences}\label{sec:sr_claims}

Supplementary Table~\ref{tab:claims} separates empirical archive states from
methodological evidence. Across the three biological archives, no evaluation
justified a nondegenerate measurement-saving selective policy that met every
study-specific criterion. This shared statement does not collapse the results
into one outcome. Cell Painting produced a lower-burden candidate with
positive executed value, but its binding FDP uncertainty kept the fixed
plan. LINCS--LJP paired positive mean point-estimate value with acquisition of
every final episode, indicating broad acquisition rather than selective
saving. CTRP fell back because its current frozen score did not rank
measured opportunity.

The three states imply different next experiments. Cell Painting calls for
stronger active-set enrichment followed by block-disjoint held-out evaluation.
LINCS--LJP points to budgeting the optional profile broadly within the observed
setting. CTRP requires a revised decision-time score before another selective
policy is evaluated. The formal and simulator results answer a distinct
question: they define identification and finite-sample boundaries and show,
under constructed known-truth conditions, that risk control can retain a
nontrivial active branch. The retrospective and constructibility analyses
likewise establish methodological properties rather than additional
biological archive states.

\begin{sidewaystable}[p]
\centering
\begin{minipage}{0.93\textheight}
\centering
\caption{\textbf{Claim--evidence map across empirical and methodological layers.} Empirical states report the experimental consequence supported in each fixed archive. Formal, simulated and reconstructive analyses characterize framework behavior rather than additional archive states.}
\label{tab:claims}
\fontsize{7.5}{7.9}\selectfont
\renewcommand{\arraystretch}{0.90}
\rowcolors{2}{tablerow}{white}
\begin{tabularx}{\linewidth}{@{}Y L{4.2cm}L{3.5cm}Y@{}}
\specialrule{\heavyrulewidth}{0pt}{0pt}
\rowcolor{tablehead}
Claim & Evidence & Final status & Boundary \\
\specialrule{\lightrulewidth}{0pt}{0pt}
Opportunity is bounded by information & Proposition~\ref{prop:opportunity} and Proposition~\ref{prop:frontier} & Supported analytically & Common utility and payoff-sufficient information channels \\
Probe option value vanishes under a common optimizer & Proposition~\ref{prop:option}; retrospective selector-frontier mechanism diagnostic & Theoretical; empirical signature observed & Finite action set; gross value before cost; retrospective selector analysis did not satisfy exact equality for every bank \\
Selected policy exceeds fixed \(K\) & Retrospective simultaneous executed selector frontier & Not supported & Retrospective development evidence; no universal inferiority claim \\
Cohort-level simultaneous-band safety & Proposition~\ref{prop:safe}; descriptive application to the retrospective selector-frontier analysis band & Supported analytically; fallback illustrated & No learned member had simultaneous LCB \(\geq0\) at \(\epsilon=0\); no active safe improvement was certified \\
Joint risk--value--nontriviality authorization & Proposition~\ref{prop:joint_authorization} & Supported analytically; not established by the locked external result & Risk, value, nontriviality endpoint, minimum and error allocation must all precede outcome opening \\
Source-only joint target authorization under unrestricted conditional outcome shift & Proposition~\ref{prop:target_shift}; domain-adaptation identification literature & Not identifiable without target outcomes or explicit transport restrictions & Distinct from the finite-campaign \(N=16\) boundary; does not imply every source-only policy fails \\
Target-calibrated A--D recovery & Corollary~\ref{cor:target_calibrated_joint}; cpg0012 design-only separation check, development selection, sealed assignments and one-shot final result & FDP upper bound kept the fixed plan in place & Single-held-out final test; FDP 34.62\% point estimate but 37.97\% 95\% UCB versus 35\% internal limit \\
Finite-campaign exact procedure is valid & Propositions~\ref{prop:exact_pair_coverage} and~\ref{prop:finite_safe}; finite-support coverage and non-certifiability certificates & Supported & Known channel grid and frozen value/cost table; pilot cost remains charged under fallback \\
Finite-campaign active qualification is feasible & Propositions~\ref{prop:impossibility} and~\ref{prop:full_likelihood_impossibility}; finite-campaign actionable operating characteristics and feasibility sets & Not supported under the prespecified design & \(N=16\), frozen three-class likelihood, complement-value target and error family \\
NULL false activation of \(S^{\mathrm{eval}}\) is controlled & Proposition~\ref{prop:risk}; prespecified held-out simulator validation & Supported: \(0/400\), UCB 0.01316 & One fresh campaign per distinct fixture; evaluation-only score \\
Sensitivity under constrained risk & Held-out continuous-population estimate and exact interval & 0.684 [0.641, 0.725] & Estimation target without a minimum threshold \\
The opportunity diagnostic is deployable & Observable \(S^{\mathrm{dep}}\), deployment-isomorphic calibration and physical logging & Not tested & \(S^{\mathrm{eval}}\) uses potential outcomes and cannot substitute \\
Archive-backed proxy computation is constructible & Physical-archive constructibility analysis of measured trajectories & Supported descriptively & Post-selection; schema-verified columns but author-constructed probe families and \(K\) levels \\
Archive-backed proxy improves physical outcomes & Causal action logging and physical reward & Not tested & Archived predictive loss is not physical reward; no logged propensity \\
External observable-only NULL activation meets the protocol threshold & Locked pharmacogenomic exact-outcome evaluation & 0/244, UCB 0.02148 & Assignments were identical to always-fallback; the repeated-sampling bound assumes independent Bernoulli NULL-family indicators with a common activation probability \\
External observable-only policy recovers measured opportunity & locked pharmacogenomic exact-outcome evaluation paired-value component and POSITIVE sensitivity & Not supported & Paired lower bound \(=0\); 0/7 POSITIVE families activated \\
Target-calibrated policy is non-trivial and controls false activation & cpg0012 sealed final assignments and one-shot final outcomes & Supported within archive: 595/11,265 active; false-activation UCB 5.18\% \(\leq7.5\%\) & Author-controlled procedural separation; adaptive-replicate abstraction \\
Target-calibrated policy has positive executed value and minimum sensitivity & cpg0012 all-unit worst-case value and POSITIVE-stratum bounds & Prespecified model: value LCB \(1.229\times10^{-3}>0\); sensitivity LCB 5.41\% \(\geq5\%\) & Block sensitivity LCB 4.86\%; value positive by percentile but not BCa; new-block claims need block-disjoint evidence \\
Target-calibrated policy meets complete joint certificate & cpg0012 one-shot final joint result & Not supported & FDP point estimate 34.62\% was below 35\%, but its binding 95\% UCB was 37.97\% \\
The optional 24-h profile has positive component-equal mean point-estimate value in LINCS--LJP & Component-disjoint FINAL\_LOCK aggregate & Met under the operational point rule: FDP 39.92\% \(\leq45\%\), mean value 0.04295 \(>0\) & 860/860 episodes selected; indicates broad acquisition, not selectivity or measurement saving; UCB 66.11\% and LCB \(-1.02534\) are uncertainty summaries \\
Frozen observable score exhausts decision-time information & Alternative observable representations or a sharp \(G_X\) bound & Not established & The data show policy--opportunity misalignment, not information-theoretic absence \\
Lifecycle adoption is feasible & Positive operational band and finite \(H^\star\) & Not established in the selector-frontier or finite-campaign evidence layers & Operational contrasts were non-positive; normalized synthetic costs \\
Prospective physical-facility generalization and adoption & Online external evaluation with documented actions and calibrated costs & Not tested & Offline archives and post-selection constructibility do not establish prospective facility transport \\
\specialrule{\heavyrulewidth}{0pt}{0pt}
\end{tabularx}
\end{minipage}
\end{sidewaystable}

\clearpage
\section{SUPPLEMENTARY METHODS}\label{sec:supp_methods}
\setcounter{section}{2}
\setcounter{subsection}{0}
\setcounter{subsubsection}{0}

\subsection{Decision theory and deployment contract}\label{sec:si_method_theory}

\subsubsection{Notation and evidence roles}\label{sec:si_overview}

This Supplementary Information provides the assumptions, complete proofs,
simulator specifications, inference algorithms, cost accounting and provenance
for the framework described in Main Methods. Table~\ref{tab:notation}
summarizes the notation and separates deployable, evaluation-only and oracle
evidence roles used below.

\begin{table}[t]
\centering
\caption{\textbf{Core notation and evidence roles.}}
\label{tab:notation}
\small
\renewcommand{\arraystretch}{1.08}
\rowcolors{2}{tablerow}{white}
\begin{tabularx}{\textwidth}{@{}L{3.0cm}Y L{3.45cm}@{}}
\specialrule{\heavyrulewidth}{0pt}{0pt}
\rowcolor{tablehead}
Symbol & Definition & Evidence role \\
\specialrule{\lightrulewidth}{0pt}{0pt}
\(M\) & Payoff-relevant latent campaign state & Oracle truth \\
\(X\) & Information available before commitment & Deployable survey/context \\
\(Z_r\) & Optional probe result & Deployable after purchase; \(r\neq q\) \\
\(\mathcal K\) & Finite capacity action set; primary \(\{0,1,2\}\) & Common action set \\
\(K_{\mathrm{score},-f},K_0\) & Fold score reference; executed fixed comparator and fallback & Record both identities explicitly \\
\(V(\mathcal I)\) & Optimal value under information \(\mathcal I\) & Information-indexed value \\
\(G_M,G_X,G_{XZ}\) & Full-state, precommitment and probe-augmented opportunity & Information hierarchy \\
\(\Omega_r\) & Value of delaying commitment for probe \(r\) & Gross and net forms \\
\(\pi_\lambda\) & Selector at threshold \(\lambda\) & Frozen frontier member \\
\(\Delta_{\mathrm{camp}},G_{A,D}(q)\) & Net campaign value; pre-bill gain on the deployment complement & Finite-campaign estimands \\
\(R_0(\tau)\) & NULL activation probability & Binding constrained risk \\
\(J,E,A,\Gamma\) & Target truth stratum, evaluability, activation and incremental net gain & Target-calibrated design only \\
\(R_{\mathrm{wc}},F_{\mathrm{wc}},S_{\mathrm{wc}},\mu_{\mathrm{wc}}\) & Endpoint-specific worst-case false activation, active-set failure, POSITIVE sensitivity and all-unit executed value & Joint target certificate \\
\(S^{\mathrm{dep}},S^{\mathrm{eval}}\) & Deployable and evaluation-only scores & Roles cannot be interchanged \\
\(H,H^\star\) & Reuse before retraining; break-even reuse & Lifecycle denominator and boundary \\
\specialrule{\heavyrulewidth}{0pt}{0pt}
\end{tabularx}
\end{table}

\subsubsection{Decision contract and deployment isomorphism}\label{sec:deployment_isomorphism}

\paragraph{Physical interpretation of capacity}\label{sec:capacity_interpretation}

Capacity \(K\) denotes a resource-dependent ability to alter an experimental trajectory after initial information becomes available. Examples include reserving a calibration-dependent branch, an operator-supported intervention, a switchable instrument mode or a software-mediated decision slot. The theory requires only a finite action set and a common utility scale; capacity need not be an explicitly named facility resource.

Four commitment modes motivate the abstraction. In a fully reversible mode, capacity is charged only when used and can be chosen at each step. In soft and strong commitment modes, capacity is reserved after a survey or before probes and unused reservations receive partial or zero refunds. In hard commitment, capacity is fixed before the campaign and spot expansion is unavailable. Commitment friction changes reservation, idle, cancellation and spot terms but never multiplies scientific success.

\paragraph{Deployment-isomorphism invariant}\label{sec:deployment_isomorphism_rule}

For every deployed quantity \(T\), let
\[
\begin{aligned}
\mathcal I_{\mathrm{dep}}(T)
=\sigma(&\text{frozen identifiers and prespecified randomization,}\\
        &\text{observations available by the decision time}).
\end{aligned}
\]
The invariant requires:

\begin{enumerate}
\item \(T\) is measurable with respect to \(\mathcal I_{\mathrm{dep}}(T)\);
\item calibration recomputes \(T\) using the identical function and identical number of observations;
\item no latent class, future response, evaluation outcome or oracle action enters a quantity claimed to be deployable;
\item pilot selection is permutation-invariant with respect to storage order and uses a committed hash or prespecified random sample;
\item the independent unit and estimand used for uncertainty match the unit and population acted on at deployment.
\end{enumerate}

The final condition extends data-flow requirements to the estimand itself. A
statistic calibrated for a new-bank superpopulation does not automatically
measure value on the unobserved complement of the current finite campaign. A
simulator may use potential outcomes to construct a truth or evaluation score,
but that score must be labeled \(S^{\mathrm{eval}}\), kept outside the
deployable feature set and excluded from claims of real-time observability.

\subsubsection{Theoretical guarantees}\label{sec:theory}

\paragraph{Information opportunity}\label{sec:opportunity_proof}

Let \((\Omega,\mathcal F,\mathbb P)\) be a probability space and let \(\mathcal K\) be finite. Utilities \(U(k)\) are integrable and already include action-specific operational costs. For a sigma-field \(\mathcal I\), define
\[
\mu_k(\mathcal I)=\mathbb E\{U(k)\mid\mathcal I\},
\qquad
V(\mathcal I)=\mathbb E\max_{k\in\mathcal K}\mu_k(\mathcal I).
\]
Let \(\mathcal I_0\) be the trivial information set, so
\(
V(\mathcal I_0)=\max_k\mathbb E U(k)=V_0.
\)

\begin{lemma}[Information refinement]\label{lem:refinement}
If \(\mathcal I_1\subseteq\mathcal I_2\), then
\[
V(\mathcal I_1)\leq V(\mathcal I_2).
\]
\end{lemma}

\begin{proof}
For any \(\mathcal I_1\)-measurable action \(a_1\), the same action is feasible under \(\mathcal I_2\). Taking the supremum over the larger policy class proves the result. Equivalently,
\[
\max_k\mathbb E\{\mu_k(\mathcal I_2)\mid\mathcal I_1\}
\leq
\mathbb E\{\max_k\mu_k(\mathcal I_2)\mid\mathcal I_1\}
\]
by conditional Jensen for the convex maximum; expectation gives the claim. This is the finite-action decision form of information monotonicity~\cite{blackwell1953}.
\end{proof}

\begin{lemma}[Payoff-sufficient garbling]\label{lem:garbling}
Let \(M\) be a payoff-relevant state and let \(A\) be an information channel, not necessarily a sub-\(\sigma\)-field of \(\sigma(M)\). Suppose, for every \(k\in\mathcal K\),
\[
\mathbb E\{U(k)\mid M,A\}
=
\mathbb E\{U(k)\mid M\}.
\]
Then \(V\{\sigma(A)\}\leq V\{\sigma(M)\}\).
\end{lemma}

\begin{proof}
Write \(m_k(M)=\mathbb E\{U(k)\mid M\}\). Conditional sufficiency and the tower property give
\[
\mathbb E\{U(k)\mid A\}
=
\mathbb E\{m_k(M)\mid A\}.
\]
Therefore, by conditional Jensen,
\[
\begin{aligned}
V\{\sigma(A)\}
&=
\mathbb E\max_k\mathbb E\{m_k(M)\mid A\}\\
&\leq
\mathbb E\,\mathbb E\{\max_k m_k(M)\mid A\}\\
&=
\mathbb E\max_k m_k(M)
=V\{\sigma(M)\}.
\end{aligned}
\]
This is the finite-action consequence of \(A\) being Blackwell-dominated by the payoff-relevant state~\cite{blackwell1953}.
\end{proof}

\begin{proposition}[Opportunity monotonicity and the common-optimizer condition]\label{prop:opportunity}
Suppose
\begin{equation}\label{eq:si_payoff_sufficiency}
\mathbb E\{U(k)\mid M,X,Z_r\}
=
\mathbb E\{U(k)\mid M\},
\qquad k\in\mathcal K.
\end{equation}
No inclusion \(\sigma(X)\subseteq\sigma(M)\) is required. Then
\[
0\leq G_X\leq G_{XZ}\leq G_M,
\]
where
\[
\begin{aligned}
G_X&=V\{\sigma(X)\}-V_0,\\
G_{XZ}&=V\{\sigma(X,Z_r)\}-V_0,\\
G_M&=V\{\sigma(M)\}-V_0.
\end{aligned}
\]
Also, \(G_M=0\) if and only if there exists \(k_0\in\mathcal K\) such that
\[
m_{k_0}(M)
=
\max_{k\in\mathcal K}m_k(M)
\quad\text{almost surely},
\qquad
m_k(M)=\mathbb E\{U(k)\mid M\}.
\]
\end{proposition}

\begin{proof}
Non-negativity follows because every informed policy can ignore its information. Since
\(
\sigma(X)\subseteq\sigma(X,Z_r),
\)
Lemma~\ref{lem:refinement} gives \(G_X\leq G_{XZ}\). Equation~\eqref{eq:si_payoff_sufficiency} and Lemma~\ref{lem:garbling}, with \(A=(X,Z_r)\), give \(G_{XZ}\leq G_M\). The result therefore accommodates independent noise in \(X\) or \(Z_r\) without requiring either variable to be a sub-\(\sigma\)-field of \(M\).

For the equality condition, choose
\(
k_0\in\arg\max_k\mathbb E U(k).
\)
The random variable
\[
D
=
\max_k m_k(M)-m_{k_0}(M)
\]
is non-negative and has expectation
\[
\mathbb E D
=V\{\sigma(M)\}-\mathbb E U(k_0)
=G_M.
\]
If \(G_M=0\), then \(D=0\) almost surely, giving the common optimizer. Conversely, if a common optimizer exists, its expected utility equals both the informed optimum and the best fixed value, so \(G_M=0\).
\end{proof}

\begin{corollary}[No architecture can create absent opportunity]\label{cor:no_architecture}
If \(G_M=0\), every deployable policy \(\pi(A)\) based on a channel satisfying
\[
\mathbb E\{U(k)\mid M,A\}
=
\mathbb E\{U(k)\mid M\}
\quad\text{for all }k
\]
satisfies
\[
V(\pi)\leq V_0.
\]
After non-negative model, probe or switching cost is added, the inequality
remains valid and is strict whenever the expected incremental
policy-specific cost is strictly positive.
\end{corollary}

\begin{proof}
Lemma~\ref{lem:garbling} gives
\(
V\{\sigma(A)\}\leq V\{\sigma(M)\}=V_0.
\)
The policy \(\pi(A)\) cannot exceed the optimal \(A\)-measurable value. Subtracting non-negative cost cannot increase value.
\end{proof}

\paragraph{Option value of delayed commitment}\label{sec:option_proof}

For \(P_X\)-almost every \(x\) at which the regular conditional expectations
below are defined, fix \(X=x\) and write
\[
\mu_k(z)=\mathbb E\{U(k)\mid X=x,Z=z\},
\qquad
\bar\mu_k=\mathbb E\{\mu_k(Z)\mid X=x\}.
\]

\begin{proposition}[Gross option value and equality]\label{prop:option}
For finite \(\mathcal K\),
\[
\Omega^{\mathrm{gross}}(x)
=
\mathbb E\{\max_k\mu_k(Z)\mid X=x\}
-\max_k\bar\mu_k
\geq0.
\]
Equality holds if and only if there exists \(k_x\in\mathcal K\) such that
\[
\mu_{k_x}(Z)=\max_k\mu_k(Z)
\quad\text{almost surely conditional on }X=x.
\]
\end{proposition}

\begin{proof}
Let \(k_x\in\arg\max_k\bar\mu_k\). Then
\[
\Omega^{\mathrm{gross}}(x)
=
\mathbb E[
\max_k\mu_k(Z)-\mu_{k_x}(Z)
\mid X=x]\geq0.
\]
If the expectation is zero, its non-negative integrand vanishes almost surely, so \(k_x\) is a common optimizer. The converse follows by substituting a common optimizer.
\end{proof}

\begin{corollary}[Net option value]\label{cor:net_option}
For cost \(c_r+c_r^{\mathrm{lat}}\geq0\), an individual net probe value
\(
\Omega_r^{\mathrm{gross}}-c_r-c_r^{\mathrm{lat}}
\)
may be negative. The optimized value
\[
\Omega^\star
=
\max\{0,\max_r(\Omega_r^{\mathrm{gross}}-c_r-c_r^{\mathrm{lat}})\}
\]
is non-negative solely because the no-probe action is retained.
\end{corollary}

\paragraph{Frontier upper bound}\label{sec:frontier_proof}

\begin{proposition}[Deployable frontier ceiling]\label{prop:frontier}
Let \(\Pi_\Lambda=\{\pi_\lambda:\lambda\in\Lambda\}\) be a family of policies measurable with respect to \(\mathcal I_{XZ}\). Here \(V(\pi_\lambda)\) denotes expected executed utility after subtracting any policy-specific information cost, whereas \(V(\mathcal I_{XZ})\) is the gross optimal decision value under that information. For an executed fixed comparator \(K_0\), define its population fixed-action gap
\[
b_0=V_0-V(K_0)\geq0,
\qquad
V_0=\max_{k\in\mathcal K}V(k).
\]
If policy-specific information costs are non-negative, then
\[
\sup_{\lambda\in\Lambda}
\{V(\pi_\lambda)-V(K_0)\}
\leq G_{XZ}+b_0
\leq G_M+b_0.
\]
If \(K_0\) is the population-optimal fixed action on the same population,
utility scale and cost contract, then \(b_0=0\).
\end{proposition}

\begin{proof}
\(V(\mathcal I_{XZ})\) is the supremum over all
\(\mathcal I_{XZ}\)-measurable actions. Hence
\[
V(\pi_\lambda)-V(K_0)
\leq
V(\mathcal I_{XZ})-V_0+\{V_0-V(K_0)\}
=G_{XZ}+b_0.
\]
Subtracting a non-negative policy-specific information cost cannot increase
the left-hand side. Proposition~\ref{prop:opportunity} gives the second
inequality. When \(V(K_0)=V_0\), the baseline-gap term vanishes.
\end{proof}

\begin{remark}
The proposition gives a ceiling rather than an equality. A contrast against a
suboptimal fixed comparator may exceed \(G_{XZ}\) by at most \(b_0\), but this
increment comes from the baseline rather than routing. A flat empirical
frontier may reflect low \(G_{XZ}\), a restrictive selector architecture,
imprecise value estimation or cost.
\end{remark}

\paragraph{Cohort-level safe policy selection}\label{sec:safe_proof}

\begin{proposition}[Simultaneous-band safety under data-dependent selection]\label{prop:safe}
Let \(\Pi\) be a finite, prespecified policy class containing \(K_0\). Define
\[
\Delta(\pi)=V(\pi)-V(K_0).
\]
Suppose random lower bounds \(L(\pi)\) satisfy
\[
\Pr(\mathcal E)\geq1-\alpha,
\qquad
\mathcal E=
\{L(\pi)\leq\Delta(\pi)\text{ for every }\pi\in\Pi\}.
\]
For \(\epsilon\geq0\), let the eligible set be
\[
\mathcal A=\{\pi\in\Pi:L(\pi)\geq-\epsilon\}.
\]
If \(\widehat\pi\) is any data-dependent member of \(\mathcal A\), with \(K_0\) returned when \(\mathcal A\) is empty, then
\[
\Pr\{\Delta(\widehat\pi)\geq-\epsilon\}\geq1-\alpha.
\]
\end{proposition}

\begin{proof}
On \(\mathcal E\), every \(\pi\in\mathcal A\) satisfies
\(
\Delta(\pi)\geq L(\pi)\geq-\epsilon.
\)
Thus any selection rule within \(\mathcal A\) is safe on the same event. If the set is empty, the fallback has \(\Delta(K_0)=0\). Therefore failure of the conclusion is contained in \(\mathcal E^c\), whose probability is at most \(\alpha\).
\end{proof}

The implementation decomposes
\begin{equation}\label{eq:epsilon_registry}
\epsilon_{\mathrm{registry}}
=
r_{\max}\delta_{\mathrm{harm}}(\lambda,\alpha)
+C_{\mathrm{decision}}+C_{\mathrm{probe}},
\end{equation}
where \(r_{\max}\) is a prespecified cohort switch-rate cap and \(\delta_{\mathrm{harm}}\) is a simultaneous per-deviation harm bound. A realized switch rate may yield a tighter descriptive bound but does not alter \(\epsilon_{\mathrm{registry}}\). Independent bankwise intervals are insufficient because data-dependent selection can accumulate familywise error.

\begin{proposition}[Conjunctive risk--value--nontriviality authorization]
\label{prop:joint_authorization}
Conditional on all development and calibration data, let the active action,
gate, threshold and action-cost contract be fixed before an independent
evaluation sample is opened. Let
\[
R_0=\Pr\{A(X)=1\mid\mathrm{NULL}\}
\]
be false-activation risk. Let \(\pi_A\) execute the active action when \(A=1\)
and \(K_0\) otherwise, and let
\[
\mu=\mathbb E\{U(\pi_A(X))-U(K_0)\}
\]
be executed value, and let \(T\) be one prespecified nontriviality endpoint,
such as POSITIVE sensitivity or activation coverage. Define the marginal pass
events
\[
P_R=\{U_R\leq r_0\},\qquad
P_\mu=\{L_\mu>0\},\qquad
P_T=\{L_T\geq t_{\min}\},
\]
and suppose that each is a level-\(\alpha_j\) test of its complementary
component null:
\[
\sup_{R_0>r_0}\Pr(P_R)\leq\alpha_R,\quad
\sup_{\mu\leq0}\Pr(P_\mu)\leq\alpha_\mu,\quad
\sup_{T<t_{\min}}\Pr(P_T)\leq\alpha_T.
\]
If \(t_{\min}>0\), the identity of \(T\), and every pass rule were frozen
before evaluation, and authorization requires \(P_R\cap P_\mu\cap P_T\), then
\[
\sup_{\{R_0>r_0\}\cup\{\mu\leq0\}\cup\{T<t_{\min}\}}
\Pr(\mathrm{authorize})
\leq\max(\alpha_R,\alpha_\mu,\alpha_T).
\]
No independence among component statistics is required. Separately, if
\(U_R,L_\mu,L_T\) have marginal coverage \(1-\alpha_R\),
\(1-\alpha_\mu\) and \(1-\alpha_T\), the probability that all three bounds
cover simultaneously is at least
\(1-\alpha_R-\alpha_\mu-\alpha_T\).
\end{proposition}

\begin{proof}
If the conjunction is false, at least one component null is true. For any
such component \(j\), authorization is a subset of \(P_j\), so its
probability is at most \(\alpha_j\), and hence at most the largest component
level. This is the intersection--union test argument. The simultaneous
coverage statement follows separately by the union bound.
\end{proof}

\begin{remark}
The nontriviality endpoint cannot be selected after evaluation. If several
endpoints are admissible, either one must be designated in advance or their
selection must be incorporated into a valid component test. Strict
\(L_\mu>0\) already excludes
always-\(K_0\), whose paired gains are identically zero; \(t_{\min}>0\)
additionally excludes policies that activate only at a negligible rate.
\end{remark}

\paragraph{Source-only target-outcome non-identifiability and
target-calibrated recovery}\label{sec:target_calibrated_recovery}

For this paragraph, let \(A(X)\in\{0,1\}\) be a frozen gate and let
\(E\in\{0,1\}\) indicate whether the target outcome is evaluable. When
\(E=1\), let
\(J\in\{\mathrm{NULL},\mathrm{POSITIVE},\mathrm{AMBIGUOUS}\}\) be the
evaluation-only stratum and let
\(\Gamma=U(a_1)-U(a_0)\) be active-minus-fallback gain after all prespecified
action, information and delay costs. Suppose
\(\Gamma\in[g_{\min},g_{\max}]\). The gated policy's incremental value is
\(\mathbb E(A\Gamma)\); \(\mu_{\mathrm{wc}}\) below is its worst-case version. Define
\[
\widetilde\Gamma=
\begin{cases}
\Gamma,&E=1,\\
g_{\min},&E=0.
\end{cases}
\]
Assume the conditional denominators below are positive. The
endpoint-specific worst-case quantities are
\[
\begin{aligned}
R_{\mathrm{wc}}
&=
\frac{
 \Pr(A=1,E=1,J=\mathrm{NULL})+\Pr(A=1,E=0)}
{\Pr(E=1,J=\mathrm{NULL})+\Pr(A=1,E=0)},\\
F_{\mathrm{wc}}
&=
\Pr\{E=0\ \mathrm{or}\ (E=1,J=\mathrm{NULL})\mid A=1\},\\
S_{\mathrm{wc}}
&=
\frac{\Pr(A=1,E=1,J=\mathrm{POSITIVE})}
{\Pr(E=1,J=\mathrm{POSITIVE})+\Pr(A=0,E=0)},\\
\mu_{\mathrm{wc}}
&=\mathbb E\{A\widetilde\Gamma\}.
\end{aligned}
\]
If \(\Pr(A=1)=0\), \(F_{\mathrm{wc}}\) is undefined and the authorization
software assigns the conservative gate value one. Measured AMBIGUOUS units
are reported separately and are neither NULL nor POSITIVE. For every
completion of the missing labels and gains, false-activation risk and FDP are
no larger than \(R_{\mathrm{wc}}\) and \(F_{\mathrm{wc}}\), whereas
sensitivity and mean executed gain are no smaller than
\(S_{\mathrm{wc}}\) and \(\mu_{\mathrm{wc}}\). The four least-favorable
completions need not be the same; each inequality holds endpoint-wise and can
therefore enter the conjunctive rule. The target joint claim is
\begin{equation}\label{eq:target_joint_claim}
R_{\mathrm{wc}}\leq r_0,\qquad F_{\mathrm{wc}}\leq f_0,\qquad
S_{\mathrm{wc}}\geq s_0,\qquad
\mu_{\mathrm{wc}}>\delta_\mu,\qquad \delta_\mu\geq0.
\end{equation}

The following proposition parallels domain-adaptation impossibility results
showing that source labels and unlabeled target observations do not guarantee
target performance without assumptions linking the domains
~\cite{bendavid2010impossibility}. Our target is different: it concerns the
identifiability of a conjunctive authorization claim on the subset selected
for an additional measurement, rather than a general upper bound on target
prediction error. The proof is self-contained and leads directly to the
target-calibrated recovery below.

\begin{proposition}[Source-only target-outcome non-identifiability]
\label{prop:target_shift}
Let \(\mathcal F_{\mathrm{pre}}\) contain every source-domain outcome,
unlabeled target covariate, schema object, frozen model, threshold and random
seed available before target outcomes are opened. Suppose no restriction is
placed on the target conditional law of \(E\) and, when \(E=1\),
\((J,\Gamma)\) given \(X\). Assume \(0\leq r_0<1\). Fix a target
covariate law and an \(\mathcal F_{\mathrm{pre}}\)-measurable gate for which
at least one outcome extension satisfying Eq.~\eqref{eq:target_joint_claim}
exists. Then an adverse extension exists with the same distribution on
\(\mathcal F_{\mathrm{pre}}\), and therefore the same assignments, for which
Eq.~\eqref{eq:target_joint_claim} is false.

Consequently, if an \(\mathcal F_{\mathrm{pre}}\)-measurable certificate
\(D\) obeys
\[
\sup_{Q:\,\text{Eq.~\eqref{eq:target_joint_claim} false}}
\Pr_Q(D=1)\leq\alpha,
\]
then every favorable \(Q^+\) pre-outcome-indistinguishable from an adverse
\(Q^-\) satisfies
\[
\Pr_{Q^+}(D=1)=\Pr_{Q^-}(D=1)\leq\alpha.
\]
\end{proposition}

\begin{proof}
Keep the target covariate law and every
\(\mathcal F_{\mathrm{pre}}\)-measurable object fixed. Favorable value implies
\(a=\Pr(A=1)>0\). If \(a<1\), set \(E=0\) for every target unit. If
\(a=1\), positivity of the prespecified sensitivity denominator under the
favorable extension guarantees an admissible evaluable POSITIVE outcome; set
\(E=0\) with any probability in \((0,1)\) and use that outcome otherwise. All
conditional denominators remain positive and \(R_{\mathrm{wc}}=1>r_0\). This
adverse law has the same pre-outcome distribution and assignments as the
favorable law; measurability of \(D\) and uniform control give the bound.
\end{proof}

\begin{remark}
The existence qualifier matters. Always-\(K_0\) has no favorable extension
with strictly positive executed value and non-trivial activation. The result
does not say that a source-only gate cannot perform well under a particular
target law; it says that its joint target certificate is not identified
uniformly over unrestricted conditional outcome shifts.
\end{remark}

\begin{corollary}[Uniform deterministic target-risk boundary]
\label{cor:uniform_target_risk}
Suppose the target family may concentrate on any covariate point and make its
outcome evaluable and NULL. If a deterministic gate activates at any \(x\), a
target law concentrated on \(X=x,E=1,J=\mathrm{NULL}\) has
\(R_{\mathrm{wc}}=1\). Thus, for \(r_0<1\),
the only deterministic gate that uniformly controls target false activation
over this unrestricted family is always-\(K_0\). For a randomized gate with
activation probability \(a(x)\), risk-only uniform control requires
\(\sup_x a(x)\leq r_0\); an adverse gain extension still prevents a uniform
strictly positive value guarantee.
\end{corollary}

The qualifiers \emph{uniform}, \emph{deterministic} and
\emph{unrestricted} are essential. Non-trivial recovery requires either a
defensible transport restriction or target labels. The cpg0012 study used
target labels while retaining a final partition untouched by model and
threshold development.

\begin{corollary}[Target-calibrated specialization]
\label{cor:target_calibrated_joint}
Use target tuning outcomes to select one pointwise gate, and hold that gate
fixed before opening an independent target certification sample. Conditional
on selection, suppose the retained certification units are independent at
the prespecified inferential-cluster level and arise from the declared target
sampling frame. Let \(U_R,U_F,L_S,L_\mu\) be valid one-sided bounds for
\(R_{\mathrm{wc}},F_{\mathrm{wc}},S_{\mathrm{wc}},\mu_{\mathrm{wc}}\),
with marginal error probabilities
\(\alpha_R,\alpha_F,\alpha_S,\alpha_\mu\). Authorize only if
\[
U_R\leq r_0,\qquad U_F\leq f_0,\qquad
L_S\geq s_0,\qquad L_\mu>\delta_\mu.
\]
Then
\[
\sup_{\substack{
R_{\mathrm{wc}}>r_0\ \mathrm{or}\ F_{\mathrm{wc}}>f_0\\
\mathrm{or}\ S_{\mathrm{wc}}<s_0\ \mathrm{or}\
\mu_{\mathrm{wc}}\leq\delta_\mu}}
\Pr(\mathrm{authorize})
\leq
\max(\alpha_R,\alpha_F,\alpha_S,\alpha_\mu).
\]
No independence among the four endpoint statistics is required. Separately,
the probability that all four marginal bounds cover simultaneously is at
least \(1-\alpha_R-\alpha_F-\alpha_S-\alpha_\mu\).
\end{corollary}

\begin{proof}
After target tuning, the selected rule is fixed. Each marginal interval has
its prespecified component-test level on the independent target sample. If any
component claim is false, authorization is a subset of the pass event for
that true component null. The intersection--union argument therefore gives
the maximum component level. The separate simultaneous-coverage statement
follows from the union bound. This is the four-endpoint specialization of
Proposition~\ref{prop:joint_authorization}.
\end{proof}

One candidate error allocation is
\[
(\alpha_R,\alpha_F,\alpha_S,\alpha_\mu)
=(5\times10^{-3},10^{-2},10^{-2},2.5\times10^{-2}),
\]
which gives a false-authorization bound of 0.025 by the
intersection--union rule and simultaneous marginal coverage of at least 0.95
by the union bound. This is a design option rather than the allocation used in
cpg0012. That study evaluated each of its four component bounds at one-sided
95\%, giving a nominal 0.05 intersection--union decision under the stated iid
sampling and BCa assumptions, but neither a distribution-free finite-sample
guarantee nor a simultaneous 95\% confidence set. The binomial components use one-sided
Clopper--Pearson bounds~\cite{clopper1934}. For a gain range
\([g_{\min},g_{\max}]\), define the intention-to-authorize value vector
\[
\widetilde\Gamma_i=
\begin{cases}
\Gamma_i,&A_i=1,\ E_i=1,\\
g_{\min},&A_i=1,\ E_i=0,\\
0,&A_i=0.
\end{cases}
\]
The value certificate is applied directly to the mean of all
\(\widetilde\Gamma_i\), not to an active-subset mean multiplied by observed
coverage. For \(n\geq2\) independent and identically distributed bounded
inferential clusters, sample variance \(\widehat V\), and a rule fixed
before these units are opened, the primary empirical-Bernstein lower bound is
\begin{equation}\label{eq:target_empirical_bernstein}
L_\mu=
\overline{\widetilde\Gamma}
-\sqrt{\frac{2\widehat V\log(2/\alpha_\mu)}{n}}
-\frac{7(g_{\max}-g_{\min})\log(2/\alpha_\mu)}
       {3(n-1)}
\end{equation}
~\cite{maurer2009empirical}. The prespecified Hoeffding cross-check is
\[
\overline{\widetilde\Gamma}
-(g_{\max}-g_{\min})
\sqrt{\frac{\log(1/\alpha_\mu)}{2n}}
\]
~\cite{hoeffding1963}. These are repeated-sampling statements under the iid
inferential-cluster sampling contract. For a completely observed fixed archive,
the observed mean is an archive property; neither formula adds uncertainty to
that quantity. The complete target-calibrated sequence and its realized
cpg0012 status are summarized in Table~\ref{tab:target_calibrated_stages}.

\begin{table}[!t]
\centering
\caption{\textbf{Target-calibrated authorization chain and realized cpg0012
status.} The realized design treats both labeled pre-final partitions as
development and leaves final outcomes unopened until one rule and all
assignments are sealed.}
\label{tab:target_calibrated_stages}
\small
\renewcommand{\arraystretch}{1.08}
\rowcolors{2}{tablerow}{white}
\begin{tabularx}{\textwidth}{@{}L{1.25cm}L{3.0cm}Y L{3.8cm}@{}}
\specialrule{\heavyrulewidth}{0pt}{0pt}
\rowcolor{tablehead}
Stage & Information released & Binding operation & Realized cpg0012 status \\
\specialrule{\lightrulewidth}{0pt}{0pt}
A & Outcome-free units, actions, planned-assay manifest and decision-time features & Define non-overlapping inferential units, common actions, bounded utility and salted target split & Before profile access: 25,013 clusters split 2,444/11,304/11,265 \\
B & Labeled target development outcomes & Select model and one final rule without final outcomes & First two labeled partitions formed 13,748-unit development set; final remained sealed \\
C & Final decision-time observables only & Apply the frozen rule, seal assignments, verify the final activation-count and coverage guards, and meet the pre-outcome base-feature-missingness check & 595/11,265 active; 5.2818\% coverage; 1.1895\% base missingness. Activation count and coverage are final finite-frame guards; base-feature missingness is a separate readiness check \\
D & Final additional and verifier profiles & Evaluate the complete joint certificate once under worst-case intention-to-authorize missingness & FDP UCB 37.9660\% exceeded the 35\% limit, retaining the fixed plan \\
\specialrule{\heavyrulewidth}{0pt}{0pt}
\end{tabularx}
\end{table}
\FloatBarrier

Stage-C base-feature missingness is a pre-outcome readiness check rather than one
of the seven final components; Stage-D outcome missingness is the seventh final
component. Their denominators and roles differ: 1.1895\% describes unavailable
decision-time feature rows before outcome opening, whereas 2.3347\% describes
missing or unevaluable final outcomes. Final activation count, coverage and
outcome missingness are finite-frame operational guards, not additional
population claims with unallocated error.
Activated missing rows already enter
\(R_{\mathrm{wc}},F_{\mathrm{wc}},S_{\mathrm{wc}}\) and
\(\mu_{\mathrm{wc}}\) through the prespecified worst-case recoding. If population activation
coverage is claimed in a future study, its lower bound requires a separate
error allocation.

\begin{center}
\begin{minipage}{\textwidth}
\centering
\captionof{table}{\textbf{Exact conditional NULL-risk boundary.} Minimum number of
independent NULL units required for a one-sided 99.5\% Clopper--Pearson upper
bound no greater than 0.05.}
\label{tab:target_cp_boundary}
\small
\renewcommand{\arraystretch}{1.05}
\begin{tabular}{@{}rrrrrrrr@{}}
\specialrule{\heavyrulewidth}{0pt}{0pt}
False activations & 0 & 1 & 2 & 3 & 4 & 5 & 6 \\
Minimum NULL units & 104 & 146 & 182 & 216 & 248 & 279 & 309 \\
\specialrule{\heavyrulewidth}{0pt}{0pt}
\end{tabular}
\end{minipage}
\end{center}

Table~\ref{tab:target_cp_boundary} conditions on the final NULL count and
therefore does not specify a
universal total sample size. Stage B instead lower-bounds target NULL prevalence and checks
the exact predictive binomial probability that the already frozen final
partition contains the required NULL denominator. A 97.5\% prevalence lower
bound followed by 97.5\% conditional count assurance gives a combined
assurance lower bound of 95\% by the union bound.

The design calculation integrates random NULL, POSITIVE and activation
denominators under binomial planning laws, combining exact component pass
probabilities with a sufficient bounded-value assurance calculation. For the
value component, it controls a Bernstein lower tail for the sample mean and a
Maurer--Pontil inflation bound for the sample standard deviation before
applying Eq.~\eqref{eq:target_empirical_bernstein}. The selected size is a
conservative sufficient boundary under this scenario rather than a universal
sample-complexity theorem. These values belong to the strict three-way planning
calculation, not to the empirical operating characteristics of the realized
single-held-out cpg0012 design. The cpg0012 study used a separate, internally
frozen 95\% contract; its development and final outcomes appear in
Supplementary Section~\ref{sec:sr_cpg0012}.

\paragraph{Finite-campaign exact safety}\label{sec:finite_safety_proof}

Let \(\vartheta\) denote the finite campaign state: latent class composition, observation-channel parameter if unknown, and any value nuisance. At look \(q\), let \(\mathcal C_q\) be a confidence set satisfying
\begin{equation}\label{eq:simultaneous_coverage}
\Pr_\vartheta\{
\vartheta\in\mathcal C_q\text{ for all }q\in\mathcal Q
\}\geq1-\alpha.
\end{equation}
Let \(\Delta_q(\vartheta)\) be the costed contrast in
Eq.~\ref{eq:si_delta_campaign}.

\begin{proposition}[Finite-campaign safety and information-bill fallback]\label{prop:finite_safe}
Define
\[
L_q=\inf_{\vartheta'\in\mathcal C_q}\Delta_q(\vartheta').
\]
At a prespecified stopping look \(T\), switch the deployment complement only if \(L_T>0\). Then:

\begin{enumerate}
\item with probability at least \(1-\alpha\), every activated switch has \(\Delta_T(\vartheta)\geq0\);
\item if the policy abstains and all campaign banks retain \(K_0\), its net contrast relative to immediate \(K_0\) deployment is
\[
-C_{\mathrm{info,total}}(T)/N;
\]
\item if \(C_{\mathrm{info,total}}(q)\leq C_{\max}\) for all prespecified looks, the fallback loss is bounded below by \(-C_{\max}/N\).
\end{enumerate}
\end{proposition}

\begin{proof}
On the simultaneous coverage event in Eq.~\eqref{eq:simultaneous_coverage}, the true \(\vartheta\) belongs to \(\mathcal C_T\), including when \(T\) is data dependent. Therefore
\(
\Delta_T(\vartheta)\geq L_T>0
\)
whenever switching occurs. If the procedure abstains, every action-specific
utility equals its baseline counterpart and cancels in the contrast; the only
remaining term in Eq.~\ref{eq:si_delta_campaign} is the information bill. The
cost cap gives the last statement.
\end{proof}

\begin{remark}
Fallback still incurs the pilot cost. The guarantee therefore protects active
deviations while bounding the loss under abstention by the prespecified
information bill.
\end{remark}

\paragraph{Finite-population certifiability and sample complexity}
\label{sec:impossibility_proof}

Let a finite population contain \(B\) actionable banks among \(N\), and let
\(X\sim\mathrm{Hypergeometric}(N,B,q)\) be the actionable count in a simple
random pilot of size \(q<N\). Write \(d=N-q\), and fix a
deployment-complement target \(\rho\in[0,1)\). A possibly randomized rule
\(\varphi_q:\{0,\ldots,q\}\rightarrow[0,1]\) is a level-\(\alpha\)
certificate if
\[
\sup_{0\leq B\leq N}
\mathbb E_B\!\left[
 \varphi_q(X)\,
 \mathbf 1\!\left\{\frac{B-X}{d}\leq\rho\right\}
\right]\leq\alpha .
\]
Thus the controlled error is the probability of certifying when the
specific unobserved complement fails the prespecified actionability target.
Define
\[
c_\rho(d)=\lfloor\rho d\rfloor,\qquad
h(N,q,\rho)
=
\frac{\binom{q+c_\rho(d)}{q}}{\binom{N}{q}}.
\]

\begin{proposition}[Exact best-case finite-population certifiability boundary]
\label{prop:impossibility}
Every level-\(\alpha\) rule satisfies
\[
\varphi_q(q)\leq
\min\!\left\{1,\frac{\alpha}{h(N,q,\rho)}\right\}.
\]
Here monotone means that \(\varphi_q(x)\) is non-decreasing in \(x\). If
\(h(N,q,\rho)>\alpha\), no non-randomized monotone rule can certify at
any pilot outcome. If \(h(N,q,\rho)\leq\alpha\), the rule
\(\varphi_q(x)=\mathbf 1\{x=q\}\) is a non-randomized monotone
level-\(\alpha\) certificate. Hence \(h(N,q,\rho)\leq\alpha\) is the exact
best-case boundary for existence of a nontrivial monotone binary-count
certificate at look \(q\).
\end{proposition}

\begin{proof}
Under the boundary composition \(B_0=q+\lfloor\rho d\rfloor\), the event
\(X=q\) leaves exactly \(\lfloor\rho d\rfloor\) actionable banks in the
complement and is therefore a false-certification event. Its probability is
\(h(N,q,\rho)\), so level control implies
\(\varphi_q(q)h(N,q,\rho)\leq\alpha\). If the rule is
non-randomized and \(h>\alpha\), it has \(\varphi_q(q)=0\);
monotonicity then forces \(\varphi_q(x)=0\) for every \(x\leq q\).

Conversely, consider the rule that certifies only when \(X=q\). A false
certificate can occur only for
\(B\leq q+\lfloor\rho d\rfloor\). Since
\(\Pr_B(X=q)=\binom{B}{q}/\binom{N}{q}\) is non-decreasing in \(B\), the
largest false-certification probability is attained at the boundary
composition and equals \(h(N,q,\rho)\leq\alpha\).
\end{proof}

For a prespecified look set \(\mathcal Q_N\), define
\[
\mathcal Q^\dagger(N,\rho,\alpha)
=
\{q\in\mathcal Q_N:h(N,q,\rho)\leq\alpha\},
\qquad
q^\dagger(N,\rho,\alpha)
=
\min\mathcal Q^\dagger(N,\rho,\alpha),
\]
with \(q^\dagger\) undefined if the feasible set is empty. Because the
deployment complement and its integer target change with \(q\), the feasible
set can exhibit finite-population saw-tooth behavior; every prespecified look
must be evaluated rather than assuming feasibility for all
\(q\geq q^\dagger\). With \(L\) Bonferroni-spent looks, \(\alpha\) is
replaced by \(\alpha_{\mathrm{tot}}/L\).

\begin{corollary}[Large-campaign sample-complexity limit]\label{cor:n16}
For fixed \(q\) and \(0<\rho<1\),
\[
h(N,q,\rho)
=
\prod_{j=0}^{q-1}
\frac{q+\lfloor\rho(N-q)\rfloor-j}{N-j}
\longrightarrow \rho^q
\quad\text{as }N\rightarrow\infty.
\]
Thus the limiting best-case threshold is
\[
q_\infty(\rho,\alpha)
=
\left\lceil\frac{\log\alpha}{\log\rho}\right\rceil,
\]
apart from the exact equality case, for which the finite-\(N\) calculation
remains controlling. If eventual strict qualification is required, use the
smallest integer \(q\) satisfying \(\rho^q<\alpha\), rather than equality.
For the prespecified strict-majority target
\(\rho=1/2\) and \(\alpha=0.02/7\), \(q_\infty=9\); direct finite
enumeration still gives the first non-excluded design as \(N=25,q=14\), and
excludes every prespecified look at \(N=16\).
\end{corollary}

\begin{proof}
The product expression follows by expanding the ratio of binomial
coefficients. Each factor tends to \(\rho\) for fixed \(q\), giving
\(\rho^q\). Solving \(\rho^q\leq\alpha\) gives the limiting integer
threshold; direct enumeration supplies the finite-\(N\) statements.
\end{proof}

\begin{corollary}[Cost-adjusted binary value boundary]
\label{cor:cost_adjusted_boundary}
Suppose an activated action has per-bank gain at least \(g_+>0\) on an
actionable bank and at least \(-g_-\), \(g_-\geq0\), otherwise. Let \(C_q\)
be the non-negative total information cost at look \(q\). A sufficient condition for
strictly positive campaign-level executed value is
\[
\frac{B-X}{N-q}
>
\rho_V(q)
=
\frac{g_-+C_q/(N-q)}{g_++g_-}.
\]
If \(\rho_V(q)\geq1\), no complement composition can satisfy this
worst-case sufficient condition. Otherwise,
Proposition~\ref{prop:impossibility}, with
\(\rho=\rho_V(q)\), gives the exact best-case pilot-size boundary for
certifying it.
\end{corollary}

\begin{proof}
If the complement actionable fraction is \(p\), the stated lower gain bound
is \(p g_+-(1-p)g_-\). Requiring its total over \(N-q\) banks to exceed
\(C_q\), and solving for \(p\), gives \(\rho_V(q)\). The preceding
proposition then applies to this binary sufficient condition.
\end{proof}

\paragraph{Full three-class likelihood certificate}\label{sec:full_likelihood_certificate}

The binary-count proposition provides an interpretable special case, whereas
the complete three-class rule requires the prespecified observation model
below. Let
\[
\mathcal S_q=\{\mathbf s\in\mathbb Z_{\geq0}^3:\textstyle\sum_j s_j=q\}.
\]
For identity-mixture parameter \(p\),
\[
\Pr(S=s\mid k^\star=j,p)=p\,1(s=j)+(1-p)/3,
\]
so the diagonal signal accuracy is \((1+2p)/3\), not \(p\). Let
\(\mathcal C_{q,p}(\mathbf s)\) denote the exact set of
campaign/pilot composition pairs \((\mathbf M,\mathbf H)\) retained by the
prespecified probability-ordering screen and conditional prediction set. If
\(a_{q,p}(\mathbf s)\) is the prespecified equal-prior Bayes action and
\(\mathbf R=\mathbf M-\mathbf H\), define
\[
\gamma_{q,p}(\mathbf s;\mathbf M,\mathbf H)
=V_{\mathbf R}\{a_{q,p}(\mathbf s)\}-V_{\mathbf R}(K_0)
\]
and the deterministic least-favorable pair
\begin{equation}\label{eq:least_favourable_pair}
(\mathbf M^\dagger,\mathbf H^\dagger)_{q,p,\mathbf s}
=\mathop{\arg\min}_{(\mathbf M,\mathbf H)\in\mathcal C_{q,p}(\mathbf s)}
\gamma_{q,p}(\mathbf s;\mathbf M,\mathbf H),
\end{equation}
with lexicographic tie-breaking. The complete certification envelopes are
\begin{align}
B_{q,p}
&=\max_{\mathbf s\in\mathcal S_q}
\min_{(\mathbf M,\mathbf H)\in\mathcal C_{q,p}(\mathbf s)}
\gamma_{q,p}(\mathbf s;\mathbf M,\mathbf H),\label{eq:full_gain_envelope}\\
B^\Delta_{q,p}
&=\max_{\mathbf s\in\mathcal S_q}
\left\{\frac{16-q}{16}
\min_{(\mathbf M,\mathbf H)\in\mathcal C_{q,p}(\mathbf s)}
\gamma_{q,p}(\mathbf s;\mathbf M,\mathbf H)
-\frac{0.1+0.5q}{16}\right\}.\label{eq:full_net_envelope}
\end{align}

\begin{proposition}[Full-likelihood non-certifiability for the prespecified rule]\label{prop:full_likelihood_impossibility}
For the frozen prior support, value table, state classifier and
\[
q\in\{2,4,6,8,10,12,14\},\qquad
p\in\{0.70,0.80,0.92\},
\]
\(B_{q,p}=0\) and
\[
B^\Delta_{q,p}=-(0.1+0.5q)/16<0.
\]
Every possible signal composition is assigned state
\texttt{INDETERMINATE}; none is assigned
\texttt{ACTIONABLE\_}\allowbreak\texttt{OPPORTUNITY}, and none activates.
\end{proposition}

\begin{proof}
There are \(\binom{q+2}{2}\) possible signal-count vectors at a fixed
\((q,p)\). For each vector, we enumerated all 153 campaign
compositions of 16 banks, every feasible pilot composition, the exact
three-class emission likelihood, the probability-ordered composition screen,
the conditional pilot-composition prediction set, the Bayes action and every
gain in Eq.~\eqref{eq:least_favourable_pair}. The enumeration covers
\[
3\sum_{q\in\{2,4,\ldots,14\}}\binom{q+2}{2}=1,113
\]
signal compositions. For every row, the least-favorable gain is non-positive;
the maximum over rows is zero. The prespecified actionable state requires a
strictly positive lower opportunity exceeding its frozen threshold, so it is
unattainable. Substitution in Eq.~\eqref{eq:full_net_envelope} yields the
strictly negative information-bill values in Table~\ref{tab:full_likelihood_certificate}.

The enumeration records the least-favorable
\((\mathbf M^\dagger,\mathbf H^\dagger,\mathbf R^\dagger)\) for every row.
An independent implementation verified support completeness, membership of
every least-favorable pair in \(\mathcal C_{q,p}(\mathbf s)\), attainment of
the reported infimum and the cost identity. This exhaustive finite result
applies to the frozen algorithm, not to every possible three-class
diagnostic.
\end{proof}

\paragraph{Campaign-size, target and error-budget surface}
\label{sec:n_alpha_map}

For each \(N\in\{8,\ldots,128\}\) and one-sided test size
\(\alpha\in[10^{-4},10^{-1}]\), we enumerated every
\(q\in\{1,\ldots,N-1\}\) and recorded
\[
q^\dagger(N,1/2,\alpha)
=
\min\left\{
q:
\frac{\binom{q+\lfloor(N-q)/2\rfloor}{q}}
     {\binom{N}{q}}
\leq\alpha
\right\},
\]
with \(q^\dagger\) undefined when the set is empty. This analytic design slice
uses the prespecified strict-majority target \(\rho=1/2\); Proposition~\ref{prop:impossibility} defines the complete
\((N,q,\rho,\alpha)\) surface. At the exact prespecified allocation
\(\alpha=0.02/7\), no qualifying pilot exists for \(N=16\); the first
campaign size not excluded by this binary-count condition is \(N=25\), where
\(q^\dagger=14\). The corresponding first sizes not excluded are \(29\),
\(25\), \(19\) and \(12\) for
\(\alpha=10^{-3},0.02/7,0.01\) and \(0.05\), respectively (main-text
Fig.~3). These values describe the most favorable
binary pilot event and give necessary, rather than sufficient, conditions for
the complete three-class rule. The boundary does not include measurement error,
pilot cost or imperfect actionable classification, each of which can add
constraints. At the rounded display value \(\alpha=3\times10^{-3}\),
\(N=24,q^\dagger=15\) would escape this condition with boundary-null
probability \(2.964\times10^{-3}\); because this exceeds \(0.02/7\), it is reported only
as a knife-edge rounding sensitivity.

\paragraph{Constrained-risk validation}\label{sec:risk_proof}

\begin{proposition}[Exact validation after independent calibration]\label{prop:risk}
Let \(\widehat\tau\) be any threshold measurable with respect to tuning and lock data. Conditional on those data, let \(D_1,\ldots,D_n\) be iid held-out simulator validation flags with
\[
D_i\sim\mathrm{Bernoulli}\{R_0(\widehat\tau)\}.
\]
Let \(U_\alpha(K,n)\) be the one-sided Clopper--Pearson upper confidence limit computed from \(K=\sum_iD_i\). Then
\[
\Pr\{R_0(\widehat\tau)\leq U_\alpha(K,n)\}\geq1-\alpha.
\]
Therefore the rule that claims risk control only when \(U_\alpha(K,n)\leq r_0\) has false certification probability at most \(\alpha\).
\end{proposition}

\begin{proof}
Conditional on tuning and lock data, \(\widehat\tau\) is fixed and \(K\) is binomial with parameter \(R_0(\widehat\tau)\). Exact Clopper--Pearson coverage applies~\cite{clopper1934}. Integrating over the tuning and lock data preserves the inequality.
\end{proof}

\begin{remark}
The result requires independent Bernoulli units, disjoint tuning and validation
fixtures, and a threshold fixed before validation. The prespecified simulation
therefore uses one campaign per distinct fixture and four disjoint partitions.
\end{remark}

\paragraph{Adoption boundary}\label{sec:adoption_proof}

\begin{proposition}[Lifecycle boundary, band inversion and random reuse]\label{prop:adoption}
Let \(C_{\mathrm{inc}}>0\) and
\[
\Delta_{\mathrm{life}}(H)
=
\Delta_{\mathrm{op}}-\frac{C_{\mathrm{inc}}}{H}.
\]
Then:

\begin{enumerate}
\item if \(\Delta_{\mathrm{op}}>0\), lifecycle value is non-negative exactly when \(H\geq H^\star=C_{\mathrm{inc}}/\Delta_{\mathrm{op}}\); for integer reuse, the smallest admissible value is \(\lceil H^\star\rceil\);
\item if \(\Delta_{\mathrm{op}}\leq0\), no finite positive \(H^\star\) exists;
\item if \(C_{\mathrm{inc}}\in[C_L,C_U]\) and \(0<L\leq\Delta_{\mathrm{op}}\leq U\), then
\[
H^\star\in[C_L/U,C_U/L];
\]
\item if \(H>0\) almost surely is random, the displayed expectations are finite
and the per-model cost is allocated equally over realized reuse, expected lifecycle contrast is
\[
\mathbb E\Delta_{\mathrm{life}}
=
\Delta_{\mathrm{op}}-C_{\mathrm{inc}}\mathbb E(1/H),
\]
and \(\mathbb E(1/H)\geq1/\mathbb E(H)\).
\end{enumerate}
\end{proposition}

\begin{proof}
The first two statements follow by solving
\(
\Delta_{\mathrm{op}}-C_{\mathrm{inc}}/H\geq0
\)
for \(H>0\). The ratio is increasing in cost and decreasing in positive operational gain, which gives the interval endpoints. The expectation identity is linearity of expectation. Since \(h\mapsto1/h\) is convex on \(h>0\), Jensen's inequality gives the final inequality.
\end{proof}

\subsection{Policy learning and finite-campaign computation}\label{sec:si_method_policy}

\subsubsection{Policy architecture and estimation}\label{sec:architecture}

\paragraph{Relationship to neighbouring decision systems}\label{sec:related_systems}

Value-of-information analysis compares the expected benefit of further
evidence with its acquisition cost~\cite{fenwick2020voi}. Active feature
acquisition selects costly, instance-specific measurements to improve a
downstream prediction or decision~\cite{ma2019eddi,vonkleist2025afa}. The
\OPAL{} probe layer shares these structures but places them inside a
commitment contract: predictive improvement without a possible action change
has zero gross option value under Proposition~\ref{prop:option}, and
information cost remains charged after fallback.

High-confidence policy improvement, Seldonian algorithms, SPIBB and CSPI-MT
protect a baseline or behavioral constraint while selecting a candidate
policy~\cite{thomas2015,thomas2019seldonian,laroche2019spibb,cho2025cspimt}.
Selective classification and policy learning with abstention instead
formalize a risk--coverage or default-policy trade-off
~\cite{chow1970,elyaniv2010selective,sawarni2026abstention}.
Learn-then-Test and conformal risk control calibrate a prespecified family
against population risk~\cite{angelopoulos2025,angelopoulos2024crc}.
Conservative bandits and safe Bayesian optimization protect performance or
feasibility during sequential online exploration
~\cite{wu2016conservative,sui2015safeopt}. \OPAL{} combines these established
ideas within an integrated scientific authorization contract that jointly
adjudicates action opportunity, precommitment observability, finite-complement
inference, risk, executed value, nontriviality and lifecycle cost.

Randomized probe exploration supplies known acquisition propensities. This
enables inverse-propensity or doubly robust evaluation of a target probe policy
when overlap, consistency and logging assumptions hold
~\cite{swaminathan2015,jiang2016}. In the author-generated simulator,
common-random-number potential outcomes additionally permit direct evaluation
of all prespecified actions. Those oracle outcomes are evaluation-only and
cannot replace propensity-aware off-policy evaluation in a physical system.
Table~\ref{tab:authorization_comparator_contract} states the common comparison
contract and the additional data requirements of the nearest method families.

\begin{table}[!t]
\centering
\caption{\textbf{Nearest-neighbour comparison contract.} Structural comparison
uses each method's native data contract. Quantitative comparison is valid only
when assignments are generated from the identical frozen decision-time
observables, candidate actions, outcome-independent split and cost contract.
``Not directly applicable'' is not interpreted as inferior performance.}
\label{tab:authorization_comparator_contract}
\fontsize{7.0}{7.9}\selectfont
\renewcommand{\arraystretch}{1.08}
\rowcolors{2}{tablerow}{white}
\begin{tabularx}{\textwidth}{@{}L{2.2cm}L{3.4cm}L{3.0cm}Y@{}}
\specialrule{\heavyrulewidth}{0pt}{0pt}
\rowcolor{tablehead}
Method & Frozen implementation under a common one-step contract &
Additional native requirements & Status for the reported CTRP lock \\
\specialrule{\lightrulewidth}{0pt}{0pt}
Always-\(K_0\) & Execute \(K_0\) for every family & None & Exact negative control \\
Forced adaptation & Execute the development learner's candidate action without an authorization gate & Same candidate action as the gated method & Exact endpoint; not a safety method \\
Simple uncertainty gate & Switch when a fixed standardized predicted contrast exceeds a prespecified threshold & Frozen contrast direction, uncertainty estimator and threshold & Computable retrospectively; not binding if added after opening \\
ENB / EVSI & Switch when predicted expected net benefit is positive; purchase information only when EVSI exceeds acquisition cost & Full EVSI additionally needs a frozen prior, likelihood, posterior update and sampling design & CTRP has no optional decision-time probe, so a fair static comparator is ENB, not full EVSI \\
HCPI & Select from a finite observable-policy class only when a simultaneous paired-value lower bound clears the baseline & Logged or full-information value contract for candidate policies & Closest static safe-policy comparator; new implementation would be retrospective \\
SPIBB / CSPI-MT & Bootstrap to a behavior baseline in poorly supported state--action regions, or calibrate multiple threshold candidates against baseline value & Behavior policy, action support and usually an MDP/logged-trajectory contract for SPIBB; a frozen threshold class and safety tests for CSPI-MT & CSPI-MT can be adapted to a frozen threshold class; SPIBB cannot be reproduced faithfully by inventing propensities \\
LTT / CRC gate & Calibrate a prespecified threshold family against a monotone risk on independent calibration data & Exchangeability, fixed loss, candidate family and multiplicity rule & Direct risk-gate comparator; its risk guarantee does not imply positive executed value \\
Conservative bandit / SafeOpt & Update after settled feedback while maintaining cumulative baseline performance or pointwise feasibility & Frozen sequence, feedback times, horizon, and bandit or GP/kernel safety model & Not directly applicable to one-shot frozen assignments; requires a separate sequential benchmark \\
\OPAL{} full contract & Candidate action plus opportunity, risk, executed-value, nontriviality and cost qualification, otherwise \(K_0\) & All gates and error allocations frozen before evaluation & Binding only for components present in the pre-outcome seal \\
\specialrule{\heavyrulewidth}{0pt}{0pt}
\end{tabularx}
\end{table}

For each quantitatively implemented method \(m\), define
\[
A_{mi}=\mathbf 1\{\pi_m(X_i)\neq K_0\},\qquad
D_{mi}=U_i\{\pi_m(X_i)\}-U_i(K_0).
\]
We report NULL false activation and its one-sided 99.5\% upper bound; overall
activation coverage; POSITIVE sensitivity and its exact interval; all-family
and activated-only executed value; the corresponding one-sided paired-value
lower bound; action and information costs; development and calibration sample
sizes; the first learning-curve size meeting the joint qualification rule;
harmful-activation fraction; and fallback frequency with reason codes. A
complete evaluation requires two comparisons: a shared-learner ablation that
changes only the authorization layer, and a full-method comparison in which
each method uses its native training procedure under the same observables,
budget, actions and utility. The simulator threshold ablation provides the
first comparison but not the second. Post-lock analyses remain retrospective.

\paragraph{Retrospective static CTRP comparison.}
After the one-time CTRP lock was opened, the complete measured action--outcome
table enabled a reproducible full-information reanalysis of deterministic
one-step policies (Supplementary Table~\ref{tab:ctrp_retro_comparators}). This
\emph{post-lock, descriptive} analysis leaves the sealed \OPAL{} assignment
unchanged and does not provide prospective head-to-head evidence. The forced learned
candidate is the development model's predicted expected-net-benefit (ENB)
action; the uncertainty rule requires its standardized predicted contrast to
exceed 1.96. The HCPI-style analogue authorizes that one ENB policy only if
its calibration-wide one-sided 97.5\% paired-value lower bound is positive.
The LTT/CRC-like rule selects the first member of the frozen threshold grid
whose calibration NULL-risk upper bound is at most 0.05. The CSPI-MT-style
analogue additionally requires at least one activation and a Bonferroni
positive value lower bound across the eight frozen thresholds. These static
analogues do not reproduce the native data and interaction requirements of
logged-data HCPI, full EVSI, SPIBB, CSPI-MT or a sequential conservative
bandit/BO experiment.

\begin{sidewaystable}[p]
\centering
\begin{minipage}{0.92\textheight}
\centering
\caption{\textbf{Post-lock full-information CTRP comparator reanalysis.}
All methods use the same 254-family measured-outcome table, decision-time
metadata, candidate action and \(K_0=K_2\) utility contract. Values in
parentheses are one-sided 99.5\% Clopper--Pearson NULL-risk upper bounds or
one-sided 97.5\% paired-\(t\) value lower bounds, as indicated. This
retrospective table is not a binding prospective comparison.}
\label{tab:ctrp_retro_comparators}
\fontsize{7.5}{8.4}\selectfont
\renewcommand{\arraystretch}{1.08}
\rowcolors{2}{tablerow}{white}
\begin{tabular}{@{}L{2.7cm}L{2.55cm}L{1.45cm}L{1.85cm}L{3.25cm}L{1.65cm}L{1.4cm}@{}}
\specialrule{\heavyrulewidth}{0pt}{0pt}
\rowcolor{tablehead}
Method & NULL false activation (UCB) & Coverage & POSITIVE sensitivity &
Mean executed gain (LCB) & Mean action cost & Fallback \\
\specialrule{\lightrulewidth}{0pt}{0pt}
Always-\(K_2\) & \(0/244\;(0.02148)\) & \(0/254\) & \(0/7\) & \(0\;(0)\) & \(0.01000\) & \(1.0000\) \\
Always-\(K_4\) forced & \(244/244\;(1.00000)\) & \(254/254\) & \(7/7\) & \(-0.013073\;(-0.018604)\) & \(0.03000\) & \(0\) \\
ENB / forced learned candidate & \(9/244\;(0.08013)\) & \(9/254\) & \(0/7\) & \(-7.09\times10^{-4}\;(-1.166\times10^{-3})\) & \(0.01071\) & \(0.9646\) \\
Simple 1.96-SE gate & \(0/244\;(0.02148)\) & \(0/254\) & \(0/7\) & \(0\;(0)\) & \(0.01000\) & \(1.0000\) \\
HCPI-style global LCB & \(0/244\;(0.02148)\) & \(0/254\) & \(0/7\) & \(0\;(0)\) & \(0.01000\) & \(1.0000\) \\
LTT/CRC-like risk gate & \(0/244\;(0.02148)\) & \(0/254\) & \(0/7\) & \(0\;(0)\) & \(0.01000\) & \(1.0000\) \\
CSPI-MT-style multiplicity gate & \(0/244\;(0.02148)\) & \(0/254\) & \(0/7\) & \(0\;(0)\) & \(0.01000\) & \(1.0000\) \\
\OPAL{} frozen \(\tau=0.5\) & \(0/244\;(0.02148)\) & \(0/254\) & \(0/7\) & \(0\;(0)\) & \(0.01000\) & \(1.0000\) \\
\specialrule{\heavyrulewidth}{0pt}{0pt}
\end{tabular}
\end{minipage}
\end{sidewaystable}

The ungated ENB policy had calibration mean gain \(-4.676\times10^{-4}\)
and lower bound \(-8.423\times10^{-4}\); its \(\tau=0\) NULL-risk upper
bound was \(0.06501>0.05\). Every frozen threshold \(\tau\geq0.5\) produced
zero activation. Consequently, no reconstructed threshold simultaneously
met the risk limit, nonzero coverage and a strictly positive value lower
bound. The archive contains no optional decision-time measurement, so
incremental information-acquisition cost is zero for these static rules; the
table instead reports the prespecified action cost
\(0.01(K-1)\). A single development--calibration--lock split cannot identify a
common learning curve, runtime overhead or repeated-lock failure rate, so
sample efficiency and these failure endpoints are unavailable.

\paragraph{Cross-fitted fixed baseline and value model}\label{sec:crossfit}

The deployed fixed baseline is learned separately inside each outer training fold:
\[
\widehat K_{0,-f}
=
\arg\max_{k\in\mathcal K}
\frac{1}{|I_{-f}|}
\sum_{i\in I_{-f}}\widehat U_i(k).
\]
Evaluation campaigns in fold \(f\) do not participate in baseline selection, feature scaling, model training, expert selection or threshold calibration. This prevents a full-data baseline from leaking information into every outer fold. Cross-fitted prediction and evaluation follow the separation principle used in debiased machine learning~\cite{chernozhukov2018}, although no causal interpretation is attached without its own identification assumptions.

For each campaign and capacity, the value learner produces a mean and paired contrast uncertainty against \(K_0\). The same outputs define the entire \(\lambda\) frontier. Changing \(\lambda\) changes only a critical value, never the model, features, folds, cost contract or score. Degenerate uncertainty has a frozen action: \(K_0\).

\paragraph{Soft residual experts}\label{sec:soft_experts}

The expert layer starts with a global prediction and models only its residual. For outer fold \(f\), expert hyperparameters are chosen using \(I_{-f}\) alone. The candidate expert count is finite and includes zero. Soft gating avoids discontinuous assignment, while residual shrinkage prevents small experts from replacing a stable global effect. The one-standard-error rule chooses the simplest model statistically indistinguishable from the inner-fold optimum.

We checked three properties:

\begin{enumerate}
\item perturbing an outer-test campaign cannot alter that fold's selected hyperparameters;
\item posterior belief changes can change predicted action values when an expert is supported;
\item beliefs outside prespecified support invoke the global-only model.
\end{enumerate}

Mixtures of local experts provide the architectural precedent~\cite{jacobs1991,jordan1994}. Here, they operate within a fail-closed opportunity ceiling rather than defining a new generic mixture model.

\paragraph{Probe exploration and action-change gate}\label{sec:probe}

For each probe \(q\), development exploration records the known propensity, pre- and post-probe beliefs, actions, gross realized gain, resource charge, latency charge and net realized gain. The action-change probability model and net-gain model are trained on this randomized data. At deployment, a probe is purchased only when both
\[
\widehat{\Pr}\{\pi(X,Z_r)\neq\pi(X)\mid X\}\geq p_{\min}
\]
and a calibrated lower confidence bound for net gain is positive. The explicit no-probe action has cost and value zero. The confirmatory evaluation does not continue forced exploration unless a separate policy explicitly includes and charges it.

\subsubsection{Exact finite-campaign algorithm}\label{sec:finite_algorithm}

\paragraph{Primary population and observation channel}\label{sec:finite_population}

The prespecified boundary study contains \(N=16\) banks with latent class \(k^\star\in\{0,1,2\}\), from which the pilot is sampled without replacement. The primary simulator uses a signal channel known by design:
\[
\Pr(S=s\mid k^\star=k,p)
=
p\,\mathbf 1(s=k)+\frac{1-p}{3},
\qquad
p\in\{0.70,0.80,0.92\}.
\]
The value table is
\[
\begin{array}{@{}lccc@{}}
\toprule
 & k^\star=0 & k^\star=1 & k^\star=2\\
\midrule
K=0&0.495&0.495&0.495\\
K=1&0.465&0.665&0.665\\
K=2&0.445&0.645&0.845
\\\bottomrule
\end{array}
\]
and is known in this simulation study. An external application would require independent confidence sets for both the observation channel and the value table.

\paragraph{Likelihood inversion}\label{sec:likelihood_inversion}

For composition \(\mathbf M=(M_0,M_1,M_2)\), \(\sum_jM_j=N\), let \(\mathbf H\) be the latent class counts in the pilot. The probability of \(\mathbf H\) is multivariate hypergeometric:
\[
\Pr(\mathbf H=\mathbf h\mid\mathbf M)
=
\frac{\prod_j\binom{M_j}{h_j}}{\binom{N}{q}}.
\]
Conditional on \(\mathbf H\), observed signal counts follow the prespecified identity-mixture channel
\[
\Pr(S=s\mid k^\star=j,p)=p\,1(s=j)+(1-p)/3,
\qquad p\in\{0.70,0.80,0.92\}.
\]
Thus \(p\) is the weight on an identity channel and the diagonal accuracy is
\((1+2p)/3\). Write \(\mathcal S_q=\{\mathbf s\in
\mathbb Z_{\geq0}^3:\sum_j s_j=q\}\), and let
\[
\ell_{\mathbf M,q,p}(\mathbf s)
=
\sum_{\substack{\mathbf h:\,\sum_jh_j=q\\0\leq h_j\leq M_j}}
\Pr(\mathbf H=\mathbf h\mid\mathbf M)
\Pr_p(\mathbf S=\mathbf s\mid\mathbf H=\mathbf h)
\]
be the exact marginal likelihood. Its probability-ordering \(p\)-value is
\begin{equation}\label{eq:composition_pvalue}
a_{\mathbf M,q,p}(\mathbf s)
=
\sum_{\mathbf s'\in\mathcal S_q:\,
\ell_{\mathbf M,q,p}(\mathbf s')
\leq\ell_{\mathbf M,q,p}(\mathbf s)}
\ell_{\mathbf M,q,p}(\mathbf s').
\end{equation}
All outcomes tied at the observed likelihood are included. With
\[
\alpha_M=\alpha_H=\frac{0.020}{2\times7},
\]
the composition confidence set is
\begin{equation}\label{eq:composition_set}
\mathcal C^M_{q,p}(\mathbf s)
=
\{\mathbf M:a_{\mathbf M,q,p}(\mathbf s)>\alpha_M\}.
\end{equation}

For each retained \(\mathbf M\), the conditional mass of a feasible pilot
composition is
\[
\rho_{\mathbf M,q,p}(\mathbf h\mid\mathbf s)
=
\frac{\Pr(\mathbf H=\mathbf h\mid\mathbf M)
\Pr_p(\mathbf S=\mathbf s\mid\mathbf H=\mathbf h)}
{\ell_{\mathbf M,q,p}(\mathbf s)}.
\]
The set \(\mathcal C^H_{q,p}(\mathbf s,\mathbf M)\) orders these masses from
largest to smallest, includes terms until cumulative mass is at least
\(1-\alpha_H\), and includes every composition tied at the boundary mass.
The exact pair set used by the policy is therefore
\begin{equation}\label{eq:pair_set}
\mathcal C_{q,p}(\mathbf s)
=
\{(\mathbf M,\mathbf H):
\mathbf M\in\mathcal C^M_{q,p}(\mathbf s),\
\mathbf H\in\mathcal C^H_{q,p}(\mathbf s,\mathbf M)\}.
\end{equation}
The calculation uses no normal approximation, sample-variance substitution or
outcome-dependent change of ordering. Operating-characteristic confidence limits use a
separate \(0.015/(3\times7)\) allocation over three endpoints and seven pilot
sizes.

\begin{proposition}[Exact finite-support pair coverage]\label{prop:exact_pair_coverage}
For every prespecified \(q\), known channel \(p\) and finite campaign
composition \(\mathbf M\),
\[
\Pr_{\mathbf M,p}\{(\mathbf M,\mathbf H)
\in\mathcal C_{q,p}(\mathbf S)\}
\geq1-\alpha_M-\alpha_H
=0.997142857.
\]
Consequently, for the seven prespecified looks,
\[
\Pr_{\mathbf M,p}\{(\mathbf M,\mathbf H_q)
\in\mathcal C_{q,p}(\mathbf S_q)
\text{ for every }q\in\mathcal Q\}
\geq1-7(\alpha_M+\alpha_H)=0.98.
\]
\end{proposition}

\begin{proof}
The discrete probability-ordering \(p\)-value in
Eq.~\eqref{eq:composition_pvalue} is super-uniform, so the probability that
the true \(\mathbf M\) is excluded from Eq.~\eqref{eq:composition_set} is at
most \(\alpha_M\). Conditional on \((\mathbf S,\mathbf M)\), construction of
\(\mathcal C^H\) contains conditional design probability at least
\(1-\alpha_H\); inclusion of
boundary ties can only increase that mass. A union bound gives the pair
coverage at one look and a second union bound gives simultaneous coverage
over the frozen looks; independence between looks is not required.

As a finite-support check, an independent implementation enumerated all 153
campaign compositions, every feasible pilot composition and every observed
count vector in all 21 \((q,p)\) cells. The minimum attained pair coverage was
\(0.997392773\), above the analytic target \(0.997142857\), in every cell.
Re-execution reproduced the archived certificate exactly.
\end{proof}

For sequential looks \(\mathcal Q\), the error probability is allocated before outcomes are observed. At each look, the algorithm:

\begin{enumerate}
\item enumerates all integer compositions \(\mathbf M\);
\item evaluates the exact pilot-signal likelihood;
\item retains compositions in the inverted confidence set;
\item computes the optimal deployment-complement action for each retained composition;
\item evaluates the costed Eq.~\ref{eq:si_delta_campaign};
\item takes the minimum over the set;
\item switches only when the minimum is positive.
\end{enumerate}

The fixed campaign cost is
\[
C_{\mathrm{info,total}}(q)=0.1+0.5q
\]
in inherited normalized simulation units, producing
\begin{equation}\label{eq:si_delta_campaign}
\Delta_{\mathrm{camp}}(q)
=
\frac{16-q}{16}G_{A,D}(q)
-\frac{0.1+0.5q}{16}.
\end{equation}
These are simulated design costs, not measured facility costs.

\subsection{Simulation design and LINCS--LJP archive construction}\label{sec:si_method_validation}

\subsubsection{Mechanism population and risk validation}\label{sec:si_method_risk}

\paragraph{Continuous mechanism population}\label{sec:population}

\paragraph{Mechanism variables}\label{sec:continuous_axes}

Table~\ref{tab:axes} gives the complete sampling axes and distinguishes design
stress envelopes from fitted facility distributions.

\begin{table}[t]
\centering
\caption{\textbf{Continuous mechanism population.} Axes are sampled independently and the demand mixture is Dirichlet. The dimensions are facility-interpretable, but the numerical endpoints are design stress envelopes rather than fitted facility distributions.}
\label{tab:axes}
\small
\renewcommand{\arraystretch}{1.14}
\rowcolors{2}{tablerow}{white}
\begin{tabularx}{\textwidth}{@{}L{2.7cm}C{2.45cm}L{2.75cm}Y@{}}
\specialrule{\heavyrulewidth}{0pt}{0pt}
\rowcolor{tablehead}
Axis & Range / distribution & Generative role & Rationale / anchor \\
\specialrule{\lightrulewidth}{0pt}{0pt}
Matched-capacity success & Uniform [0.70, 0.92] & Base success probability under matched capacity & Author-specified envelope from moderately to highly reliable execution; not a facility estimate \\
Under-provision penalty & Uniform [0.05, 0.30] & Success loss per missing capacity level & Normalized stress envelope chosen to create resolvable under-capacity consequences \\
Over-provision penalty & Uniform [0.05, 0.30] & Success loss per excess capacity level & Symmetric normalized envelope for waste or unnecessary intervention; not a measured cancellation rate \\
Survey separability & Uniform [0.34, 0.97] & Diagonal probability of the three-class signal channel & Lower endpoint is near chance (\(1/3\)); upper endpoint is highly informative but imperfect \\
Routing-cost multiplier & Uniform [0.80, 1.20] & Multiplier on the frozen operational cost vector & Symmetric \(\pm20\%\) sensitivity envelope around the prespecified normalized contract \\
Common success shift & Uniform [$-0.05$, 0.05] & Capacity-invariant campaign difficulty & Symmetric nuisance perturbation used to test robustness without changing action ordering by construction \\
Demand mixture & Dirichlet(2,2,2) & Prevalence of three capacity needs & Symmetric interior design that avoids fixing class balance while retaining heterogeneous campaigns \\
\specialrule{\heavyrulewidth}{0pt}{0pt}
\end{tabularx}
\end{table}

For need class \(n\) and action \(k\), the success probability is the matched probability minus linear under- or over-provision penalties plus the common shift, clipped to \([0.02,1]\). Operational action cost is subtracted on the same utility scale. Oracle best fixed capacity, class-specific capacity, \(G_M\), \(G_{XZ}\), class support and utility gaps are computed deterministically before a learned policy is evaluated.

NULL mechanisms have \(G_M\leq g_0\). Positive-actionable mechanisms require
\(G_M\geq g_1>0\), so that \(\eta=G_{XZ}/G_M\) is defined, together with
\(\eta\geq\eta_{\min}\), at least two supported optimal-capacity classes and a
minimum utility gap. We classify all other mechanisms as weak/ambiguous.
Conditional samplers either return the requested number of distinct fixtures or
fail closed.

The general framework learns \(K_0\) within training folds. For this prespecified simulation, we instead conditioned the sampled mechanism population on the oracle-optimal fixed capacity being \(K_0=1\). This restriction isolates whether observable information justifies routing away from a common intermediate-capacity baseline. The results therefore do not establish that \(K_0=1\) is generally optimal in facilities or directly transport to populations with an optimal fixed action of \(0\) or \(2\).

\paragraph{Campaign score}\label{sec:campaign_score}

Each campaign contains 200 banks. Need classes and survey signals are generated from the fixture. A frozen global value map chooses
\[
\widehat K_i
=
\arg\max_k\widehat V(k\mid X_i)
\]
from the survey and prespecified decision randomization only. This action rule is deployable in the simulated decision contract. The prespecified simulation then evaluates it using
\[
S^{\mathrm{eval}}
=
\frac{1}{200}\sum_{i=1}^{200}
\{U_i(\widehat K_i)-U_i(K_0)\},
\]
where common-random-number potential outcomes supply \(U_i(k)\) for every action. Thus, \(S^{\mathrm{eval}}\) is an oracle-assisted simulation evaluation score, unavailable at facility commitment time. A field implementation would instead require an observable pilot score \(S^{\mathrm{dep}}(X,Z_{\mathrm{pilot}})\), calibrated on the same pilot information set and evaluated under an independently prespecified coverage argument.

\paragraph{Prespecified tuning, locking and validation}\label{sec:study3_design}

The design has four non-overlapping populations
(Table~\ref{tab:study3_design}):

\begin{table}[t]
\centering
\caption{\textbf{Prespecified simulation design for constrained-risk calibration.}}
\label{tab:study3_design}
\small
\renewcommand{\arraystretch}{1.14}
\rowcolors{2}{tablerow}{white}
\begin{tabularx}{\textwidth}{@{}Y C{2.0cm} C{2.0cm} Y@{}}
\specialrule{\heavyrulewidth}{0pt}{0pt}
\rowcolor{tablehead}
Role & Distinct fixtures & Campaigns & Function \\
\specialrule{\lightrulewidth}{0pt}{0pt}
NULL tuning & 140 & 140 & Construct order-statistic threshold class and select lowest passing threshold \\
NULL risk lock & 500 & 500 & Independently lock threshold or force always-\(K_0\) fallback \\
NULL validation & 400 & 400 & Evaluate binding false-activation risk \\
Positive validation & 500 & 500 & Estimate sensitivity and heterogeneity; no minimum threshold \\
\specialrule{\lightrulewidth}{0pt}{0pt}
Total & 1,540 & 1,540 & 308,000 banks at 200 per campaign \\
\specialrule{\heavyrulewidth}{0pt}{0pt}
\end{tabularx}
\end{table}

The primary unit is one continuous fixture with one fresh campaign (\(M_{\mathrm{primary}}=1\)). Repeated campaigns from one fixture may be used only in a separately labeled cluster-valid variance decomposition. The binding risk level is \(r_0=0.05\), evaluated with a one-sided exact upper bound at component \(\alpha=5\times10^{-3}\). If the risk lock fails, the active evaluator is disabled. Sensitivity is reported with an exact interval and no prespecified lower threshold.

Candidate design selection and design certification use separate fixture
libraries, campaign realizations and seed namespaces. Certification fixes the
sample sizes before the independent simulation is drawn. The resulting study
tests the simulator risk claim, not deployability or physical-facility
performance.

\subsubsection{LINCS--LJP temporal-profile acquisition evaluation}
\label{sec:si_method_lincs_ljp}

\paragraph{Source construction and component-disjoint partitions}
\label{sec:si_lincs_sources}

The transcriptomic source was the GSE70138 Level 5 MODZ matrix and its
associated signature, perturbagen, cell and landmark-gene annotations. We
retained the ordered 978 directly measured landmark genes and used each
published MODZ consensus signature once, without an additional replicate
collapse~\cite{subramanian2017l1000}. The outcome source was the LJP 72-h
growth-rate dataset~\cite{niepel2017lincs}. Episodes
were matched exactly by LJP batch, cell line, compound identity and dose.
For each matched episode, the endpoint was the arithmetic mean of exactly
three unique finite biological-replicate \(\mathrm{GRvalue}\) measurements,
clipped to
\[
M_i^*=\min\{1,\max(-1,M_{i,\mathrm{raw}})\}.
\]
The clipping defines a bounded decision endpoint and does not replace the
reported raw values in the source archive.

The complete-case frame contained 4,049 episodes in 105 non-overlapping
perturbation-identity components. An episode comprised one exact combination of
batch, cell line, component, compound and dose. Six cell lines and the LJP005
and LJP006 batches yielded 12 batch-by-cell technical groups. Two additional
compound identities had matched 3-h and 24-h profiles but no 72-h endpoint;
their 70 episodes were excluded before partitioning and do not enter any
reported denominator.

Components were assigned without balancing to LEARN, CALIBRATION,
CERTIFICATION and FINAL\_LOCK in counts of 42, 21, 21 and 21. All doses,
cells, batches and repeated episodes belonging to one component remained in
the same partition. The partition episode counts were 1,649, 750, 790 and
860, respectively. LEARN, CALIBRATION and CERTIFICATION were used for model,
threshold and operational-criterion development. Only FINAL\_LOCK is treated
as the untouched final evaluation.

After the independent evaluation was complete, we reconstructed two
publication indexes from the public source metadata and the frozen matching
and partition rules. The file
\texttt{lincs\_ljp\_episode\_index.csv} records the 4,049 episode
identifiers, source-signature identifiers, component membership and fixed
partition assignments. The file
\texttt{lincs\_ljp\_component\_map.csv} records the 105
component identities, structure identifiers, fixed partitions and episode
counts. Neither file contains growth-rate outcomes, model predictions,
certification values or FINAL values. Their construction did not repeat model
fitting, rule selection or evaluation; the independently returned aggregate
results remain the source of all reported LINCS--LJP statistics.

\paragraph{Prediction models, acquisition score and net value}
\label{sec:si_lincs_models}

The 3-h L1000 profile \(X_i\) was available before the optional measurement,
and \(Z_i\) was the matched 24-h profile. Two component-equal weighted ridge
models with unpenalized intercepts predicted \(M_i^*\). The first used
\(X_i\), log-dose, cell-line indicators and batch indicators; the second used
the same inputs together with \(Z_i\). Ridge penalties were selected
separately by six-fold component-disjoint cross-validation within the 42
LEARN components, followed by refitting on all LEARN components. Predictions
\(p_{X,i}\) and \(p_{XZ,i}\) were clipped to \([-1,1]\).

The gross reduction in bounded absolute error was
\[
g_i
=
\lvert M_i^*-p_{X,i}\rvert
-
\lvert M_i^*-p_{XZ,i}\rvert.
\]
The normalized cost of acquiring the 24-h profile was fixed at
\(c^*=0.02\) bounded-GR mean-absolute-error units. This quantity is 1\% of
the two-unit endpoint range and is not a monetary cost. A component-weighted
ridge acquisition model predicted \(g_i-c^*\) from the 3-h profile,
\(p_{X,i}\), log-dose, cell line and batch. Nested component-disjoint
cross-fitting within CALIBRATION generated the score used for threshold
selection; the final acquisition model was then refitted on all 21
CALIBRATION components.

For threshold \(t\), the acquisition action was
\[
A_i(t)=\mathbf 1\{s_i\geq t\}.
\]
The 29-candidate grid ranged from \(-2.02\) to \(1.98\) and included the
selected threshold \(-0.06\). The selected episode value was
\[
y_i=A_i(g_i-c^*),
\]
for an evaluable episode, while inactive episodes contributed zero. Let
\(E_i\) indicate that the selected episode and both predictions were
evaluable. A selected episode was false when \(E_i=0\) or its gross
improvement did not strictly exceed \(c^*\). A selected unevaluable episode
remained in the denominator and received the least-favorable net value
\(-2.02\).

For component \(j\), with fixed episode set \(I_j\), let
\[
S_j=\sum_{i\in I_j}A_i,\qquad
F_j=\sum_{i\in I_j}A_i\left[(1-E_i)+E_i\mathbf 1\{g_i\leq c^*\}\right],
\qquad
v_j=\frac{1}{|I_j|}\sum_{i\in I_j}y_i,
\]
with the least-favorable value included in \(v_j\) when needed. The reported
activation fraction was the mean of \(S_j/|I_j|\) over the 21
components. FDP was the mean of \(F_j/S_j\) over activated components, and
value was the mean of \(v_j\) over all 21 components, including components
with no activation.

Candidates with no activated component were ineligible and had no FDP point
estimate. The stored value 1.0 is a conservative software sentinel and is not
interpreted as an FDP estimate.

\paragraph{Operational rule and final-opening sequence}
\label{sec:si_lincs_final_sequence}

The final operational rule required FDP point estimate no greater than 0.45
and component-equal mean value strictly greater than zero. It also required
at least 14 activated components, component-equal activation fraction above
0.5395375, action support across both batches, all six cell lines and at least
eight of the 12 technical groups, and no excessive concentration in one
technical group. Candidate ordering first maximized empirical
component-equal value, then minimized empirical FDP and total acquisitions,
with the larger threshold breaking any remaining tie. The FDP point criterion
was fixed after CALIBRATION and before CERTIFICATION. The positive-mean value
criterion and support constants were fixed after CERTIFICATION, which was
therefore treated as development; only FINAL\_LOCK evaluated the complete
final rule. The support constants, including 14 components and the 0.5395375
activation floor, were author-set operational screens; they are not universal
biological thresholds.

One-sided 95\% component-level FDP upper and value lower bounds were computed
with KL--Hoeffding and Hoeffding constructions, respectively. They are
reported as uncertainty summaries rather than acceptance criteria for this
final operational rule. Outcome diagnostics required positive mean value
after leaving out each of the 12 batch-by-cell groups and limited the largest
absolute group-contribution share to 0.50. These groups are shared technical
strata, not independent new-batch replications.

Before any FINAL\_LOCK \(Z\) or \(M\) value was read, the acquisition model
and threshold were applied to FINAL\_LOCK \(X\) and permitted covariates and
the complete action map was fixed. The 24-h profiles and 72-h outcomes were
opened only after the pre-outcome support checks passed. No model,
preprocessing transformation, threshold or action was updated after this
opening.

\subsection{Value accounting, inference, reproducibility and scope}\label{sec:si_method_value}

\subsubsection{Operational and lifecycle value}\label{sec:cost}

\paragraph{Operational ledgers}\label{sec:operational_cost}

Let \(\phi\in\Phi\) index a prespecified commitment contract rather than an unspecified scalar. For action \(a\) and context \(x\),
\[
\begin{aligned}
C_\phi(a;x)
=\;&C_{\phi}^{\mathrm{reserve}}(a;x)
+C_{\phi}^{\mathrm{setup}}(a;x)
+C_{\phi}^{\mathrm{idle}}(a;x)\\
&+C_{\phi}^{\mathrm{cancel}}(a;x)
+C_{\phi}^{\mathrm{spot}}(a;x)
+C_{\phi}^{\mathrm{decision}}(a;x)\\
&+C_{\phi}^{\mathrm{delay}}(a;x)
+C_{\phi}^{\mathrm{information}}(a;x).
\end{aligned}
\]
Each policy records every term in the ledger, with absent charges set to zero rather than absorbed into scientific success. The operational contrast is
\[
\Delta_{\mathrm{op}}(\phi)
=
\mathbb E[
Y(\pi)-Y(K_0)
-C_\phi(\pi;X)+C_\phi(K_0;X)].
\]
The symbol \(\phi\) is a discrete index for a complete prespecified
commitment-and-cost contract. Every reported contrast is reconstructed from
the full ledger above. No scalar-friction approximation is used in the
reported studies.

\paragraph{Fixed-assignment cpg0012 sensitivity}
\label{sec:cpg_cost_sensitivity_methods}

The post-hoc cpg0012 cost analysis retained the six unit-level assignments,
the \(c=0.02\) NULL, AMBIGUOUS and POSITIVE strata, the final population and
the prespecified least-favorable completion. For a measured compound, we
recovered cost-free incremental utility by adding 0.02 to its deposited net
gain. At a candidate two-well bundle cost \(c\), the executed value for method
\(m\) was
\[
W_{im}(c)
=A_{im}
\begin{cases}
\Gamma_i(0.02)+0.02-c, & \text{measured},\\
-1-c, & \text{unevaluable},
\end{cases}
\]
with value zero when \(A_{im}=0\). Consequently, the all-unit point mean obeys
\(\overline W_m(c)=\overline W_m(0.02)-\widehat q_m(c-0.02)\), where
\(\widehat q_m\) is the frozen activation coverage. At each cost we recomputed
the nominal one-sided 95\% compound-bootstrap BCa lower endpoint using 20,000
resamples. The reported zero crossings were located by numerical bisection
between adjacent cost-grid points with opposite endpoint signs.

No model, score, threshold or assignment was retrained or reselected. Because
the prespecified outcome strata were also held fixed, false activation, FDP and
POSITIVE sensitivity were not reinterpreted at the alternative costs. The
result is a pointwise, descriptive sensitivity analysis rather than a new
confirmatory test or a simultaneous confidence band over cost. Separately, the
risk-budget sensitivity at \(c=0.02\) retained the reported final-test bounds
and classified an active rule as risk-feasible only when its one-sided 95\%
false-activation UCB was no greater than \(r_0\).

\paragraph{Shared model costs and reuse scenarios}\label{sec:shared_cost}

The inherited shared learned-policy cost is 2737.152 normalized units and the characterization component is 108, giving 2845.152 under the inclusive allocation. Because the characterization cost cannot be allocated uniquely between learned and fixed policies, we report incremental cost as an interval. Table~\ref{tab:h_scenarios} lists the prespecified reuse scenarios used for this sensitivity analysis.

\begin{table}[t]
\centering
\caption{\textbf{Prespecified deployment-reuse scenarios.} These are sensitivity scenarios, not five competing estimates of one true \(H\).}
\label{tab:h_scenarios}
\small
\renewcommand{\arraystretch}{1.14}
\rowcolors{2}{tablerow}{white}
\begin{tabularx}{\textwidth}{@{}C{1.5cm}L{3.6cm}L{3.1cm}Y@{}}
\specialrule{\heavyrulewidth}{0pt}{0pt}
\rowcolor{tablehead}
\(H\) & Scenario & Unit & Interpretation \\
\specialrule{\lightrulewidth}{0pt}{0pt}
250 & Campaign-family reference & Campaign & Conservative illustrative case in which one frozen model serves roughly one quarter of a facility-scale annual program \\
\(10^3\) & Facility-wide annual scale & Campaign & External upper-scale anchor; not an assumption that one mechanism-specific model serves every proposal \\
\(3\times10^3\) & Run-cycle condition block & Condition block / shift & Higher-frequency reuse within operating periods \\
\(10^4\) & Annual high-throughput & Condition block / shift & High-throughput sensitivity \\
\(10^5\) & Long-reuse stress only & Sample / decision event & Unit-relaxed stress case; not campaign-primary \\
\specialrule{\heavyrulewidth}{0pt}{0pt}
\end{tabularx}
\end{table}

The annual facility-wide scale is anchored to 1,312 awarded Diamond proposals and 989 allocated ESRF proposals in the cited annual sources~\cite{diamond2024,esrf2024}. These counts establish only the order of magnitude. The \(H=250\) family scenario uses an illustrative, rather than estimated, 25\% allocation share. The primary quantity is the break-even curve \(H^\star(\phi)\); the listed scenarios do not validate the normalized model-development costs.

\subsubsection{Statistical inference and reproducibility}\label{sec:inference}

\paragraph{Resampling and simultaneous families}\label{sec:resampling}

All learned-policy values are aggregated to the prespecified independent unit
before resampling. One resample plan is shared within each family. The maximum
absolute studentized statistic supplies a common critical value and
multiplicity correction across selectors, friction points and
regimes~\cite{chernozhukov2013}.

Oracle opportunities are descriptive truth quantities in simulation. Estimated
opportunity and policy value use cross-fitted models and disjoint evaluation
data. Recovered fractions
\[
\rho
=
\frac{V(\pi)-V(K_0)}{V(\pi_M^\star)-V(K_0)}
\]
are reported only when a simultaneous lower bound establishes a non-negligible
denominator. Joint inversion is used rather than dividing two point estimates.

\paragraph{Assurance near decision thresholds}\label{sec:knife_edge}

A design estimate close to a threshold cannot support a binary claim by
itself. A binding claim additionally requires independent design simulations
to establish:

\begin{enumerate}
\item a prespecified minimum point probability of joint success;
\item an exact lower confidence bound on that assurance;
\item a Monte Carlo precision requirement;
\item valid coverage at the scientific independent unit.
\end{enumerate}

Without this assurance, the quantity remains an estimation target. In the
held-out continuous-population study, sensitivity is therefore an estimation
target, whereas NULL false-activation risk remains binding.

\paragraph{Analysis separation and reproducibility}\label{sec:governance}

Design-power simulation, development calibration, risk locking and independent
evaluation use disjoint data and seed namespaces. Outputs are archived by
analysis stage, and corrected implementations generate new manifests while
preserving earlier runs. Each run records its scientific role, configuration,
source version, unit and seed registries, overlap audit, policy and cost ledgers,
and predefined fallback states. The protocol prohibits repeated evaluation
until a criterion passes, relabelling development output as validation, or
inferring a lifecycle denominator from observed break-even results.

\subsubsection{Data and reproducibility}\label{sec:data_register}

The evidence package separates three core biological archive sources from
supporting methodological and constructibility sources. CTRP supplies the
pre-partitioned measured-outcome fallback evaluation; cpg0012 supplies the
target-calibrated sparse-candidate evaluation; and LINCS and LJP supply the
broad temporal-profile acquisition evaluation. Author-generated simulator
ledgers support the theoretical and risk-control analyses, Causal Chambers
supports constructibility only, and facility documents provide scheduling or
scale context. None of these supporting sources substitutes for an empirical
archive state.

Figure source data, frozen summaries, validation reports and manifests are
archived in
\href{https://doi.org/10.5281/zenodo.21859443}{doi: 10.5281/zenodo.21859443}.
The cited record contains all 29 LINCS--LJP calibration candidates,
certification and FINAL aggregate results and the deterministic LINCS--LJP
frame-construction and analysis pipeline. The record also contains a 4,049-row
episode index and a 105-component identity and partition map, together with
their schema, deterministic builder, validation record and checksum manifest.
These publication indexes were reconstructed after the independently executed
evaluation from public source metadata and the frozen matching and partition
rules. They contain no outcomes, predictions or held-out statistics and did not
repeat model fitting, rule selection or outcome aggregation. The aggregate
tables directly reproduce every value in Main Fig.~5; full episode-level
reanalysis requires reacquiring the public third-party inputs and rerunning the
deposited pipeline. The latest manuscript source and post-hoc
score-sweep materials accompany this submission. Author-generated software in
the submission package is released under the MIT License; the public
\href{https://github.com/JIABI/mard_autonomous}{GitHub repository} is released
under the same MIT License. Supplementary
Table~\ref{tab:data_sources} records the public location and evidentiary role
of every external source.

\begin{sidewaystable}[p]
\centering
\begin{minipage}{0.93\textheight}
\centering
\caption{\textbf{Data and contextual-source register.} ``Used'' means that the source enters a method, scale anchor or core outcome. External repositories that did not meet the eligibility criteria are listed to make their evidentiary role explicit.}
\label{tab:data_sources}
\fontsize{8.1}{9.0}\selectfont
\renewcommand{\arraystretch}{1.08}
\rowcolors{2}{tablerow}{white}
\begin{tabularx}{\linewidth}{@{}L{2.6cm}L{3.5cm}Y L{3.2cm}@{}}
\specialrule{\heavyrulewidth}{0pt}{0pt}
\rowcolor{tablehead}
Source & Public location & Role in this study & Status \\
\specialrule{\lightrulewidth}{0pt}{0pt}
Author-generated OPAL simulator & Versioned public repository and archival record accompanying submission & Author-generated simulation outcome ledgers, oracle truth, campaign observations, potential outcomes and costs & Used; synthetic only \\
OPAL evidence and source data & \href{https://doi.org/10.5281/zenodo.21859443}{doi: 10.5281/zenodo.21859443}; code at \href{https://github.com/JIABI/mard_autonomous}{GitHub} & Figure source data, result tables, validation reports and reproducibility records for the original studies & Public manuscript record; third-party raw source files excluded \\
LINCS--LJP reproducibility source & \href{https://doi.org/10.5281/zenodo.21859443}{doi: 10.5281/zenodo.21859443}; latest revision materials accompanying submission & Calibration candidate table, final aggregate, Main Fig.~5 build script, 4,049-row episode index and 105-component partition map & Public aggregate record and derived identity/partition indexes; third-party expression profiles and growth-response values excluded \\
Diamond Light Source & \href{https://www.diamond.ac.uk/Home/News/LatestNews/2024/20112024.html}{annual statistics} & Annual proposal/shift scale for \(H\) scenarios & Used as scale anchor \\
Diamond application FAQ & \href{https://www.diamond.ac.uk/Users/Apply-for-Beamtime/Application-FAQs.html}{application FAQ} & Proposal and allocation context & Used as context \\
ESRF Highlights 2024 & \href{https://www.esrf.fr/files/live/sites/www/files/ExternalFiles/Communication/Highlights/2024/172-ESRF-Highlights-2024.html}{annual highlights} & Annual proposal/shift scale for \(H\) scenarios & Used as scale anchor \\
SOLEIL availability policy & \href{https://www.synchrotron-soleil.fr/en/users/current-call-proposals/beamtime-availability-and-partition}{availability page} & Block-structured access context & Cite-only context \\
SSRL scheduling procedure & \href{https://www-ssrl.slac.stanford.edu/ssrl/web/node/28}{scheduling page} & Advance scheduling context & Cite-only context \\
APS access policy & \href{https://www.aps.anl.gov/files/APS-sync/centraldocs/policy_procedures/user/docs/APS_1700813.pdf}{policy PDF} & Beam-time allocation context & Cite-only context \\
NSLS-II proprietary policy & \href{https://www.bnl.gov/nsls2/docs/pdf/proprietaryresearchpolicyandprocedure.pdf}{policy PDF} & Economic sensitivity context only & No outcome or utility conversion \\
Olympus & \href{https://github.com/aspuru-guzik-group/olympus}{GitHub repository} & Candidate external benchmark framework & Compatibility assessed; not used for core outcomes \\
Causal Chambers & \href{https://github.com/juangamella/causal-chamber}{GitHub repository; commit \texttt{6da4e646}} & Fourteen scalar archives examined; 13 used only for the Supplementary post-selection proxy-constructibility analysis & Descriptive constructibility analysis only; no core or locked decision-rule claim; prospective instrument execution not started (max 59 independent sessions; no logged propensities) \\
CTRP v2 & \href{https://portals.broadinstitute.org/ctrp.v2.1/}{Broad CTRP portal}; \href{https://ocg.cancer.gov/programs/ctd2/data-portal}{NCI CTD\({}^{2}\) portal} & Measured four-compound response panel; 125 development, 244 calibration and 254 locked cell-line families & Used in locked pharmacogenomic exact-outcome evaluation; offline exact-outcome evaluation; author-defined nested capacity \\
cpg0012 Cell Painting & Cell Painting Gallery cpg0012; original publications and gallery resource~\cite{wawer2014profiling,bray2017cellpaintingdata,weisbart2024gallery} & Per-well morphology, plate maps, structures and replicate roles for 25,013 eligible compound clusters; 13,748 development and 11,265 final & Used in target-calibrated one-shot final test; public archive, author-controlled procedural separation, adaptive-replicate abstraction \\
LINCS L1000 and LJP growth-rate data & \href{https://www.ncbi.nlm.nih.gov/geo/query/acc.cgi?acc=GSE70138}{NCBI GEO GSE70138}; \href{https://doi.org/10.1038/s41467-017-01383-w}{published LJP growth-rate resource} & Matched 3-h and 24-h profiles over 978 landmark genes and 72-h GR endpoints; 105 perturbation-identity components and 4,049 complete-case episodes split 42/21/21/21 & Used in component-disjoint external final evaluation; positive point-estimate net value under 860/860 final acquisition; offline complete-case frame \\
\specialrule{\heavyrulewidth}{0pt}{0pt}
\end{tabularx}
\begin{tablenotes}[flushleft]\footnotesize
\item Public facility documents support institutional structure or scale only. They do not provide OPAL policy outcomes, normalized utility values or measured cancellation penalties unless the cited source explicitly states them. No private facility trace was used.
\end{tablenotes}
\end{minipage}
\end{sidewaystable}

\clearpage

\subsubsection{Scope of empirical evidence}\label{sec:limitations}

The three core empirical studies are offline evaluations of recorded
biological archives; none is a prospective online deployment on a physical
instrument. Their conclusions are archive-specific evidence states rather
than claims of transport to new facilities, batches or biological
populations. The simulator and Causal Chambers analyses form a separate
methodological evidence layer.

The cpg0012 study used both labeled pre-final partitions for model and
threshold development, leaving one untouched 11,265-compound final test. This
is a single held-out evaluation rather than an independent two-stage or
brokered confirmation. The FDP point estimate was below its target, but the
95\% UCB was not, so the complete contract kept the fixed plan in place. Adding two
archived replicates is an adaptive-replicate abstraction rather than a
prospective facility operation. The earlier 1,253-per-evaluation-partition and
3,000-registry assurance values are planning scenarios, not realized cpg0012
operating characteristics. The post-hoc plate-map analysis motivates, but does
not replace, block-disjoint evaluation.

The LINCS--LJP study is an offline complete-case replay within six cell lines
and two batches. Its 105 perturbation-identity components are
component-disjoint allocation and analysis units, but the 12 batch-by-cell
groups recur across partitions and do not support inference to new batches or
cell lines. The final rule selected 860/860 episodes, so the positive mean
point estimate pertains to broad acquisition rather than selective measurement
saving. FDP and value passed the stated point-estimate criteria; their
component-level UCB and LCB are uncertainty summaries and do not provide
confidence-bound authorization. The study does not establish a causal benefit
of delaying measurement to 24 h or prospective laboratory performance.

The locked CTRP evaluation exercised no active branch. Its \(K\) actions are
author-defined compound-bundle sizes evaluated by lookup in a complete
measured response panel, and the seven locked POSITIVE families leave
positive-stratum sensitivity imprecisely estimated. The result establishes
fallback for the current frozen score, not exhaustion of the available
decision-time information or failure of every alternative representation.

The held-out simulator exercises an active branch, but
\(S^{\mathrm{eval}}\) uses potential outcomes unavailable at decision time.
The Causal Chambers analysis establishes replay of an author-defined
prefix-score--action--fallback computation on measured trajectories; it does
not identify causal value or controlled activation risk. Facility documents
supply scheduling context and reuse scale rather than policy outcomes or
monetary utility. These analyses establish methodological properties and
scope, not additional biological archive states.

\clearpage
\section{SUPPLEMENTARY TABLES}\label{sec:supp_tables}

For quick cross-reference, Table~\ref{tab:ed_estimands} summarizes the
selector-frontier estimands and baselines, and Table~\ref{tab:ed_d6}
summarizes the constrained-risk calibration and validation partitions.

\begin{table}[ht]
\centering
\caption{\textbf{Selector-frontier estimand and baseline crosswalk.} Internal
analysis identifiers are replaced by their scientific definitions.}
\label{tab:ed_estimands}
\small
\renewcommand{\arraystretch}{1.13}
\setlength{\tabcolsep}{3pt}
\rowcolors{2}{tablerow}{white}
\begin{tabularx}{\textwidth}{@{}L{2.35cm}L{2.55cm}L{2.3cm}L{1.9cm}Y@{}}
\specialrule{\heavyrulewidth}{0pt}{0pt}
\rowcolor{tablehead}
Estimand & Outcome/value source & Comparator & Mean contrast & Scientific role \\
\specialrule{\lightrulewidth}{0pt}{0pt}
\(K\)-only surrogate & Label-derived per-\(K\) value & \(K_{\mathrm{score},-f}=1\) & \(+5.549\times10^{-3}\) & Surrogate decision value; not adoption \\
Executed operational & Executed no-probe policy utility & \(K_0=2\) & \(-0.814\times10^{-3}\) & Correct operational comparison \\
Incremental-cost lifecycle & Executed operational plus incremental shared cost & \(K_0=2\) & No finite \(H^\star\) & Fixed-selection cost not identified in frozen ledger \\
Common-cost lifecycle & Executed operational with common-cost cancellation & \(K_0=2\) & No finite \(H^\star\) & Lifecycle contrast equals non-positive operational contrast \\
\specialrule{\heavyrulewidth}{0pt}{0pt}
\end{tabularx}
\end{table}

\begin{table}[ht]
\centering
\caption{\textbf{Constrained-risk calibration and prespecified held-out
simulator validation.}}
\label{tab:ed_d6}
\small
\renewcommand{\arraystretch}{1.13}
\rowcolors{2}{tablerow}{white}
\begin{tabularx}{\textwidth}{@{}L{2.7cm}C{1.3cm}C{1.5cm}L{2.3cm}Y@{}}
\specialrule{\heavyrulewidth}{0pt}{0pt}
\rowcolor{tablehead}
Partition & Fixtures & Events & Exact interval or bound & Role \\
\specialrule{\lightrulewidth}{0pt}{0pt}
NULL tuning & 140 & 0 above selected maximum & Selected \(\tau=0.0180806\) & Threshold construction only \\
NULL risk lock & 500 & 0 activations & 99.5\% upper bound 0.01054 & Independent lock; passed \\
NULL validation & 400 & 0 activations & 99.5\% upper bound 0.01316 & Binding claim; risk limit 0.05 \\
Positive validation & 500 & 342 activations & 95\% interval 0.641--0.725 & Sensitivity estimate 0.684; no pass/fail target \\
\specialrule{\heavyrulewidth}{0pt}{0pt}
\end{tabularx}
\end{table}

\clearpage

\end{document}